\documentclass{article}

\usepackage[final]{neurips_2022}

\usepackage[utf8]{inputenc} 
\usepackage[T1]{fontenc}    
\usepackage{url}            
\usepackage{booktabs}       
\usepackage{amsfonts}       
\usepackage{nicefrac}       
\usepackage{microtype}      
\usepackage[dvipsnames]{xcolor}         
\usepackage{mkolar_definitions}

\usepackage{amsmath, amsthm, amssymb, amsfonts, mathtools, graphicx, enumerate}
\usepackage{xspace}
\usepackage{algorithm}
\usepackage{algcompatible}
\usepackage{enumitem}
\usepackage{comment}
\usepackage{bm}
\usepackage{array}
\usepackage[colorlinks=true, linkcolor=blue, citecolor=blue]{hyperref}

\usepackage{colortbl}

\usepackage{tikz}
\usepackage{bbm}
\usetikzlibrary{automata, arrows}
\usetikzlibrary{positioning}
\usetikzlibrary{tikzmark}

\usepackage[capitalise,nameinlink]{cleveref}

\usepackage{natbib}
\bibpunct{(}{)}{;}{a}{,}{,}

\newtheorem*{theorem*}{Theorem}
\newtheorem*{lemma*}{Lemma}
\newtheorem*{corollary*}{Corollary}

\newcommand{\olive}{\textsc{Olive}\xspace}
\newcommand{\rfolive}{\textsc{RFOlive}\xspace}
\newcommand{\jolive}{\textsc{JointOlive}\xspace}
\newcommand{\golf}{\textsc{Golf}\xspace}
\newcommand{\fqi}{\textsc{FQI}\xspace}

\newcommand{\veps}{\varepsilon}

\newcommand{\unif}{\mathrm{unif}}
\newcommand{\etamin}{\eta_{\mathrm{min}}}

\newcommand{\poly}{\mathrm{poly}}

\newcommand{\defeq}{:=}

\newcommand{\ecover}{\varepsilon_{\textrm{elim}}/64}

\newcommand{\nactv}{n_{\textrm{actv}}}
\newcommand{\nelim}{n_{\textrm{elim}}}

\newcommand{\eactv}{\varepsilon_{\textrm{actv}}}
\newcommand{\eelim}{\varepsilon_{\textrm{elim}}}

\newcommand{\dqbe}{\mathrm{dim}_{\mathrm{qbe}}}
\newcommand{\dvbe}{\mathrm{dim}_{\mathrm{vbe}}}
\newcommand{\dqbee}{d_{\mathrm{qbe}}}
\newcommand{\dvbee}{d_{\mathrm{vbe}}}
\newcommand{\dde}{d_{\mathrm{de}}}
\newcommand{\dbr}{d_{\mathrm{br}}}

\newcommand{\EcalR}{\mathcal{E}^{R}}
\newcommand{\EcalQ}{\mathcal{E}^{R}_{\mathrm{Q}}}
\newcommand{\EcalV}{\mathcal{E}^{R}_{\mathrm{V}}}

\newcommand{\hEcalQ}{\hat{\mathcal{E}}^{R}_{\mathrm{Q}}}

\newcommand{\hEcalsqzero}{\hat{\Ecal}_{\square}^\zero}

\newcommand{\EcalQzero}{\mathcal{E}^{\zero}_{\mathrm{Q}}}
\newcommand{\EcalVzero}{\mathcal{E}^{\zero}_{\mathrm{V}}}
\newcommand{\tEcalQzero}{\tilde{\mathcal{E}}^{\zero}_{\mathrm{Q}}}
\newcommand{\tEcalVzero}{\tilde{\mathcal{E}}^{\zero}_{\mathrm{V}}}
\newcommand{\hEcalQzero}{\hat{\mathcal{E}}^{\zero}_{\mathrm{Q}}}
\newcommand{\hEcalVzero}{\hat{\mathcal{E}}^{\zero}_{\mathrm{V}}}

\newcommand{\Fon}{\mathcal{F}_{\mathrm{on}}}
\newcommand{\Foff}{\mathcal{F}_{\mathrm{off}}}
\newcommand{\Zon}{\mathcal{Z}_{\mathrm{on}}}
\newcommand{\Pon}{\Pi_{\mathrm{on}}}
\newcommand{\Fsur}{\mathcal{F}_{\mathrm{sur}}}

\newcommand{\philc}{\phi^{\mathrm{lc}}}
\newcommand{\philci}{\phi^{\mathrm{lc},i}}
\newcommand{\philr}{\phi^{\mathrm{lr}}}
\newcommand{\mulr}{\mu^{\mathrm{lr}}}
\newcommand{\Philr}{\Phi^{\mathrm{lr}}}
\newcommand{\Philc}{\Phi^{\mathrm{lc}}}
\newcommand{\dlr}{d_{\mathrm{lr}}}
\newcommand{\dlc}{d_{\mathrm{lc}}}

\newcommand{\ca}{c_1} 
\newcommand{\cb}{c_2} 
\newcommand{\ce}{c_3} 
\newcommand{\cf}{c_4} 
\newcommand{\cc}{c_5} 
\newcommand{\cd}{c_6} 
\newcommand{\cg}{c_7} 

\newcommand{\tabincell}[2]{\begin{tabular}{@{}#1@{}}#2\end{tabular}}
\newcommand{\para}[1]{\textbf{#1}~}

\newcount\Comments  
\definecolor{darkred}{rgb}{0.7,0,0}
\definecolor{darkgreen}{rgb}{0,0.5,0}
\definecolor{orange}{rgb}{0.7,0.4,0}
\definecolor{purple}{rgb}{0.8,0.0,0.8}
\definecolor{mycolor}{rgb}{0.2,0.8,0.8}
\newcommand{\kibitz}[2]{\ifnum\Comments=1{\textcolor{#1}{\textsf{\footnotesize #2}}}\fi}

\title{On the Statistical Efficiency of Reward-Free Exploration in Non-Linear RL}

\author{%
  Jinglin Chen \thanks{Equal contribution} \\
    Department of Computer Science\\
    University of Illinois Urbana-Champaign\\
  \texttt{jinglinc@illinois.edu} \\
 \And
 Aditya Modi $^*$
 \\
    Microsoft \\
 \texttt{admodi@umich.edu} \\
\AND Akshay Krishnamurthy
\\
       Microsoft Research\\
       \texttt{akshaykr@microsoft.com} \\
    \And Nan Jiang \\
      Department of Computer Science\\
       University of Illinois Urbana-Champaign\\
       \texttt{nanjiang@illinois.edu}\\
  \AND 
  Alekh Agarwal \\
  Google Research  \\
  \texttt{alekhagarwal@google.com} \\
}

\begin{document}

\maketitle

\begin{abstract}
We study reward-free reinforcement learning (RL) under general non-linear function approximation, and establish sample efficiency and hardness results under various standard structural assumptions. On the positive side, we propose the \rfolive (Reward-Free \olive) algorithm for sample-efficient reward-free exploration under minimal structural assumptions, which covers the previously studied settings of linear MDPs \citep{jin2019provably}, linear completeness \citep{zanette2020provably} and low-rank MDPs with unknown representation \citep{modi2021model}.  Our analyses indicate that the \emph{explorability} or \emph{reachability} assumptions, previously made for the latter two settings, are not necessary statistically for reward-free exploration. On the negative side, we provide a statistical hardness result for both reward-free and reward-aware exploration under linear completeness assumptions when the underlying features are unknown, showing an exponential separation between low-rank and linear completeness settings. 
\end{abstract}

\section{Introduction}
Designing a reward function which faithfully captures the task of interest remains a central practical hurdle in reinforcement learning (RL) applications. To address this, a series of recent works~\citep{jin2020reward,zhang2020task,wang2020on,zanette2020provably,qiu2021reward} investigate the problem of reward-free exploration, where the agent initially interacts with its environment to collect experience (``online phase''), that enables it to perform \emph{offline learning} of near optimal policies for any reward function from a potentially pre-specified class (``offline phase''). 
Reward-free exploration also provides a basic form of multitask RL, enabling zero-shot generalization, across diverse rewards, and provides a useful primitive in tasks such as representation learning~\citep{agarwal2020flambe,modi2021model}. So far, most of the study of reward-free RL has focused on tabular and linear function approximation settings, in sharp contrast with the literature on reward-aware RL, where abstract structural conditions identify when general function approximation can be used in a provably sample-efficient manner \citep{jiang2017contextual,jin2021bellman,du2021bilinear}.

In this paper, we seek to bridge this gap and undertake a systematic study of reward-free RL in a model-free setting with general function approximation. 
We devise an algorithm, \rfolive, which is non-trivially adapted from its reward-aware counterpart \citep{jiang2017contextual}, and provide polynomial sample complexity guarantees under general conditions that significantly relax the assumptions needed by prior reward-free RL works. Our results produce both algorithmic contributions and important insights about the tractability of reward-free RL, as we summarize below (see also \Cref{tab:summary}).

\begin{table}[t]
\renewcommand{\arraystretch}{1.5}
\centering
\begin{tabular}{p{0.05\linewidth}|p{0.6\linewidth}|p{0.25\linewidth}} \hline
& \textbf{Setting} & \textbf{Reference} \\ \hline
1 & Linear \tikzmark{a} MDP  &  \citet{wang2020on} \\ \hline
2 & \tikzmark{h}Linear co\tikzmark{b}mpleteness + {\color{red} explorability} & \citet{zanette2020provably} \\ \hline
3 & Comple\tikzmark{c}teness + \tikzmark{d}Q-type B-E dimension & \cellcolor[gray]{0.85} \textbf{\pref{thm:rfolive_q}} \\\hline
4 & Complete\tikzmark{i}ness + \tikzmark{e}V-type B-E dimension + small  $|\mathcal{A}|$ & \cellcolor[gray]{0.85} \textbf{\pref{thm:rfolive_v}} \\ \hline
5 & Low-rank\tikzmark{f} MDP ($\phi^* \in \Phi$) + small $|\mathcal{A}|$ + {\color{red} reachability} & \citet{modi2021model}\\ \hline
6 & \tabincell{c}{\tikzmark{j}Linear co\tikzmark{g}mpleteness ($\phi^* \in \Phi$) + small  $|\mathcal{A}|$ + reward-aware    \vspace{-.5em}  \\ + explorability + reachability  + generative model \quad \quad \quad \quad \quad} \hfill & \cellcolor[gray]{0.85} \textbf{\pref{thm:rf_lower_bound}} (intractable) \\ \hline
\end{tabular}
\begin{tikzpicture}[overlay, remember picture, shorten >=.5pt, shorten <=.5pt, transform canvas={yshift=.25\baselineskip}]
\draw [->, thick, blue] ([yshift=-3pt]{pic cs:a}) to ([yshift=3pt]{pic cs:b});
\draw [->, thick, blue] ([yshift=-3pt,xshift=-1pt]{pic cs:b}) -- ([yshift=3pt, xshift=1pt]{pic cs:c});
\draw [->, thick, blue] ([yshift=-3pt,xshift=-1pt]{pic cs:b}) -- ([yshift=3pt, xshift=1pt]{pic cs:d});
\draw [->, thick, blue] ([yshift=3pt]{pic cs:f}) -- ([yshift=-3pt,xshift=1pt]{pic cs:e});
\draw [->, thick, blue] ([yshift=3pt]{pic cs:f}) -- ([yshift=-3pt,xshift=-1pt]{pic cs:i});
\draw [->, thick, blue] ([yshift=-4pt,xshift=-1pt]{pic cs:f}) -- ([yshift=3pt, xshift=-1pt]{pic cs:g});
\draw [->, thick, blue] ([xshift=-3pt]{pic cs:h}) [bend right] to ([xshift=-3pt]{pic cs:j});
\end{tikzpicture}
\caption{Summary of our results and comparisons to most closely related works in reward-free exploration. Blue arrows 
represent implication ($A$ {\color{blue} $\to$} $B$ means $B$ is a consequence of and hence weaker condition than $A$), and the {\color{red} red assumptions} are what prior works need that are avoided by us. For linear settings, the true feature $\phi^*$ is assumed known unless otherwise specified (e.g., in Rows 5 \& 6, $\phi^*$ is unknown but belongs to a feature class $\Phi$). ``B-E'' stands for (low) Bellman Eluder dimension \citep{jin2021bellman}. Row 6 has many assumptions, which make it strong since it is a negative result. The detailed comparisons of existing sample complexity rates and our corollaries can be found in \pref{app:comparison}.
\label{tab:summary}}
\vspace{-1em}
\end{table}

\para{Algorithmic contribution: beyond linearity} A unique challenge in reward-free RL is that the agent must exhaustively explore the environment during the online phase, since it does not know which states will be rewarding in the offline phase. A natural idea to tackle this challenge is to deploy a reward-aware RL ``base algorithm'' with the $\zero$ reward function, since this algorithm must explore to certify that there is indeed no reward. Prior works adopt this idea with optimistic value-iteration (VI) approaches, which use proxy reward functions to drive the agent to new states. However these optimistic methods rely heavily on linearity assumptions to construct the proxy reward, and it is difficult to extend them to general function approximation. Instead of using optimistic VI, our basic building block is the \olive\footnote{We use the Q-type and V-type versions of \olive from \citet{jin2021bellman} as their structural assumption of low Bellman Eluder dimension subsumes the low Bellman rank assumption in \citet{jiang2017contextual} (see \pref{prop:br_be}).} algorithm of~\citet{jiang2017contextual}, a constraint-gathering and elimination algorithm that is a central workhorse for reward-aware RL with general function approximation. In the online phase of \rfolive, we run this algorithm with the $\zero$ reward function, and we save the set of constraints gathered (in the form of \textit{separate} datasets) for use in the offline phase.

\para{Algorithmic contribution: novel offline module}
Prior works for reward-free RL typically use regression approaches \citep{ernst2005tree,chen2019information,jin2019provably} in the offline phase, e.g., \fqi \citep{modi2021model,zanette2020provably}, or its optimistic variants~\citep{zhang2020task,wang2020on}.
In the offline phase of \rfolive, rather than relying on regression, we enforce the constraints gathered in the online phase, which amounts to eliminating functions that have large average Bellman errors on state-distributions visited in the online phase. This generic elimination scheme does not rely on tabular or linear structures and allows us to move beyond these assumptions to obtain reward-free guarantees in much more general settings.

\para{Implications: positive results} The major assumptions that enable our sample complexity guarantees are Bellman-completeness (\pref{assum:completeness_F}) and low Bellman Eluder dimension 
(\pref{def:dim_qbe} and \pref{def:dim_vbe}); see Rows 3 and 4 in \Cref{tab:summary}. 
These conditions significantly relax prior assumptions in the more restricted settings. Furthermore, prior works in the linear completeness and low-rank MDP settings require \textit{explorability/reachability} assumptions \citep{zanette2020provably,modi2021model}, which, roughly speaking, assert that every direction in the state-action feature space can be visited with sufficient probability. These assumptions are often not needed in reward-aware RL but suspected to be necessary for model-free reward-free settings. Our results do not depend on such assumptions, showing that they are not necessary for sample-efficient reward-free exploration either. 

\para{Implications: negative results} We develop lower bounds, showing that some of the structural assumptions made here are not easily relaxed further. While the settings of linear completeness with known features (Row 3), and low-rank MDPs with unknown features  (Row 4) are both independently tractable, we show a hardness result against learning under linear completeness when the features are unknown, even under a few additional assumptions (Row 6). 

Taken together, our results take a significant step in bridging the sizeable gap in our understanding of reward-aware and reward-free settings and bring the two closer to an equal footing. 

\paragraph{Related work}
In recent years, we have seen a wide range of results for reward-aware RL under general function approximation~\citep{jiang2017contextual,dann2018oracle,sun2019model,wang2020reinforcement,jin2021bellman,du2021bilinear}. These works develop statistically efficient algorithms using structural assumptions on the function class. Despite their generality, a trivial extension to the reward-free setting incurs an undesirable linear dependence on the size of the reward class. 

There also exists a line of research on reward-free RL in various settings: tabular MDPs \citep{jin2020reward,zhang2020task, kaufmann2021adaptive,menard2021fast,yin2021optimal,wu2022gap}, MDPs with the linear structure \citep{wang2020on,zhang2021reward,zanette2020provably,huang2021towards,wagenmaker2022reward}, kernel MDPs \citep{qiu2021reward}, block/low-rank MDPs \citep{misra2020kinematic,agarwal2020flambe,modi2021model}, and multi-agent settings \citep{bai2020provable,liu2021sharp}. Many of these settings can be subsumed by our more general setup.

Our offline module uses average Bellman error constraints, which is related to a line of work in offline RL \citep{xie2020q,jiang2020minimax,chen2022offline,zanette2022bellman}. However, there is only one dataset in the standard offline RL setting, and these works form multiple average Bellman error constraints using an additional helper class for reweighting, and need to impose additional realizability- or even completeness-type assumptions on such a class. In contrast, we naturally collect \textit{multiple} datasets in the online phase, so we do not require a parametric class for reweighting during offline learning.

\section{Preliminaries}
\paragraph{Markov Decision Processes (MDPs)}
We consider a finite-horizon episodic Markov decision process (MDP) defined as $M = (\Xcal, \Acal, P, H)$, where $\Xcal$ is the state space, $\Acal$ is the action space, $P = (P_0,\ldots,P_{H-1})$ with $P_h: \Xcal \times \Acal \rightarrow \Delta(\Xcal)$ is the transition dynamics, and $H$ is the number of timesteps in each episode. If the number of actions is finite, we denote the cardinality  $\abr{\Acal}$ by $K$. In each episode, an agent generates a trajectory $\tau = \rbr{x_0,a_0,x_1,\ldots, x_{H-1}, a_{H-1}, x_H}$ by taking a sequence of actions $a_0,\ldots,a_{H-1}$, where $x_0$ is a fixed starting state and $x_{h+1} \sim P_h(\cdot\mid x_h,a_h)$. For simplicity, we will use $a_{i:j}$ to denote $a_i,\dots,a_j$ and use the notation $[H]$ to refer to $\cbr{0,1,\ldots, H-1}$.
We use the notation $\pi$ to denote a collection of $H$ (deterministic) policy functions $\pi = (\pi_0,\ldots, \pi_{H-1})$, where $\pi_h: \Xcal \rightarrow \Acal$. For any $h \in [H]$ with $h'>h$, we use the notation $\pi_{h:h'}$ to denote the policies $(\pi_h, \pi_{h+1} \ldots,\pi_{h'})$. For any policy $\pi$ and reward function\footnote{We consider deterministic reward and initial state for simplicity. Our results easily extend to stochastic versions.} $R =(R_0,\ldots,R_{H-1})$ with $R_h:\Xcal\times\Acal\rightarrow[0,1]$, we define the policy-specific action-value (or Q-) function as $Q^\pi_{R,h}(x,a) = \EE_{\pi}[\sum_{h'=h}^{H-1} R(x_{h'},a_{h'})\mid x_h=x, a_h=a]$ and state-value function as $V^\pi_{R,h}(x) = \EE_{\pi}[Q^\pi_{R,h}(x,a_h)\mid x_h=x, a_h \sim \pi]$. 
We also use $v^\pi_R = V^\pi_{R,0}(x_0)$ to denote the expected return of policy $\pi$. For any fixed reward function $R$, there exists a policy $\pi^*_R$ such that $v^*_R=V^{\pi^*_R}_{R,h}(x) = \sup_{\pi} V^\pi_{R,h}(x)$ for all $x \in \Xcal$ and $h \in [H]$, where $v^*_R$ denotes the optimal expected return under $R$. 
We use $\Tcal^R_h$ to denote the reward-dependent Bellman operator: $\forall f_{h+1} \in \RR^{\Xcal\times\Acal}$, $(\Tcal^R_h f_{h+1})(x,a) \defeq R_h(x,a) + \EE\sbr{\max_{a' \in \Acal} f_{h+1}(x',a') \mid x' \sim P_h(\cdot\mid x,a)}$ 
and similarly define $\Tcal^\zero_h$ for the operator with zero reward.
The optimal action-value function (under reward $R$) $Q^*_R$ satisfies the Bellman optimality equation $Q^*_{R,h} = \Tcal^R_h Q^*_{R,h+1},\forall h\in[H]$. 

\paragraph{Reward-free RL with function approximation} We study reward-free RL with value function approximation, wherein, the agent is given a function class 
$\Fcal=\Fcal_0\times\ldots\times\Fcal_{H-1}$ 
where $\Fcal_h: \Xcal \times \Acal \rightarrow [-(H-h-1), H-h-1],\forall h\in[H]$.\footnote{Since it is natural to use $\Fcal$ to capture the reward-independent component (\pref{assum:realizability_F}) in our reward-free setting, we assume $\Fcal_h$ is upper bounded by $H-h-1$. We include the negative range to simplify the discussions for various instantiations. Our main results also hold if we assume $\Fcal_h:\Xcal\times\Acal\rightarrow[0,H-h-1]$.}
Without loss of generality, we assume $\zero\in\Fcal_h,\forall h\in[H]$ and
$f_{H}\equiv \zero,\forall f\in\Fcal$. 
For any $f\in\Fcal$, we use $V_{f,h}$ to denote its induced state-value function, i.e., $V_{f,h}(x)=\max_a f_h(x,a)$ and $\pi_f(x)$ as its greedy policy, i.e., $\pi_{f,h}(x)=\argmax_a f_h(x,a)$. When these functions take $x_h$ as input and there is no confusion, we may drop the subscript $h$ and use $V_f(x_h)$ and $\pi_f(x_h)$.

In reward-free RL, the agent is given access to a reward class $\Rcal$, but the specific reward function is only selected after the agent finishes interacting with the environment. 
Specifically, the agent operates in two phases: an \emph{online} phase where it explores the given MDP $M$ to collect a dataset of trajectories $\Dcal$ without the reward information, and an \emph{offline} phase, where it uses the collected dataset $\Dcal$ to optimize for any revealed reward function $R \in \Rcal$. 

Our goal is to investigate the statistical efficiency of reward-free RL with general non-linear function approximation: how many trajectories does the agent need to collect in the online phase such that in the offline phase, with probability at least $1-\delta$, for any $R \in \Rcal$, it can compute a near-optimal policy $\pi_R$ satisfying $v^{\pi_R}_R \geq v^*_R - \veps$?
We measure the statistical efficiency in terms of the structural complexity of function class $\Fcal$, reward class $\Rcal$, horizon $H$, accuracy $\veps$ and failure probability $\delta$. 

As for expressivity assumptions, we assume the function class $\Fcal$ is realizable and complete. 
Realizability requires that the optimal function $Q^*_R$ belongs to the reward-appended class $\Fcal + R$, which is natural in the reward-free setting where the agent uses $\Fcal$ to capture reward-independent information.
Completeness requires that the Bellman backups of 
and 
$\Fcal_{h+1} + R_{h+1}$ belong to $\Fcal_h$, and additionally that the Bellman backup of $\Fcal_{h+1}-\Fcal_{h+1}$ belongs to $\Fcal_{h}-\Fcal_{h}$.
\begin{assum}[Realizability of the function class]
\label{assum:realizability_F}
We assume
$\forall R\in\Rcal, h\in[H]$, $Q^*_{R,h}\in \Fcal_h + R_h$, where $\Fcal_h+R_h=\{f_h+R_h:f_h\in\Fcal_h\}$.
\end{assum}
\begin{assum}[Completeness]
\label{assum:completeness_F}
We assume $\forall h\in[H]$, $\Tcal_{h}^\zero \Fcal_{h+1},\Tcal_{h}^\zero (\Fcal_{h+1}+\Rcal_{h+1})\subseteq \Fcal_h$ and $\Tcal_{h}^\zero (\Fcal_{h+1}-\Fcal_{h+1})\subseteq \Fcal_h-\Fcal_h$, 
where $\Fcal_h-\Fcal_h=\{f_h-f_h':f_h,f_h'\in\Fcal_h\}$.
\end{assum}

Next we define the covering number, which measures the statistical capacities of function classes. 
\begin{definition}[$\veps$-covering number, e.g., \citet{wainwright2019high}]
We use $\Ncal_\Fcal(\veps)$ to denote the $\veps$-covering number of a set $\Fcal=\Fcal_0\times\ldots\times\Fcal_{H-1}$ under metric $\sigma(f,f')=\max_{h\in[H]}\|f_h-f'_h\|_\infty$ for $f,f'\in\Fcal$. We define it as $\Ncal_\Fcal(\veps) = \min \abr{\Fcal_{\textrm{cover}}}$ such that $\Fcal_{\textrm{cover}} \subseteq \Fcal$ and for any $f \in \Fcal$, there exists $f' \in \Fcal_{\textrm{cover}}$ that satisfies $\sigma(f, f') \le \veps$. For the reward class $\Rcal$, $\Ncal_\Rcal(\veps)$ is defined in the same way.
\end{definition}

Finally, as our guarantees depend on Bellman Eluder (BE) dimensions---which are structural properties of the MDP that enable sample-efficient exploration---we will need the following definitions  \citep[see][]{russo2013eluder,jin2021bellman} which the later definitions of BE dimensions will build on.
\begin{definition}[$\veps$-independence between distributions]
\label{def:ind_dist}
	Let $\Fcal'$ be a function class defined on some space $\Xcal'$, and	$\nu,\mu_1,\ldots,\mu_n$ be probability measures over $\Xcal'$.
	We say 	$\nu$ is $\veps$-independent of $\{\mu_i\}_{i=1}^n$ w.r.t. $\Fcal'$ if  $\exists \, f'\in\Fcal'$ such that
	$\sqrt{\sum_{i=1}^{n} ( \EE_{\mu_i} [f'])^2}\le \veps$, but $|\EE_{\nu}[f']| > \veps$. 
\end{definition}

\begin{definition}[Distributional Eluder (DE) dimension]
\label{def:DE}
Let $\Fcal'$ be a function class defined on some space $\Xcal'$, and $\Gamma'$ be a family of probability measures over $\Xcal'$. 
	The DE dimension $\dde(\Fcal',\Gamma',\veps)$ is the length of the longest sequence $\{\rho_i\}_{i=1}^n \subseteq \Gamma'$ s.t. $\exists \, \veps'\ge\veps$ where $\rho_i$ is $\veps'$-independent of $\{\rho_j\}_{j=1}^{i-1},\forall i= 1,\ldots,n$.
\end{definition}
We also introduce the notation
$\Dcal_{\Fcal}:=\{\Dcal_{\Fcal,h}\}_{h\in[H]}$, where $\Dcal_{\Fcal,h}$ denotes the collection of all possible roll-in distributions at the $h$-th step generated by $\pi_f$ for some $f\in\Fcal$. Formally, $\Dcal_{\Fcal,h} := \{d_h^{\pi_f}\}_{f \in \Fcal}$ where $d_h^{\pi_f}(x,a) = \PP_{\pi_f}[x_h=x,a_h=a]$ is the state-action occupancy measure. 

\section{\rfolive algorithm and results}
In this section, we describe our main algorithm \rfolive, a reward-free variant of \olive \citep{jiang2017contextual,jin2021bellman}.  The algorithmic template for \rfolive is shown in the pseudocode (\pref{alg:rf_olive}) and it can be instantiated with both Q-type and V-type versions of \olive from \citet{jin2021bellman}.\footnote{The Q/V-type algorithms differ in whether to use uniform actions during exploration, and the distinction is needed to handle different settings of interest (see \pref{app:q_vs_v} as well as \Cref{tab:summary}).
} 
In the pseudocode, we use $\square$ as a placeholder for the respective Q/V-type definitions. For clarity, we will describe the Q-type \rfolive algorithm and its results in \pref{sec:main_q} and then state the differences for the V-type version and corresponding results in \pref{sec:main_v}.

Before introducing our algorithm, we define the following average Bellman error:
\begin{definition}[Average Bellman error]
\label{def:abe}
We denote $\EcalR$ as the average Bellman error under reward $R$: 
\begin{align*}
\Ecal^R(f,\pi,\pi',h)=\EE\sbr{f_h(x_h,a_h)-R_h(x_h,a_h)-V_{f}(x_{h+1})\mid a_{0:h-1}\sim\pi, a_h \sim \pi'}.    
\end{align*}
As shorthand, we use $\EcalQ(f,\pi,h) = \Ecal^R(f,\pi,\pi,h)$ to represent the Q-type average Bellman error and $\EcalV(f,\pi,h) = \Ecal^R(f,\pi,\pi_f,h)$ to represent the V-type average Bellman error \citep{jin2021bellman}. We use $\Ecal^\zero$ to represent the average Bellman errors under $\zero$ reward. 
\end{definition}

\begin{algorithm}[htb]
    \caption{\rfolive ($\Fcal,\varepsilon,\delta$): Reward-Free \olive}
    \label{alg:rf_olive}
    \begin{algorithmic}[1]
        \STATEx \textcolor{blue}{Online phase}, no reward information.
        \STATE Set $\eactv,\eelim,\nactv,\nelim$ according to Q-type/V-type and construct 
		$\Fon=\Fcal-\Fcal$. \label{line:set}
		\STATE Initialize $\Fcal^0 \gets \Fon$ (Q-type) or $\Fcal^0 \gets \Zon$, where $\Zon$ is an $(\eelim/64)$-cover of $\Fon$ (V-type).
		\FOR{$t=0,1,\ldots$}
		 \STATE Choose policy $\pi^t = \pi_{f^t}$, where $f^t = \argmax_{f \in \Fcal^t} V_f(x_0)$. \label{line:optimism}
            \STATE Collect $\nactv$ trajectories $\{(x^{(i)}_0,a^{(i)}_0,\ldots, x^{(i)}_{H-1},a^{(i)}_{ H-1})\}_{i=1}^{\nactv}$ by following $\pi^t$ for all $h \in [H]$ and form estimates $\tilde \Ecal^\zero(f^t,\pi^t,\pi^t,h)$ for each $h \in [H]$ via~\cref{eq:on_policy_Qberr}. \label{line:estimate_BE}
            \IF{ $\sum_{h=0}^{H-1} \tilde \Ecal^\zero(f^t,\pi^t,\pi^t,h) \leq H\eactv$}\label{line:valid}
                \STATE Set $T = t$ and exit the loop. \label{line:set_T}
            \ENDIF
            \STATE Pick any $h^t \in [H]$ for which $\tilde \Ecal^\zero(f^t,\pi^t,\pi^t,h^t) > \eactv$. \label{line:deviation}
            \STATE Set $\pi_{\textrm{est}} = \pi^t$ (Q-type) or $\pi_{\textrm{est}} = \textrm{Unif}(\Acal)$, i.e., draw actions uniformly at random (V-type).
            \STATE Collect $\nelim$ samples $\Dcal^t = \{(x_{h^t}^{(i)}, a_{h^t}^{(i)}, x_{h^t+1}^{(i)})\}_{i=1}^{\nelim}$ where $a_{0:h^t-1} \sim \pi^t$ and $a_{h^t} \sim \pi_{\textrm{est}}$.
            \label{line:collect_nelim}
            \STATE For all $f \in \Fcal^t$, compute estimate $\hEcalsqzero(f,\pi^t,h^t)$ via~\cref{eq:est_on_policy_Qberr} (Q-type) or \cref{eq:est_on_policy_Vberr} (V-type).\label{line:estimate_allBE}
            \STATE Update $\Fcal^{t+1} = \{f \in \Fcal^t: |\hEcalsqzero(f,\pi^t,h^t)|\leq \eelim \}$. \label{line:update_vspace}
		\ENDFOR
		\STATE Save the collected tuples $\{(h^t,\pi^t,\Dcal^t)\}_{t=0}^{T-1}$ for the offline phase.
		\STATEx \textcolor{blue}{Offline phase}, the reward function $R = (R_0,\ldots,R_{H-1})$ is revealed.
		\STATE Construct $\Foff(R)=\Fcal + R$, set $\Pi_{\textrm{est}}^t = \{\pi^t\}$ (Q-type) or \hspace{-.2em} $\Pi_{\textrm{est}}^t =\Pon\coloneqq \cbr{\pi_f: f \in \Zon}$ (i.e., the greedy policies induced by $\Zon $) (V-type). \label{line:elim_Pi}
		
		\STATE For each $t \in [T]$, $g \in \Foff(R)$, and $\pi \in \Pi_{\textrm{est}}^t$, compute estimate $\hat \Ecal^R(g,\pi^t,\pi,h^t)$ via~\cref{eq:Qest_bellman_errs} (Q-type) or \cref{eq:Vest_bellman_errs} (V-type). 
		\STATE Set $\Fsur(R) = \{g \in \Foff(R): \forall t \in [T], \forall \pi \in \Pi_{\textrm{est}}^t, |\hat\Ecal^R(g,\pi^t,\pi,h^t)| \leq \eelim/2\} $. \label{line:off_elim}
		\STATE Return policy $\hat \pi = \pi_{\hat g}$, where $\hat{g} = \argmax_{g \in \Fsur(R)} V_g(x_0)$.\label{line:return}
    \end{algorithmic}
\end{algorithm}

\subsection{Q-type \rfolive}
\label{sec:main_q}

Our algorithm, reward-free \olive (\rfolive) described in \pref{alg:rf_olive}, takes the function class $\Fcal$, the accuracy parameter $\veps$, and the failure probability $\delta$ as input. As we are in the reward-free setting, it operates in two phases: an online exploration phase where it collects a dataset without an explicit reward signal, and an offline phase where it computes a near-optimal policy after the reward function $R$ is revealed. Below, we describe the two phases and the intuition behind the algorithm design.

\paragraph{Online exploration phase} During the online phase, we first set elimination thresholds $\eactv,\eelim$ and sample sizes $\nactv,\nelim$ and construct the following function class $\Fon$ used in the online phase:
\begin{align*}
\Fon = \Fcal-\Fcal \defeq \cbr{(f_0-f_0',\ldots,f_{H-1}-f_{H-1}'):f_h,f_h'\in\Fcal_h, \forall h\in[H]}.
\end{align*}
The detailed specification of these parameters are deferred to \pref{thm:rfolive_q} and \pref{thm:rfolive_v}. Subsequently, we simulate Q-type \olive with the function class $\Fon$ using the zero reward function $R=\zero$ and the specified parameters. Similar to \olive, we initialize $\Fcal^0 = \Fon$ and maintain a version space $\Fcal^t\subseteq \Fcal^{t-1}\subseteq \Fon$ of surviving functions after each iteration. In each iteration, we first find the optimistic function $f^t \in \Fcal^t$ (\pref{line:optimism}) and set $\pi^t = \pi_{f^t}$. In \pref{line:estimate_BE}, we collect $\nactv$ trajectories to estimate the Q-type average Bellman error $\tEcalQzero(f^t,\pi^t,h) = \tilde \Ecal^\zero(f^t,\pi^t,\pi^t,h)$ under zero reward: 
\begin{align}
    \label{eq:on_policy_Qberr}
    \tilde \Ecal^\zero(f^t,\pi^t,\pi^t,h) = \frac{1}{\nactv}\sum_{i=1}^{\nactv} \sbr{f^t_h\rbr{x^{(i)}_h,a^{(i)}_h} - V_{f^t}\rbr{x^{(i)}_{h+1}}}.
\end{align}
If the low average Bellman error condition in \pref{line:valid} is satisfied, then we terminate the online phase and otherwise, we pick a step $h^t$ where the estimate $\tilde \Ecal^\zero(f^t,\pi^t,\pi^t,h^t) > \eactv$ (\pref{line:deviation}). Then we collect $\nelim$ trajectories using $a_{0:h^t} \sim \pi^t$ and set $\Dcal^t$ as the transition tuples at step $h^t$. Using $\Dcal^t$, we construct the Q-type average Bellman error estimates $\hEcalQzero(f,\pi^t,h^t)$ for all $f \in \Fcal^t$ in \pref{line:estimate_allBE}:
\begin{align}
    \label{eq:est_on_policy_Qberr}
    \hEcalQzero(f,\pi^t,h^t) = \frac{1}{\nelim}\sum_{i=1}^{\nelim} \sbr{f_{h^t}\rbr{x^{(i)}_{h^t},a^{(i)}_{h^t}} - V_f\rbr{x^{(i)}_{h^t+1}}}.
\end{align}
Finally, in \pref{line:update_vspace}, we eliminate all the $f \in \Fcal^t$ whose average Bellman error estimate $\hEcalQzero(f,\pi^t,h^t) > \eelim$.

The online phase returns tuples $\{(h^t,\pi^t,\Dcal^t)\}_{t=0}^{T-1}$ where $T$ is the total number of iterations and each dataset $\Dcal^t$ consists of $\nelim$ transition tuples. 

\paragraph{Offline elimination phase}  In the offline phase, the reward function $R$ is revealed, and we first construct the reward-appended function class $\Foff(R)=\Fcal+R\defeq \{(f_0+R_0,\ldots,f_{H-1}+R_{H-1}):f_h\in\Fcal_h,\forall h\in[H]\}$. Using the class $\Pi_{\textrm{est}}^t = \cbr{\pi^t}$ from \pref{line:elim_Pi} and the collected tuples $\{(h^t,\pi^t,\Dcal^t)\}_{t=0}^{T-1}$, we estimate the reward-dependent average Bellman error (\pref{def:abe}) for all iterations $t \in [T]$ of the online phase:
\begin{align}
    \hat \Ecal^R(g, \pi^t,\pi^t,h^t) = \frac{1}{\nelim} \sum_{i=1}^{\nelim} \sbr{g_{h^t}\rbr{x_{h^t}^{(i)},a_{h^t}^{(i)}} - R_{h^t}\rbr{x_{h^t}^{(i)},a_{h^t}^{(i)}} - V_g\rbr{x_{h^t+1}^{(i)}}}.
    \label{eq:Qest_bellman_errs}
\end{align}
\rfolive eliminates all $g \in \Foff(R)$ whose average Bellman error estimates are large (\pref{line:off_elim}) and returns the optimistic function $\hat g$ from the surviving set (\pref{line:return}).

\paragraph{Remark} Similar to its counterparts in reward-aware general function approximation setting \citep{jiang2017contextual,dann2018oracle,jin2021bellman,du2021bilinear}, \rfolive is in general not computationally efficient. We leave addressing computational tractability as a future direction.

\subsubsection{Main results for Q-type \rfolive}
\label{sec:main_general_q}
In this part, we present the theoretical guarantee of Q-type \rfolive. We start with introducing the Q-type Bellman Eluder (BE) dimension \citep{jin2021bellman}.
\begin{definition}[Q-type BE dimension]
\label{def:dim_qbe}
Let $(I-\Tcal_h^R)\Fcal:=\{f_h-\Tcal_h^R f_{h+1}: \ f\in\Fcal\}$ be the set of Bellman differences of $\Fcal$ at step $h$, and 
$\Gamma=\{\Gamma_h\}_{h=0}^{H-1}$ 
where $\Gamma_h$ is a set of distributions over $\Xcal\times\Acal$.
The $\veps$-BE dimension of $\Fcal$ w.r.t. $\Gamma$ is  defined as 
$\dqbe^R(\Fcal,\Gamma,\veps) := 
	\max_{h\in[H]} \dde\big((I-\Tcal_h^R)\Fcal,\Gamma_h,\veps\big).$
\end{definition}

We can now state our sample complexity result for Q-type \rfolive.
To simplify presentation, we state the result here assuming parametric growth of the covering numbers, that is $\log(\Ncal_{\Fcal}(\veps)) \leq d_{\Fcal} \log(1/\veps)$ and $\log(\Ncal_{\Rcal}(\veps)) \leq d_{\Rcal} \log(1/\veps)$. 
\begin{theorem}[Q-type \rfolive, parametric case]
\label{thm:rfolive_q}
Fix $\delta \in (0,1)$. Given a reward class $\Rcal$ and a function class $\Fcal$ that satisfies \pref{assum:realizability_F} and \pref{assum:completeness_F}, with probability at least $1-\delta$, for any $R \in \Rcal$, Q-type \rfolive (\pref{alg:rf_olive}) outputs a policy $\hat\pi$ that satisfies $v^{\hat \pi}_R \ge v^*_R - \veps$. The required number of episodes is\footnote{The $\tilde{O}(\cdot)$ notation suppresses poly-logarithmic factors in its argument.}
\begin{align*}
    \tilde{O}\rbr{\rbr{ H^7 d_{\Fcal} + H^5 d_{\Rcal}} \dqbee^2 \log(1/\delta)/\veps^2},
\end{align*}
where $\dqbee=\dqbe^\zero(\Fcal-\Fcal,\Dcal_{\Fcal-\Fcal},\veps/(4H))$.
\end{theorem}
The more general statement along with the specific values of $\eactv,\eelim,\nactv,\nelim$ are deferred to \pref{app:general_q}, where we also present the proof.
We remark that we only need the covering number of $\Rcal$ to set these parameters and do not use any other information about the reward class.

We pause to compare~\pref{thm:rfolive_q} to the reward-aware case. First, our BE dimension involves the ``difference'' function class $\Fcal - \Fcal$ under zero reward as opposed to the original class with the given reward, and our completeness assumption is also related to such a ``difference'' function class. As we will see, these differences are inconsequential for our examples of interest. Second, our sample complexity has an additional $H^4$ dependence because (a) we consider a different reward normalization from~\citet{jiang2017contextual,jin2021bellman} and (b) we use a smaller threshold in the online phase to ensure sufficient exploration. Similar gaps in $H$ factors between reward-free and reward-aware learning also appear in \citet{wagenmaker2022reward}. We also pay for the complexity of $\Rcal$ in a lower order term, which is standard in reward-free RL \citep{zhang2020learning,modi2021model}.
We believe that a similar adaptation of \golf \citep{jin2021bellman} for the reward-free setting may provide a sharper result with improved dependence on $H$ and $\dqbee$, analogously to the reward-aware setting.

\subsubsection{Q-type \rfolive for known representation linear completeness setting}
\label{sec:main_linear_completeness}
Here, we instantiate the general guarantee of Q-type \rfolive to the linear completeness setting.\footnote{\citet{zanette2020provably} only define linear completeness for $B=1$. It can be easily verified that it is equivalent for any choice of $B$. More discussion can be found in \pref{app:linear_completeness}.}
\begin{definition}[Linear completeness setting \citep{zanette2020provably}]
\label{def:lin_completeness}
We call feature $\philc=(\philc_0,\ldots,\philc_{H-1})$ with $\philc_h:\Xcal\times\Acal\rightarrow \RR^{\dlc},\|\philc_h(\cdot)\|_2\le 1,\forall h\in[H]$ a linearly complete feature, if for any $B>0, h \in [H-1]$ and $\forall f_{h+1}\in \Qcal_{h+1}(\{\philc\},B)$ we have: $\min_{f_h\in \Qcal_{h}(\{\philc\},B)}\|f_h-\Tcal_{h}^\zero f_{h+1}\|_{\infty}=0$, where $\Qcal_h(\{\philc\},B)=\{\langle \philc_h,\theta_h\rangle:\|\theta_h\|_2\le B\sqrt {\dlc}\}$.
\end{definition}

When the linearly complete features (\pref{def:lin_completeness}) $\philc$ are known, we can construct the function class  $\Fcal(\{\philc\})=\Fcal_0(\{\philc\},H-1)\times\ldots\times \Fcal_{H-1}(\{\philc\},0)$, where $\Fcal_h(\{\philc\},B_h) = \big\{f_h(x_h,a_h) = \inner{\philc_h(x_h,a_h)}{\theta_h} : \|\theta_h\|_2 \le B_h\sqrt{\dlc},\langle \philc_h(\cdot),\theta_h\rangle \in[-B_h, B_h] \big\}$ consists of appropriately bounded linear functions of $\philc$. Here superscript and subscript $\mathrm{lc}$ imply that the notations are related to the linear completeness setting. It is easy to verify that $\Fcal(\{\philc\})$ satisfies the assumptions in \pref{thm:rfolive_q}. This gives us the following corollary (see the full statement and the proof in \pref{app:linear_completeness}):

\begin{corollary}[\textit{Informal}]
\label{corr:linear_complete}
Fix $\delta \in (0,1)$. Consider an MDP $M$ that satisfies linear completeness (\pref{def:lin_completeness}) with known feature $\philc$, and the linear reward class $\Rcal=\Rcal_1\times\ldots\times\Rcal_h$, where $\Rcal_h=\cbr{\langle\philc_h,\eta_h\rangle:\|\eta_h\|_2\le\sqrt{\dlc},\langle\philc_h(\cdot),\eta_h\rangle\in[0,1] }$. 
With probability at least $1-\delta$, for any $R\in\Rcal$, Q-type \rfolive (\pref{alg:rf_olive}) with $\Fcal=\Fcal(\{\philc\})$ outputs a policy $\hat\pi$ that satisfies $v^{\hat \pi}_R \ge v^*_R - \veps$ . The required number of samples is $\tilde O\rbr{H^8\dlc^3\log(1/\delta)/\varepsilon^2}$.
\end{corollary}

The reward normalization above, called \emph{explicit regularity} in~\citet{zanette2020provably}, is standard. Compared to that work, our result
implies that \emph{explorability} is not necessary,
which significantly relaxes the existing assumptions for this setting. Our result can also be easily extended to handle approximately linearly complete features (i.e., low inherent Bellman error). 
On the other hand, our algorithm is not computationally efficient owing to our general function approximation setting. Although our sample complexity bound \emph{appears} to be worse in $H$ factors compared with their upper bound of $\tilde{O}\rbr{\dlr^3H^5\log(1/\delta)/\veps^2}$, it is indeed incomparable because their bound only holds when $\veps\le \tilde O(\nu_{\min}/\sqrt {\dlc})$ ($\nu_{\min}$ is their explorability factor). Thus, there is an implicit dependence on $1/\nu_{\min}$ in their result, which could make the bound arbitrarily worse than ours. More discussions are deferred to \pref{app:comparison} and \pref{app:linear_completeness}.

\subsection{V-type \rfolive}
\label{sec:main_v}
In this section, we describe the instantiation of \rfolive with V-type definitions. 
For V-type \rfolive, we also assume that the action space is finite with size $K$. 
\paragraph{Online exploration phase} Instead of using $\Fon$, we use its $(\eelim/64)$-cover $\Zon$ and maintain a version space $\Fcal^t$ across iterations.\footnote{Following \citet{jin2021bellman}, we run V-type \olive with the discretized class $\Zon$ for the ease of presentation.} Since the on-policy version of Q-type and V-type Bellman errors are the same, the termination check in~\pref{line:estimate_BE} and~\pref{line:valid} are unchanged. 
If the algorithm does not terminate in \pref{line:valid}, we again identify a deviation step $h^t$ such that $\tEcalVzero(f^t, \pi^t, h^t) =\tilde \Ecal^\zero(f^t,\pi^t,\pi^t,h^t) > \eactv$. Instead of using $\pi^t$ to collect trajectories, we use $a_{0:h^t-1} \sim \pi^t$ and choose $a_{h^t}$ uniformly at random to collect the dataset of $\nelim$ transition tuples at step $h^t$. Compared to Q-type \rfolive, we estimate $\hEcalVzero$ for all $f\in\Fcal^t$ in \pref{line:estimate_allBE} using importance sampling (IS):
\begin{align}
    \label{eq:est_on_policy_Vberr}
   \hEcalVzero(f,\pi^t,h^t) = \frac{1}{\nelim}\sum_{i=1}^{\nelim} \frac{\mathbf{1}[a_{h^t}^{(i)} = \pi_f(x_{h^t}^{(i)})]}{1/K} \sbr{f_{h^t}\rbr{x^{(i)}_{h^t},a^{(i)}_{h^t}} - V_{f}\rbr{x^{(i)}_{h^t+1}}}.
\end{align}
Finally, in \pref{line:update_vspace}, we eliminate all $f \in \Fcal^t$ whose V-type average Bellman error estimates are large. 

\paragraph{Offline elimination phase} In the offline phase, we consider the same reward-appended function class $\Foff(R)$ when reward $R \in \Rcal$ is revealed. For V-type \rfolive, in \pref{line:elim_Pi}, we define the policy class $\Pi^t_{\mathrm{est}} = \Pon$ which consists of greedy policies with respect to all $f\in\Zon$. Using dataset $\Dcal^t$, we estimate $\Ecal^R(g,\pi^t, \pi', h^t)$ for all $g\in\Foff(R),\pi'\in\Pon,t\in[T]$ from its empirical version: 	
\begin{equation}
\hat\Ecal^R(g, \pi^t, \pi', h^t)= \frac{1}{\nelim}\sum_{i=1}^{\nelim} \frac{\one[a_{h^t}^{(i)} = \pi'(x_{h^t}^{(i)})]}{1/K} \sbr{g_h(x_h^{(i)},a_h^{(i)}) - R_h(x_h^{(i)},a_h^{(i)}) - V_g(x_{h+1}^{(i)})}
\label{eq:Vest_bellman_errs}
\end{equation}
and eliminate invalid functions in \pref{line:off_elim}. Finally, we return the optimistic policy $\hat \pi$ from the surviving set. Apart from estimating different average Bellman errors, the noticeable difference between Q-type and V-type \rfolive is that the latter uses IS to correct the uniformly drawn action to some policy $\pi'\in\Pon$ to witness the average Bellman error \citep{jiang2017contextual}.

\subsubsection{Main results for V-type \rfolive}
\label{sec:main_general_v}
Here we present the theoretical guarantee of V-type \rfolive. Firstly, we introduce the V-type Bellman Eluder (BE) dimension \citep{jin2021bellman}.
\begin{definition}[V-type BE dimension]
\label{def:dim_vbe}
Let $(I-\Tcal_h^R)V_{\Fcal} \subseteq (\Xcal \rightarrow \mathbb{R})$ be the state-wise Bellman difference class of $\Fcal$ at step $h$ defined as 
$	(I-\Tcal_h^R)V_\Fcal := \big\{ x \mapsto (f_h-\Tcal_h^R f_{h+1})(x,\pi_{f_h}(x)) : f \in \Fcal\big\}.$
Let $\Gamma=\{\Gamma_h\}_{h=0}^{H-1}$ where $\Gamma_h$ is a set of distributions over $\Xcal$. The {V-type $\veps$-BE  dimension} of  $\Fcal$ with respect to $\Gamma$ is  defined as
$	\dvbe^R(\Fcal,\Gamma,\veps) := 
	\max_{h\in[H]} \dde\big((I-\Tcal_h^R)V_\Fcal,\Gamma_h,\veps\big).$
\end{definition}

We now state the guarantee for V-type \rfolive, assuming polynomial covering number growth.
\begin{theorem}[V-type \rfolive, parametric case]
\label{thm:rfolive_v}
Fix $\delta \in (0,1)$. Given a reward class $\Rcal$, a function class $\Fcal$ that satisfies \pref{assum:realizability_F}, \pref{assum:completeness_F}, with probability at least $1-\delta$, for any $R\in\Rcal$, V-type \rfolive outputs a policy $\hat\pi$ that satisfies $v^{\hat \pi}_R \ge v^*_R - \veps$. The required number of episodes is
\begin{align*}
	\tilde O\rbr{\rbr{H^7d_\Fcal+H^5d_\Rcal}\dvbee^2K\log(1/\delta)/\varepsilon^2},
\end{align*}
where $\dvbee=\dvbe^\zero(\Fcal-\Fcal,\Dcal_{\Fcal-\Fcal},\veps/(8H))$.
\end{theorem}

The detailed proof and the specific values of $\eactv,\eelim,\nactv,\nelim$ are deferred to \pref{app:general_v}. Our rate is again loose in $H$ factors when compared with the reward-aware version.
Compared with the Q-type version, here we also incur a dependence on $K = \abr{\Acal}$, analogous to the reward-aware case. 

\subsubsection{V-type \rfolive for unknown representation low-rank MDPs}
\label{sec:main_low_rank}
As a special case, we instantiate our V-type \rfolive result to low-rank MDPs \citep{modi2021model}:

\begin{definition}[Low-rank factorization]
\label{def:lowrank}
A transition operator $P_h:\Xcal\times\Acal\rightarrow \Delta(\Xcal)$ admits a low-rank decomposition of dimension $\dlr$ if there exists $\philr_h: \Xcal \times \Acal \rightarrow \RR^{\dlr}$ and $\mulr_h: \Xcal \rightarrow \RR^{\dlr}$ s.t. $\forall x, x' \in \Xcal , a \in \Acal: \, P_h(x'\mid x,a) = \inner{ \philr_h(x,a)}{\mulr_h(x')}$,
and additionally $\|\philr_h(\cdot)\|_2 \le 1$ and $\forall f': \Xcal \rightarrow [-1,1]$, we have $\left\|\int f'(x) \mulr_h(x) dx\right\|_2 \le \sqrt{\dlr}$.
We say $M$ is low-rank with embedding dimension $\dlr$, if for each $h \in [H]$, the transition operator $P_h$ admits a rank-$\dlr$ decomposition.
\end{definition}

Here superscript and subscript $\mathrm{lr}$ imply that the notations are related to low-rank MDPs. As in \citet{modi2021model}, we consider low-rank MDPs in a representation learning setting, where we are given realizable feature class $\Philr$ rather than the feature $\philr=(\philr_0,\ldots,\philr_{H-1})$ directly:
\begin{assum}[Realizability of low-rank feature class]
\label{assum:realizability_low_rank}
We assume that a finite feature class $\Philr =  \Philr_0\times\ldots\times\Philr_{H-1}$ satisfies $\philr_h\in\Philr_h$, $\forall h \in [H]$. In addition, $\forall h\in[H], \phi_h\in\Philr_h$, $\|\phi_h(\cdot)\|_2\le 1$.
\end{assum}

Similar to the linear completeness setting (\pref{sec:main_linear_completeness}), we construct 
$\Fcal(\Philr)$ $=\Fcal_0(\Philr,H-1)\times$ $\ldots$ $\times \Fcal_{H-1}(\Philr,0)$, where 
$\Fcal_h(\Philr,B_h)\hspace{-.2em}=\hspace{-.2em}\{f_h(x_h,a_h) = \inner{\phi_h(x_h,a_h)}{\theta_h} :\phi_h\in\Philr_h,  \hspace{-.2em} \|\theta_h\|_2 \le B_h\sqrt{\dlr},$ $\langle\phi_h(\cdot),\theta_h\rangle\in [-B_h,B_h] \}$. In \pref{prop:lr_vbe}, we show that the V-type Bellman Eluder dimension of $\Fcal(\Philr)-\Fcal(\Philr)$ in this case is $\tilde O \rbr{\dlr}$ which leads to the following corollary:

\begin{corollary}[\textit{Informal, parametric case}]
\label{corr:low_rank}
Fix $\delta \in (0,1)$. Consider a low-rank MDP $M$ of embedding dimension $\dlr$ with a realizable feature class $\Philr$ (\pref{assum:realizability_low_rank}) and a reward class $\Rcal$.
With probability at least $1-\delta$, for any $R \in \Rcal$, V-type \rfolive (\pref{alg:rf_olive}) with $\Fcal(\Philr)$ outputs a policy $\hat\pi$ that satisfies $v^{\hat \pi}_R \ge v^*_R - \veps$ . The required number of episodes is
\begin{align*}
\tilde O\rbr{\rbr{H^8\dlr^3\log(|\Philr|)+H^5\dlr^2 d_\Rcal} K\log(1/\delta)/\varepsilon^2}.
\end{align*}
\end{corollary}
We defer the full statement and detailed proof of the corollary to \pref{app:low_rank}. In the low-rank MDP setting, \citet{modi2021model} propose a more computationally viable algorithm, but additionally require a reachability assumption. 
Our result shows that reachability is not necessary for statistically efficiency, which opens an interesting avenue for designing an algorithm that is both computationally and statistically efficient without reachability. 
Moreover, our result significantly improves upon their sample complexity bound. The detailed comparisons are deferred to \pref{app:comparison} and \pref{app:low_rank}. Notice that here $K$ shows up in our bound. As another corollary, in the linear MDP (\pref{def:lowrank} plus $\philr$ is known), Q-type \rfolive yields a $K$ independent bound. The details can be found in \pref{app:linear_mdp}.

\subsection{Intuition and proof sketch for \rfolive}
We first provide the intuition. Since the online phase of \rfolive is equivalent to running \olive with $\zero$ reward function, 
any policy $\pi_f$ attains zero value (i.e., $V^{\pi_f}_{\zero,0}(x_0)= 0,\forall f\in\Fon$). By the policy loss decomposition lemma \citep{jiang2017contextual}, the value error for the greedy policy, $V_f(x_0) - V^{\pi_f}_{\zero,0}(x_0)$, is small when the algorithm stops (\pref{line:valid}). Therefore, all $f\in\Fon$ which predict large values $V_f(x_0)$ must have been eliminated before \olive terminates. This implies that, in the online phase, we gather a diverse set of constraints (roll-in distributions $\pi^t$) that can witness the average Bellman error of functions in $\Fon$.
In this sense, our algorithm focuses on function space elimination and does not try to reach all latent states or directions \citep{modi2021model,zanette2020provably}, which is the key conceptual difference that enables us to avoid reachability and explorability assumptions.

On the technical side, note that the way we use \olive in the offline setting is novel to our knowledge and is crucial to getting a good sample complexity under our assumptions, as opposed to more standard FQI style approaches. Because we have to coordinate between the online and offline phases, the analysis bears significant novelty beyond the original analysis of \olive (and its reward-aware follow-up works), and this is one of our key contributions. The most crucial part is to show that any bad $g\in \Foff(R)$ whose average Bellman error is large under the true reward $R$ will be eliminated in the offline phase. To prove this, we construct $\tilde f\in\Fon$ that has the same average Bellman error as $g$ 
and predicts a large positive value $V_{\tilde f}(x_0)$, which implies that it will be eliminated during the online phase. 
Finally, by our construction, the constraint used to eliminate $\tilde{f}$ directly witnesses the average Bellman error of $g$, thus ruling out $g$ in the offline phase. We discuss it in more detail in \pref{app:novelty}.

\section{Hardness result for unknown representation linear completeness setting}
\label{sec:main_hard}
In \pref{sec:main_linear_completeness}, we showed that Q-type \rfolive requires polynomial sample for reward-free RL in the known feature linear completeness setting. For low-rank MDPs, when given a realizable feature class, we showed V-type \rfolive is statistically efficient in \pref{sec:main_low_rank}. A natural next step is to relax the low-rank assumption on the MDP and show a sample efficiency result for the more general linear completeness and unknown feature case. However, below we state a hardness result which shows that a polynomial dependence on the feature class ($\abr{\Philc}$) 
or an exponential dependence on $H$ is \emph{unavoidable}. We first introduce the realizability of a linearly complete feature class. 

\begin{assum}[Realizability of the linearly complete feature class]
\label{assum:realizability_linear_completeness}
We assume that there exists a finite candidate feature class $\Philc=\Philc_0\times \dots\times \Philc_{H-1}$, such that $\forall h\in[H]$, we have $\philc_h\in\Philc_h$. In addition, $\forall h\in[H], \phi_h\in\Philc_h,\forall(x,a)\in\Xcal\times\Acal$, we have $\|\phi_h(\cdot)\|_2\le 1$.
\end{assum}

Now we state of hardness result for learning in the linear completeness setting with a realizable feature class (\pref{assum:realizability_linear_completeness}). A complete proof and more discussions are provided in \pref{app:rf_lower_bound}.
\begin{theorem}
\label{thm:rf_lower_bound}
There exists a family of MDPs $\Mcal$, a reward class $\Rcal$ and a feature set $\Philc$, such that  $\forall M\in\Mcal$, the $(M, \Philc)$ pair satisfies \pref{assum:realizability_linear_completeness}, yet it is information-theoretically impossible for an algorithm to obtain a $\poly\rbr{\dlc, H, \log (|\Philc|), \log (|\Rcal|), 1/\veps, \log (1/\delta)}$ sample complexity for reward-free exploration with the given reward class $\Rcal$.
\end{theorem}

The hardness result in \pref{thm:rf_lower_bound} is also applicable to easier settings: 
(i) learning with a generative model (or using a local access protocol, \citet{hao2022confident}), (ii) reward-free learning with explorability \citep{zanette2020provably} and reachability \citep{modi2021model} assumptions and (iii) reward-aware learning as $\Rcal$ is a known singleton class. Thus, the result highlights an exponential separation between the low-rank MDP and linear completeness assumptions by showing that linearly complete true feature $\philc \in \Philc$ is not sufficient for polynomial sample efficiency and additional assumptions are required to account for the unknown representation.

\section{Conclusion and discussion}
In this paper, we investigated the statistical efficiency of reward-free RL under general function approximation. The proposed algorithm, \rfolive, is the first algorithm to address reward-free exploration under general function approximation. Contrary to prior works which either try to reach all states or all directions in the feature space, \rfolive follows a value function elimination template and ensures that the collected exploration data can be used to identify and eliminate non-optimal value functions for downstream planning. This significantly sets us apart from the existing reward-free exploration works. Our positive results significantly relax the existing assumptions in the reward-free exploration framework. Our negative result shows the first sharp separation between low-rank MDP and the linear completeness settings with unknown representations. 
In addition, we provide an algorithm-specific counterexample in \pref{app:alg_counterexp} that shows \rfolive can fail when the completeness assumption is violated. 
As realizability alone is sufficient for reward-aware RL \citep{jiang2017contextual,jin2021bellman,du2021bilinear}, our results also elicit the further question: 
\begin{center}
\textbf{Are realizability-type assumptions sufficient for statistically efficient reward-free RL?}
\end{center}
We conjecture that the answer is no, and we believe that the hardness between reward-aware and reward-free RL has a deep connection to the sharp separation between realizability and completeness \citep{chen2019information,wang2020statistical,wang2021exponential,xie2021batch,weisz2021query,weisz2021exponential,weisz2022tensorplan,foster2021offline}.

\section*{Acknowledgements}
JC would like to thank Tengyang Xie for helpful discussions. Part of this work was done while AM was at University of Michigan and was supported in part by a grant from the Open Philanthropy Project to the Center for Human-Compatible AI, and in part by NSF grant CAREER IIS-1452099. NJ acknowledges funding support from ARL Cooperative Agreement W911NF-17-2-0196, NSF IIS-2112471, NSF CAREER IIS-2141781, and Adobe Data Science Research Award.

\bibliography{refs}

\begin{thebibliography}{48}
\providecommand{\natexlab}[1]{#1}
\providecommand{\url}[1]{\texttt{#1}}
\expandafter\ifx\csname urlstyle\endcsname\relax
  \providecommand{\doi}[1]{doi: #1}\else
  \providecommand{\doi}{doi: \begingroup \urlstyle{rm}\Url}\fi

\bibitem[Agarwal and Zhang(2022)]{agarwal2022non}
Alekh Agarwal and Tong Zhang.
\newblock Non-linear reinforcement learning in large action spaces: Structural
  conditions and sample-efficiency of posterior sampling.
\newblock \emph{arXiv preprint arXiv:2203.08248}, 2022.

\bibitem[Agarwal et~al.(2020)Agarwal, Kakade, Krishnamurthy, and
  Sun]{agarwal2020flambe}
Alekh Agarwal, Sham Kakade, Akshay Krishnamurthy, and Wen Sun.
\newblock Flambe: Structural complexity and representation learning of low rank
  mdps.
\newblock \emph{Advances in Neural Information Processing Systems}, 2020.

\bibitem[Bai and Jin(2020)]{bai2020provable}
Yu~Bai and Chi Jin.
\newblock Provable self-play algorithms for competitive reinforcement learning.
\newblock In \emph{International Conference on Machine Learning}, 2020.

\bibitem[Chen and Jiang(2019)]{chen2019information}
Jinglin Chen and Nan Jiang.
\newblock Information-theoretic considerations in batch reinforcement learning.
\newblock In \emph{International Conference on Machine Learning}, 2019.

\bibitem[Chen and Jiang(2022)]{chen2022offline}
Jinglin Chen and Nan Jiang.
\newblock Offline reinforcement learning under value and density-ratio
  realizability: {T}he power of gaps.
\newblock In \emph{Conference on Uncertainty in Artificial Intelligence}, 2022.

\bibitem[Dann et~al.(2018)Dann, Jiang, Krishnamurthy, Agarwal, Langford, and
  Schapire]{dann2018oracle}
Christoph Dann, Nan Jiang, Akshay Krishnamurthy, Alekh Agarwal, John Langford,
  and Robert~E Schapire.
\newblock On oracle-efficient {PAC} {RL} with rich observations.
\newblock In \emph{Advances in Neural Information Processing Systems}, 2018.

\bibitem[Du et~al.(2019)Du, Krishnamurthy, Jiang, Agarwal, Dudik, and
  Langford]{du2019provably}
Simon Du, Akshay Krishnamurthy, Nan Jiang, Alekh Agarwal, Miroslav Dudik, and
  John Langford.
\newblock Provably efficient rl with rich observations via latent state
  decoding.
\newblock In \emph{International Conference on Machine Learning}, 2019.

\bibitem[Du et~al.(2021)Du, Kakade, Lee, Lovett, Mahajan, Sun, and
  Wang]{du2021bilinear}
Simon Du, Sham Kakade, Jason Lee, Shachar Lovett, Gaurav Mahajan, Wen Sun, and
  Ruosong Wang.
\newblock Bilinear classes: A structural framework for provable generalization
  in {RL}.
\newblock In \emph{International Conference on Machine Learning}, 2021.

\bibitem[Ernst et~al.(2005)Ernst, Geurts, and Wehenkel]{ernst2005tree}
Damien Ernst, Pierre Geurts, and Louis Wehenkel.
\newblock Tree-based batch mode reinforcement learning.
\newblock \emph{Journal of Machine Learning Research}, 2005.

\bibitem[Foster et~al.(2021)Foster, Krishnamurthy, Simchi-Levi, and
  Xu]{foster2021offline}
Dylan~J Foster, Akshay Krishnamurthy, David Simchi-Levi, and Yunzong Xu.
\newblock Offline reinforcement learning: Fundamental barriers for value
  function approximation.
\newblock \emph{arXiv:2111.10919}, 2021.

\bibitem[Hao et~al.(2022)Hao, Lazic, Yin, Abbasi-Yadkori, and
  Szepesvari]{hao2022confident}
Botao Hao, Nevena Lazic, Dong Yin, Yasin Abbasi-Yadkori, and Csaba Szepesvari.
\newblock Confident least square value iteration with local access to a
  simulator.
\newblock In \emph{International Conference on Artificial Intelligence and
  Statistics}, 2022.

\bibitem[Huang et~al.(2021)Huang, Chen, Zhao, Qin, Jiang, and
  Liu]{huang2021towards}
Jiawei Huang, Jinglin Chen, Li~Zhao, Tao Qin, Nan Jiang, and Tie-Yan Liu.
\newblock Towards deployment-efficient reinforcement learning: Lower bound and
  optimality.
\newblock In \emph{International Conference on Learning Representations}, 2021.

\bibitem[Jiang and Agarwal(2018)]{jiang2018open}
Nan Jiang and Alekh Agarwal.
\newblock Open problem: The dependence of sample complexity lower bounds on
  planning horizon.
\newblock In \emph{Conference On Learning Theory}, 2018.

\bibitem[Jiang and Huang(2020)]{jiang2020minimax}
Nan Jiang and Jiawei Huang.
\newblock Minimax value interval for off-policy evaluation and policy
  optimization.
\newblock \emph{Advances in Neural Information Processing Systems}, 2020.

\bibitem[Jiang et~al.(2017)Jiang, Krishnamurthy, Agarwal, Langford, and
  Schapire]{jiang2017contextual}
Nan Jiang, Akshay Krishnamurthy, Alekh Agarwal, John Langford, and Robert~E
  Schapire.
\newblock Contextual decision processes with low {B}ellman rank are
  {PAC}-learnable.
\newblock In \emph{International Conference on Machine Learning}, 2017.

\bibitem[Jin et~al.(2020{\natexlab{a}})Jin, Krishnamurthy, Simchowitz, and
  Yu]{jin2020reward}
Chi Jin, Akshay Krishnamurthy, Max Simchowitz, and Tiancheng Yu.
\newblock Reward-free exploration for reinforcement learning.
\newblock In \emph{International Conference on Machine Learning},
  2020{\natexlab{a}}.

\bibitem[Jin et~al.(2020{\natexlab{b}})Jin, Yang, Wang, and
  Jordan]{jin2019provably}
Chi Jin, Zhuoran Yang, Zhaoran Wang, and Michael~I Jordan.
\newblock Provably efficient reinforcement learning with linear function
  approximation.
\newblock In \emph{Conference on Learning Theory}, 2020{\natexlab{b}}.

\bibitem[Jin et~al.(2021)Jin, Liu, and Miryoosefi]{jin2021bellman}
Chi Jin, Qinghua Liu, and Sobhan Miryoosefi.
\newblock Bellman eluder dimension: New rich classes of rl problems, and
  sample-efficient algorithms.
\newblock In \emph{Advances in Neural Information Processing Systems}, 2021.

\bibitem[Kaufmann et~al.(2021)Kaufmann, M{\'e}nard, Domingues, Jonsson,
  Leurent, and Valko]{kaufmann2021adaptive}
Emilie Kaufmann, Pierre M{\'e}nard, Omar~Darwiche Domingues, Anders Jonsson,
  Edouard Leurent, and Michal Valko.
\newblock Adaptive reward-free exploration.
\newblock In \emph{Algorithmic Learning Theory}, 2021.

\bibitem[Krishnamurthy et~al.(2016)Krishnamurthy, Agarwal, and
  Langford]{krishnamurthy2016pac}
Akshay Krishnamurthy, Alekh Agarwal, and John Langford.
\newblock {PAC} reinforcement learning with rich observations.
\newblock \emph{Advances in Neural Information Processing Systems}, 2016.

\bibitem[Liu et~al.(2021)Liu, Yu, Bai, and Jin]{liu2021sharp}
Qinghua Liu, Tiancheng Yu, Yu~Bai, and Chi Jin.
\newblock A sharp analysis of model-based reinforcement learning with
  self-play.
\newblock In \emph{International Conference on Machine Learning}, 2021.

\bibitem[M{\'e}nard et~al.(2021)M{\'e}nard, Domingues, Jonsson, Kaufmann,
  Leurent, and Valko]{menard2021fast}
Pierre M{\'e}nard, Omar~Darwiche Domingues, Anders Jonsson, Emilie Kaufmann,
  Edouard Leurent, and Michal Valko.
\newblock Fast active learning for pure exploration in reinforcement learning.
\newblock In \emph{International Conference on Machine Learning}, 2021.

\bibitem[Misra et~al.(2020)Misra, Henaff, Krishnamurthy, and
  Langford]{misra2020kinematic}
Dipendra Misra, Mikael Henaff, Akshay Krishnamurthy, and John Langford.
\newblock Kinematic state abstraction and provably efficient rich-observation
  reinforcement learning.
\newblock In \emph{International conference on machine learning}, 2020.

\bibitem[Modi et~al.(2020)Modi, Jiang, Tewari, and Singh]{modi2019sample}
Aditya Modi, Nan Jiang, Ambuj Tewari, and Satinder Singh.
\newblock Sample complexity of reinforcement learning using linearly combined
  model ensembles.
\newblock In \emph{International Conference on Artificial Intelligence and
  Statistics}, 2020.

\bibitem[Modi et~al.(2021)Modi, Chen, Krishnamurthy, Jiang, and
  Agarwal]{modi2021model}
Aditya Modi, Jinglin Chen, Akshay Krishnamurthy, Nan Jiang, and Alekh Agarwal.
\newblock Model-free representation learning and exploration in low-rank mdps.
\newblock \emph{arXiv:2102.07035}, 2021.

\bibitem[Qiu et~al.(2021)Qiu, Ye, Wang, and Yang]{qiu2021reward}
Shuang Qiu, Jieping Ye, Zhaoran Wang, and Zhuoran Yang.
\newblock On reward-free {RL} with kernel and neural function approximations:
  Single-agent {MDP} and markov game.
\newblock In \emph{International Conference on Machine Learning}, 2021.

\bibitem[Russo and Van~Roy(2013)]{russo2013eluder}
Daniel Russo and Benjamin Van~Roy.
\newblock Eluder dimension and the sample complexity of optimistic exploration.
\newblock \emph{Advances in Neural Information Processing Systems}, 2013.

\bibitem[Sun et~al.(2019)Sun, Jiang, Krishnamurthy, Agarwal, and
  Langford]{sun2019model}
Wen Sun, Nan Jiang, Akshay Krishnamurthy, Alekh Agarwal, and John Langford.
\newblock Model-based {RL} in contextual decision processes: {PAC} bounds and
  exponential improvements over model-free approaches.
\newblock In \emph{Conference on learning theory}, 2019.

\bibitem[Uehara et~al.(2021)Uehara, Zhang, and Sun]{uehara2021representation}
Masatoshi Uehara, Xuezhou Zhang, and Wen Sun.
\newblock Representation learning for online and offline {RL} in low-rank
  {MDP}s.
\newblock In \emph{International Conference on Learning Representations}, 2021.

\bibitem[Wagenmaker et~al.(2022)Wagenmaker, Chen, Simchowitz, Du, and
  Jamieson]{wagenmaker2022reward}
Andrew Wagenmaker, Yifang Chen, Max Simchowitz, Simon~S Du, and Kevin Jamieson.
\newblock Reward-free {RL} is no harder than reward-aware {RL} in linear markov
  decision processes.
\newblock \emph{arXiv:2201.11206}, 2022.

\bibitem[Wainwright(2019)]{wainwright2019high}
Martin~J Wainwright.
\newblock \emph{High-dimensional statistics: A non-asymptotic viewpoint}.
\newblock Cambridge University Press, 2019.

\bibitem[Wang et~al.(2020{\natexlab{a}})Wang, Du, Yang, and
  Salakhutdinov]{wang2020on}
Ruosong Wang, S.~Simon Du, F.~Lin Yang, and Ruslan Salakhutdinov.
\newblock On reward-free reinforcement learning with linear function
  approximation.
\newblock In \emph{Advances in Neural Information Processing Systems},
  2020{\natexlab{a}}.

\bibitem[Wang et~al.(2020{\natexlab{b}})Wang, Foster, and
  Kakade]{wang2020statistical}
Ruosong Wang, Dean~P Foster, and Sham~M Kakade.
\newblock What are the statistical limits of offline rl with linear function
  approximation?
\newblock In \emph{International Conference on Learning Representations},
  2020{\natexlab{b}}.

\bibitem[Wang et~al.(2020{\natexlab{c}})Wang, Salakhutdinov, and
  Yang]{wang2020reinforcement}
Ruosong Wang, Russ~R Salakhutdinov, and Lin Yang.
\newblock Reinforcement learning with general value function approximation:
  Provably efficient approach via bounded eluder dimension.
\newblock In \emph{Advances in Neural Information Processing Systems},
  2020{\natexlab{c}}.

\bibitem[Wang et~al.(2021)Wang, Wang, and Kakade]{wang2021exponential}
Yuanhao Wang, Ruosong Wang, and Sham Kakade.
\newblock An exponential lower bound for linearly realizable {MDP} with
  constant suboptimality gap.
\newblock \emph{Advances in Neural Information Processing Systems}, 2021.

\bibitem[Weisz et~al.(2021{\natexlab{a}})Weisz, Amortila, Janzer,
  Abbasi-Yadkori, Jiang, and Szepesv{\'a}ri]{weisz2021query}
Gellert Weisz, Philip Amortila, Barnab{\'a}s Janzer, Yasin Abbasi-Yadkori, Nan
  Jiang, and Csaba Szepesv{\'a}ri.
\newblock On query-efficient planning in {MDP}s under linear realizability of
  the optimal state-value function.
\newblock In \emph{Conference on Learning Theory}, 2021{\natexlab{a}}.

\bibitem[Weisz et~al.(2021{\natexlab{b}})Weisz, Amortila, and
  Szepesv{\'a}ri]{weisz2021exponential}
Gell{\'e}rt Weisz, Philip Amortila, and Csaba Szepesv{\'a}ri.
\newblock Exponential lower bounds for planning in {MDP}s with
  linearly-realizable optimal action-value functions.
\newblock In \emph{Algorithmic Learning Theory}, 2021{\natexlab{b}}.

\bibitem[Weisz et~al.(2022)Weisz, Szepesv{\'a}ri, and
  Gy{\"o}rgy]{weisz2022tensorplan}
Gell{\'e}rt Weisz, Csaba Szepesv{\'a}ri, and Andr{\'a}s Gy{\"o}rgy.
\newblock Tensorplan and the few actions lower bound for planning in {MDP}s
  under linear realizability of optimal value functions.
\newblock In \emph{International Conference on Algorithmic Learning Theory},
  2022.

\bibitem[Wu et~al.(2022)Wu, Braverman, and Yang]{wu2022gap}
Jingfeng Wu, Vladimir Braverman, and Lin Yang.
\newblock Gap-dependent unsupervised exploration for reinforcement learning.
\newblock In \emph{International Conference on Artificial Intelligence and
  Statistics}, 2022.

\bibitem[Xie and Jiang(2020)]{xie2020q}
Tengyang Xie and Nan Jiang.
\newblock Q* approximation schemes for batch reinforcement learning: A
  theoretical comparison.
\newblock In \emph{Conference on Uncertainty in Artificial Intelligence}, 2020.

\bibitem[Xie and Jiang(2021)]{xie2021batch}
Tengyang Xie and Nan Jiang.
\newblock Batch value-function approximation with only realizability.
\newblock In \emph{International Conference on Machine Learning}, 2021.

\bibitem[Yin and Wang(2021)]{yin2021optimal}
Ming Yin and Yu-Xiang Wang.
\newblock Optimal uniform {OPE} and model-based offline reinforcement learning
  in time-homogeneous, reward-free and task-agnostic settings.
\newblock \emph{Advances in neural information processing systems}, 2021.

\bibitem[Zanette and Wainwright(2022)]{zanette2022bellman}
Andrea Zanette and Martin~J Wainwright.
\newblock Bellman residual orthogonalization for offline reinforcement
  learning.
\newblock \emph{arXiv:2203.12786}, 2022.

\bibitem[Zanette et~al.(2020{\natexlab{a}})Zanette, Lazaric, Kochenderfer, and
  Brunskill]{zanette2020learning}
Andrea Zanette, Alessandro Lazaric, Mykel Kochenderfer, and Emma Brunskill.
\newblock Learning near optimal policies with low inherent bellman error.
\newblock In \emph{International Conference on Machine Learning},
  2020{\natexlab{a}}.

\bibitem[Zanette et~al.(2020{\natexlab{b}})Zanette, Lazaric, Kochenderfer, and
  Brunskill]{zanette2020provably}
Andrea Zanette, Alessandro Lazaric, Mykel~J Kochenderfer, and Emma Brunskill.
\newblock Provably efficient reward-agnostic navigation with linear value
  iteration.
\newblock In \emph{Advances in Neural Information Processing Systems},
  2020{\natexlab{b}}.

\bibitem[Zhang et~al.(2020{\natexlab{a}})Zhang, McAllister, Calandra, Gal, and
  Levine]{zhang2020learning}
Amy Zhang, Rowan McAllister, Roberto Calandra, Yarin Gal, and Sergey Levine.
\newblock Learning invariant representations for reinforcement learning without
  reconstruction.
\newblock In \emph{International Conference on Learning Representations},
  2020{\natexlab{a}}.

\bibitem[Zhang et~al.(2021)Zhang, Zhou, and Gu]{zhang2021reward}
Weitong Zhang, Dongruo Zhou, and Quanquan Gu.
\newblock Reward-free model-based reinforcement learning with linear function
  approximation.
\newblock \emph{Advances in Neural Information Processing Systems}, 2021.

\bibitem[Zhang et~al.(2020{\natexlab{b}})Zhang, Ma, and Singla]{zhang2020task}
Xuezhou Zhang, Yuzhe Ma, and Adish Singla.
\newblock Task-agnostic exploration in reinforcement learning.
\newblock \emph{Advances in Neural Information Processing Systems},
  2020{\natexlab{b}}.

\end{thebibliography}

\clearpage

\appendix

\section{Comparisons of sample complexity rates}
\label{app:comparison}

In this section, we provide comparisons of sample complexity rates. Some more specific and detailed discussions can be found in \pref{app:linear_completeness}, \pref{app:linear_mdp}, and \pref{app:low_rank} respectively. In \Cref{tab:comparison}, we transfer all bounds in related works into our notations and compare them with ours.
\begin{table}[htb]
\renewcommand{\arraystretch}{2}
\centering
\begin{tabular}{|c|c|c} \hline
Setting  & Sample complexity \\ \hline
Linear MDP   \citep{wagenmaker2022reward} 
& $\frac{\dlr^2 H^5\log(1/\delta)}{\veps^2}$ \\ \hline 
\cellcolor[gray]{0.85}
Linear MDP  (\textbf{\pref{corr:linear_mdp}}) 
& \cellcolor[gray]{0.85} $\frac{\dlr^3 H^8\log(1/\delta)}{\veps^2}$ \\ \hline
Linear completeness + {\color{red} explorability}  \citep{zanette2020provably} 
& $\frac{\dlc^3 H^5\log (1/\delta)}{\veps^2}$\\ \hline
\cellcolor[gray]{0.85}
Linear completeness (\textbf{\pref{corr:linear_complete}}) 
& \cellcolor[gray]{0.85} $\frac{\dlc^3 H^8\log(1/\delta)}{\veps^2}$\\ \hline
Low-rank MDP 
+ small $|\mathcal{A}|$ + {\color{red} reachability} \citep{modi2021model} 
& $\frac{\dlr^{11}H^7 K^{14} \log(|\Philr||\Rcal|/\delta)}{\min\{\veps^2 \eta_{\min}, \eta_{\min}^5\}}$\\ \hline
\cellcolor[gray]{0.85}
Low-rank MDP 
+ small $|\mathcal{A}|$ (\textbf{\pref{corr:low_rank}})
& \cellcolor[gray]{0.85} $\frac{\dlr^3 H^8 K\log(|\Philr|\Rcal|/\delta)}{\varepsilon^2}$\\ \hline
\cellcolor[gray]{0.85}
Completeness + Q-type BE dimension  
(\textbf{\pref{thm:rfolive_q}}) 
& \cellcolor[gray]{0.85} $\frac{\dqbee^2(H^7 d_{\Fcal} + H^5 d_{\Rcal})  \log(1/\delta)}{\veps^2}$ \\\hline
\cellcolor[gray]{0.85}
Completeness + V-type BE dimension + small  $|\mathcal{A}|$  
(\textbf{\pref{thm:rfolive_v}}) 
& \cellcolor[gray]{0.85} $\frac{\dvbee^2 K(H^7d_\Fcal + H^5 d_\Rcal)  \log(1/\delta) }{\veps^2}$ \\ \hline
\end{tabular}
\vspace{1em}
\caption{Comparisons between our results and most closely related works in reward-free exploration. {\color{red} Red assumptions} are what prior works need that are avoided by us. For simplicity, we only show the orders and hide polylog terms (i.e., using $\tilde O(\cdot)$ notation). $\etamin$ is the reachability factor in \citet{modi2021model}. 
\label{tab:comparison}}
\end{table}

In linear MDPs, our bound (\pref{corr:linear_mdp}) is $\dlr H^3$ worse compared with the most recent work \citep{wagenmaker2022reward}, but our result is also independent of $K$. It should be noted that both these bounds are sub-optimal in $H$ dependence when compared to the lower bound of $\Omega\rbr{d^2H^2/\veps^2}$ shown in \citet{wagenmaker2022reward}. In the reward-aware setting, \golf has a sharper rate than the subroutine \olive under the completeness type assumption \citep{jin2021bellman}. Since in \rfolive we only collect data when running a single (zero) reward \olive during the online phase and completeness (\pref{assum:completeness_F}) is satisfied in our paper, we believe that there also exists a reward-free version of \golf (by running \golf with zero reward function in the online phase and performing function elimination in the offline phase) that can potentially improve an $H\dlr$ factor. 

As for the linear completeness setting, our rate (\pref{corr:linear_complete}) \emph{appears} to be $H^3$ worse than \citet{zanette2020provably}. However, we want to  remark that they need to assume $\veps$ to be ``asymptotically small'' (more specifically, $\veps\le \tilde O(\nu_{\min}/\sqrt {\dlc})$, where $\nu_{\min}$ is their explorability factor). Thus there is an implicit dependence on $1/\nu_{\min}$ in their sample complexity bound. Since such a factor can be arbitrarily large while $H$ is always bounded in a finite horizon problem, our bound could be much better than theirs. Again, there could be an $H\dlc$ tighter bound for the reward-free version of \golf, which implies that the optimal $\dlc$ dependence in the linear completeness setting could also be improved.

In low-rank MDPs, it is easy to see that our result (\pref{corr:low_rank}) significantly improves upon the rate of \citet{modi2021model} in $\dlr$ and $K$ factors, while slightly worse in the $H$ factor. In addition, they require the reachability assumption ($\etamin$ is their reachability factor), which means that their bound can be arbitrary worse than ours. Similar dependence on reachability factor $1/\etamin$ also exists in the sample complexity bounds of the more restricted block MDPs \citep{du2019provably,misra2020kinematic} as they assume the reachability assumption. 

Finally, regarding lower bounds, we do not necessarily need a direct one in our general function approximation setting (or even the more restricted linear completeness setting/low-rank MDPs) to compare with. The lower bound for reward-free exploration in linear MDPs \citep{wagenmaker2022reward} is applicable to each of these and shows the necessary dependence on the respective complexity measures. Coming up with a method which incorporates general function approximation while incurring better sample complexity rates on these special instances is a challenging and interesting avenue for future work.

\section{Discussions on Q-type and V-type}
\label{app:q_vs_v}

In this paper we study both Q-type and V-type, and they are not specific to the reward-free exploration. Different versions (Q-type and V-type) already exist in the reward-aware general function approximation RL (e.g., \citet{jiang2017contextual,jin2021bellman,du2021bilinear}). They capture different scenarios of interest and so far, it seems difficult to unify them even in the reward-aware setting. Therefore, to give a comprehensive treatment of general function approximation, we consider both together. The algorithms and analyses for the two types are not very different, with only moderate differences.

Since we consider the BE dimension and it subsumes Bellman rank \citep{jin2021bellman}, we first provide a detailed comparison between Q-type and V-type Bellman rank. As discussed in \citet{agarwal2022non}, V-type permits representation learning and other non-linear scenarios that are not easily captured in Q-type. For instance, any contextual bandit problem is admissible under the V-type assumption (the V-type Bellman rank is 1), while Q-type does not capture all finite action, non-linear contextual bandit problems with a realizable reward. We refer the reader to the detailed lower bound on the Q-type Bellman rank in the contextual bandit setting in Appendix B of \citet{agarwal2022non}. In contrast, Q-type has a more linear like structure, but it also includes problems whose V-type Bellman rank is large (e.g., linear completeness setting in \citet{zanette2020learning}). Further, V-type \rfolive (or V-type \olive) requires one uniform action in exploration and therefore has an additional $K$ factor (the cardinality of action space) in the sample complexity bound.

Then we discuss the BE dimension. It can be shown that the Q-type BE dimension could also be exponentially larger than the V-type BE dimension. In \citet{agarwal2022non}, the authors show that the Q-type Bellman rank for a contextual bandit instance can be made arbitrarily higher whereas the V-type Bellman rank is always $1$ for a CB setting. Here, we show that the same instance can be shown to have high Bellman Eluder dimension as well. The construction considers a context distribution which is uniform on $1,\ldots,N$, where we have $N$ unique contexts. We have two actions $\{a_1, a_2\}$. We also have $|\Fcal| = N+1$ with the following structure:
\begin{align*}
    f^*(x,a_1) = {} & f_{N+1}(x,a_1) = 0\\
    f^*(x,a_2) = {} & f_{N+1}(x,a_2) = 0.5.
\end{align*}
For $i<N+1$, we have $f_i(x,a)=f^*(x,a)$ when $x\ne i$, and $f_i(x,a_1)=1$, $f_i(x,a_2)=0.5$ implying that the function $f_i$ makes incorrect prediction on context $i$ for action $a_1$. Now, the bound on V-type BE dimension can be obtained by using the Bellman rank to BE dimension conversion result in Proposition 21 from \citet{jin2021bellman} or \pref{prop:br_be} in our paper. Since, the V-type Bellman rank is $1$, the BE dimension is bounded as $\dvbe^R(\Fcal, \Dcal_{\Fcal}, 1/(2N)) \le \tilde O\rbr{\log (N)}$. We now show that the Q-type BE dimension is $\Omega(N)$. Consider the sequence of policies $\pi_1, \ldots, \pi_{N}$. For any $i\in\{1,\ldots,N\}$ and sequence $\pi_1,\ldots,\pi_{i-1}$, (Q-type) Bellman residuals incurred by the function $f_i$ is: $\sqrt{\sum_{k=1}^{i-1}\rbr{\frac{1}{N}\sum_{j=1}^{N} \EE_{a\sim \pi_k}\sbr{f_i(j,a)-f^*(j,a)}}^2} = 0$. The same residual on the distribution induced by $\pi_i$ can be written as $\abr{\sum_{j=1}^{N} \frac{1}{N}\EE_{a\sim \pi_i}\sbr{f_i(j,a)-f^*(j,a)}} = 1/N$. Hence, $\pi_i$ is $1/(2N)$-independent of $\{\pi_1,\ldots,\pi_{i-1}\}$ (recall \pref{def:ind_dist}). Thus, for $\veps = 1/(2N)$, the sequence $\pi_1,\ldots,\pi_{N}$ can be used to show that the DE dimension (\pref{def:DE}) and the Q-type BE dimension for this instance is $\tilde O(N)$. Hence, the Q-type BE dimension is exponentially larger than the V-type BE dimension for this instance.

\section{Q-type \rfolive results}
\label{app:rfolive_q}

In this section, we present the results related to Q-type \rfolive. In \pref{app:olive_q}, we introduce the theoretical guarantee of Q-type \olive \citep{jiang2017contextual,jin2021bellman} for completeness. In \pref{app:general_q}, we show the detailed proof of the sample complexity bound of Q-type \rfolive (\pref{thm:rfolive_q}). In \pref{app:linear_completeness}, we discuss the instantiation of Q-type \rfolive to the known representation linear completeness setting. In \pref{app:linear_mdp}, we provide another instantiation of Q-type \rfolive to the linear MDP with known feature.

\subsection{Q-type \olive}
\label{app:olive_q}
We first introduce the following assumption that will be useful for the \olive results (\pref{prop:olive_q} and \pref{prop:olive_v}). Notice that this realizability assumption is for the single reward-aware \olive, where the function class captures reward-appended optimal value functions. Thus it is different from our reward-free realizability assumption (\pref{assum:realizability_F}).
\begin{assum}[$\varepsilon$-approximate realizability of the single-reward function class] 
\label{assum:realizability_olive} 
For the reward function $R$, optimal Q-function $Q^*_R$, and the value function class $\Fcal$, there exists $Q_R^\mathrm{c}\in\Fcal$ so that $\max_{h\in[H]}\|Q_{R,h}^*-Q_{R,h}^\mathrm{c}\|_\infty\le\varepsilon$.
\end{assum}

Then we state the sample complexity result of Q-type \olive (Algorithm 2 in \citet{jin2021bellman}). In this paper, we consider the uniformly bounded reward setting ($0\le r_h\le 1,\forall h\in[H]$) instead of bounded total reward setting ($\forall h\in[H], r_h\ge 0$ and $\sum_{h=0}^{H-1} r_h\le 1$) in \citet{jiang2018open,jiang2017contextual,jin2021bellman}. Therefore we need to pay an additional $H^2$ dependency in $\nactv$ and $\nelim$ because the range of value function is $H$ times larger than the original ones, which induces an additional $H^2$ factor in the concentration inequalities.

\begin{proposition}[Sample complexity of Q-type \olive, modification of Theorem 18 in \citet{jin2021bellman}]
\label{prop:olive_q}
Under \pref{assum:realizability_olive} with exact realizability (zero approximation error), if we set
$$\eactv=\frac{\varepsilon}{2H},\ \eelim=\frac{\varepsilon}{8H\sqrt{\dqbee}},\ \nactv=\frac{H^4 \iota}{\varepsilon^2},\text{ and } \nelim=\frac{H^4\dqbee\log(\Ncal_\Fcal(\eelim/64))  \iota}{\varepsilon^2}$$
where $\dqbee=\dqbe^R(\Fcal,\Dcal_{\Fcal},\varepsilon/(4H))$ and  $\iota=\ca\log(H\dqbee/\delta\varepsilon)$, 
then with probability at least $1-\delta$, Q-type \olive (Algorithm 2 in \citet{jin2021bellman}) with $\Fcal$ will output an $\varepsilon$-optimal policy (under a single reward function $R$) using at most  $O(H\dqbee(\nactv+\nelim))$ episodes. Here $\ca$ is a large enough constant.
\end{proposition}

This sample complexity result directly follows from \citet{jin2021bellman} with minor adaptation of the parameters. We refer the reader to \citet{jin2021bellman} for the detailed proof.

\subsection{Proof of Q-type \rfolive under general function approximation}
In this part, we first provide the general statement of \pref{thm:rfolive_q} and then show the detailed proof. We also provide a detailed discussion on the different and novel part in our proof compared with \citet{jiang2017contextual,jin2021bellman}. 
\label{app:general_q}
\begin{theorem*}[Full version of \pref{thm:rfolive_q}]
Fix $\delta \in (0,1)$. Given a reward class $\Rcal$ and a function class $\Fcal$ that satisfies \pref{assum:realizability_F} and \pref{assum:completeness_F},
with probability at least $1-\delta$, for any $R \in \Rcal$, Q-type \rfolive (\pref{alg:rf_olive}) with $\Fcal$ outputs a policy $\hat\pi$ that satisfies $v^{\hat \pi}_R \ge v^*_R - \veps$. The required number of episodes is
\begin{align*}
	O\rbr{\frac{\rbr{H^7\log\rbr{\Ncal_{\Fcal}\rbr{\varepsilon/512H^2\sqrt{\dqbee}}}+H^5\log\rbr{\Ncal_{\Rcal}\rbr{\varepsilon/512H^2\sqrt{\dqbee}}}}\dqbee^2  \iota}{\varepsilon^2}}.
\end{align*}
In \rfolive, we set
\[\eactv=\frac{\varepsilon}{2H^2}, \eelim=\frac{\varepsilon}{8H^2\sqrt{\dqbee}}, \nactv=\frac{H^6 \iota}{\varepsilon^2},
\] and
\begin{align*}
\nelim=&~\frac{\rbr{H^6\log(\Ncal_{\Fcal}(\eelim/64))+H^4\log(\Ncal_{\Rcal}(\eelim/64))}\dqbee  \iota}{\varepsilon^2}
\\
=&~\frac{\rbr{H^6\log\rbr{\Ncal_{\Fcal}\rbr{\varepsilon/512H^2\sqrt{\dqbee}}}+H^4\log\rbr{\Ncal_{\Rcal}\rbr{\varepsilon/512H^2\sqrt{\dqbee}}}}\dqbee  \iota}{\varepsilon^2},
\end{align*}
where $\dqbee=\dqbe^\zero(\Fcal-\Fcal,\Dcal_{\Fcal-\Fcal},\veps/(4H))$, $\iota=\cb\log(H\dqbee/\delta\varepsilon)$, and $\cb$ is a large enough constant.
\end{theorem*}

\begin{proof} 
From the online phase of Q-type \rfolive (\pref{alg:rf_olive}), we can see that this phase is equivalent to running Q-type \olive (Algorithm 2 in \citet{jin2021bellman}) with the input function class $\Fcal-\Fcal$, the specified parameters $\eactv,\eelim,\nelim,\nactv$ and under the reward function $R=\zero$. In \pref{prop:olive_q}, we know that realizability (\pref{assum:realizability_olive}) holds because  $\zero\in\Fcal-\Fcal=\Fon$. Then the sample complexity is immediately from our specified values of $\eactv,\eelim,\nactv,\nelim$ and \pref{prop:olive_q} as we only collect samples in the online phase. Notice that the log-covering number $\log\rbr{\Ncal_{\Fon}(\cdot)}=\log\rbr{\Ncal_{\Fcal-\Fcal}(\cdot)}\le 2\log\rbr{\Ncal_\Fcal(\cdot)}$ and such a constant 2 is absorbed by large enough $\cb$. Therefore, it remains to show that the algorithm can indeed output an $\varepsilon$-optimal policy with probability $1-\delta$ in the offline phase. We will show the following three claims hold with probability at least $1-\delta$. 
    \paragraph{Claim 1} For any $g\in \Foff(R)$, if $\exists h\in[H]$, s.t. $|\EcalQ(g,\pi_g,h)|\ge\varepsilon/H$, then it will be eliminated in the offline phase. 
    \paragraph{Claim 2} $Q^*_R\in\Foff(R)$ and $Q^*_R$ will not be eliminated in the offline phase.
    \paragraph{Claim 3} At the end of the offline phase, picking the optimistic function from the survived value functions gives us $\varepsilon$-optimal policy.
    
    \vspace{1em}

Before showing these three claims, we first state properties from the online phase of Q-type \rfolive and the concentration results in the offline phase. 

\paragraph{Properties from the online phase of Q-type \rfolive}  From the equivalence between the online phase of Q-type \rfolive (\pref{alg:rf_olive}) and Q-type \olive (Algorithm 2 in \citet{jin2021bellman}) with reward $\zero$, we know that with probability at least $1-\delta/4$, the online phase terminates within $\dqbee H+1$ iterations. In addition, with probability at least $1-\delta/4$, the following properties (\cref{eq:rfolive_q_stop} and \cref{eq:rfolive_q_conc}) hold for the first $\dqbee H+1$ iterations:

(i) When the online phase exits at iteration $T$ in \pref{line:set_T} (i.e., the elimination procedure is not activated in \rfolive), for any $f\in\Fcal^T$, it predicts no more than $\veps/H$ value:
\begin{align}
\label{eq:rfolive_q_stop}
V_{f}(x_0)\le V_{f^T}(x_0)=V_{f^T}(x_0)-V^{\pi_{f^T}}_{\zero,0}(x_0)=\sum_{h=0}^{H-1}\EcalQ(f^T,\pi^T,h)<2H\eactv=\varepsilon/H.
\end{align}
The first equality is due to any policy evaluation has value 0 under the reward function $\zero$. The second equality is due to the policy loss decomposition in \citet{jiang2017contextual}. The second inequality is adapted from the ``concentration in the activation procedure'' part of the proof for Theorem 18 in \citet{jin2021bellman}.

(ii) For $T\le \dqbee H+1$, the concentration argument holds for any $f\in\Fon$ and $t\in [T]$:
\begin{align}
\label{eq:rfolive_q_conc}
\abr{\hEcalQzero(f, \pi^t, h^t)  - \EcalQzero(f, \pi^t, h^t)} < \eelim/8.
\end{align}
This is from the ``concentration in the elimination procedure'' step of the proof for Theorem 18 in \citet{jin2021bellman} and we adapt it with our parameters. 

\paragraph{Concentration results in the offline phase}
In \rfolive we use $ \hat\Ecal^R(g,\pi^t,\pi,h^t)$ and $\pi\in\Pi_{\textrm{est}}$ in \pref{line:off_elim}. Since we are in the Q-type version, we have $\Pi_{\textrm{est}}=\{\pi^t\}$. In addition, from \pref{def:abe}, we know that $\Ecal^R(g,\pi^t,\pi^t,h^t)=\EcalQ(g,\pi^t,h^t)$. Therefore, in \pref{line:off_elim}, it is equivalent to eliminating according to $\hEcalQ(g,\pi^t,h^t)$. Throughout this proof, we will use $\EcalQ(g,\pi^t,h^t)$ and $\hEcalQ(g,\pi^t,h^t)$ notations for simplicity. Now we show the concentration results in the offline phase.

Let $\overline\Rcal$ be an $(\eelim/64)$-cover of $\Rcal$. For every $R\in\Rcal$, let $ R^{\mathrm c}=\argmin_{R'\in\overline\Rcal}\max_{h\in[H]}\|R_h-R'_h\|_\infty$. Firstly, consider any fixed $R\in\overline \Rcal$ and let $\Zcal(R)$ be an ($\eelim/64$)-cover of $\Foff(R)$ with cardinality $\Ncal_{\Foff(R)}(\eelim/64)=\Ncal_{\Fcal}(\eelim/64)$. For every $g\in\Foff(R)$, let $ g^{\mathrm c}=\argmin_{g'\in\Zcal(R)}\max_{h\in[H]}\|g_h-g'_h\|_\infty$. 

Applying Hoeffding's inequality to all $(t,g')\in[T]\times\Zcal(R)$ and taking a union bound, we have that with probability at least $1-{\delta}/(2\Ncal_{\Rcal}(\eelim/64))$, the following holds for all $(t,g')\in[T]\times\Zcal(R)$
\begin{align*}
\abr{ \hEcalQ(g', \pi^t, h^t)  - \EcalQ(g', \pi^t, h^t)}  \le 4H\sqrt{\frac{\log(4T\Ncal_{\Rcal}(\eelim/64)\Ncal_{\Fcal}(\eelim/64)/\delta)}{2\nelim}} <\eelim/8.
\end{align*}
The second inequality is due to $\eelim=\varepsilon/\rbr{8H^2\sqrt{\dqbee}},$ $\iota=\cb\log(H\dqbee/\delta\varepsilon)$, and

\[
\nelim=\frac{(H^6\log(\Ncal_{\Fcal}(\eelim/64))+H^4\log(\Ncal_{\Rcal}(\eelim/64)))\dqbee  \iota}{\varepsilon^2},
\]
with $\cb$ in $\iota$ being chosen large enough.

Therefore for any $g\in\Foff(R)$, we get
\begin{align*}
& \abr{\hEcalQ(g, \pi^t, h^t)  - \EcalQ(g, \pi^t, h^t)} & \\ 
\le {} & \abr{\hEcalQ(g, \pi^t, h^t)  - \hEcalQ(g^{\mathrm{c}}, \pi^t, h^t)}+\abr{\hEcalQ(g^{\mathrm{c}}, \pi^t, h^t)  - \EcalQ(g^{\mathrm{c}}, \pi^t, h^t)} & \\
{} &  + \abr{\EcalQ(g^{\mathrm{c}}, \pi^t, h^t)  - \EcalQ(g, \pi^t, h^t)} & \\
\le {} & 2 \eelim/64 + \eelim/8+2 \eelim/64 & \\
={} & 3\eelim/16. &
\end{align*}
Union bounding over $R\in\overline\Rcal$, with probability at least $1-\delta/2$, for all $t\in[T]$, $R\in\overline\Rcal$, $g\in\Foff(R)$, we have
\[
\abr{\hEcalQ(g, \pi^t, h^t)  - \EcalQ(g, \pi^t, h^t)} \le 3\eelim/16. 
\]

Therefore, with probability at least $1-\delta/2$, for all $t\in[T]$, $R\in\Rcal$, $g\in\Foff(R)$, we have
\begin{align}
\label{eq:conc_offline_q}
& \abr{\hEcalQ(g, \pi^t, h^t)  - \EcalQ(g, \pi^t, h^t)} & 
\notag
\\
\le {} & \abr{\hEcalQ(g, \pi^t, h^t)  - \hat \Ecal_{\mathrm{Q}}^{R^{\mathrm{c}}}(g, \pi^t, h^t)} + \abr{\hat \Ecal_{\textrm{Q}}^{R^{\mathrm{c}}}(g, \pi^t, h^t)  - \Ecal_{\textrm{Q}}^{R^{\mathrm{c}}}(g, \pi^t, h^t)} & 
\notag\\
{} & + \abr{\Ecal_{\textrm{Q}}^{R^{\mathrm{c}}}(g, \pi^t, h^t)  - \EcalQ(g, \pi^t, h^t)} & 
\notag\\
\le {} &  \eelim/64 
+3\eelim/16
+ \eelim/64 &
\notag\\
< {} &\eelim/4. 
\end{align}

All statements in our subsequent proof are under the event that all the different high-probability events (the online phase terminates within $\dqbee H+1$ iterations, and \cref{eq:rfolive_q_stop}, \cref{eq:rfolive_q_conc}, \cref{eq:conc_offline_q} hold for the first $\dqbee H+1$ iterations) discussed above hold with a total failure probability of $\delta$. 

\paragraph{Proof of Claim 1}  Consider any $g\in\Foff(R)$ such that $\exists h\in[H]$, $|\EcalQ(g,\pi_g,h)|\ge\varepsilon/H$. Recall the definition of $\Foff(R)$, we know that $g$ can be written as $g=(g_0,\ldots,g_{H-1})=(f_0+R_0,\ldots,f_{H-1} + R_{H-1})$, $f_h\in\Fcal_h$. 
We will discuss the positive average Bellman error and the negative average Bellman error cases separately. 


\paragraph{Case (i) of Claim 1} $\EcalQ(g,\pi_g,h)=\EE[g_{h}(x_h,a_h) - R_h(x_h,a_h) - V_g(x_{h+1})\mid a_{0:h}\sim\pi_g]\ge\varepsilon/H$.

Since $g_h=f_h+R_h$, 
we know that
\begin{align*}
\varepsilon/H&\le\EE\sbr{g_{h}(x_h,a_h) - R_h(x_h,a_h) - V_g(x_{h+1}) \mid a_{0:h} \sim \pi_g}\\
&=\EE[f_h(x_h,a_h)- V_g(x_{h+1})\mid a_{0:h}\sim\pi_g]\\
&=\EE[f_h(x_h,a_h)- (\Tcal^\zero_h g_{h+1})(x_{h},a_{h})\mid a_{0:h}\sim\pi_g] \tag{Definition of $\Tcal^\zero_h$}\\
&=\EE[\tilde f_h(x_h,a_h)\mid a_{0:h}\sim\pi_g]. \tag{$\tilde f_h\defeq f_h- \Tcal^\zero_h g_{h+1}$}
\end{align*}
Here we construct a function $\tilde f$ that has the same value as $f_h- \Tcal^\zero_h g_{h+1}$ at level $h$, uses zero reward Bellman backup for any level before $h$, and assigns zero value after level $h$. More formally, it is defined as
\begin{align*}
\tilde f_{h'}(x_{h'},a_{h'})\hspace{-.15em}=\hspace{-.15em}\left\{
\begin{aligned}
&(\Tcal_{h'}^\zero \tilde f_{h'+1})(x_{h'},a_{h'})\hspace{-.15em}=\hspace{-.15em}\EE[\max_{a}\tilde f_{h'+1}(x_{h'+1},a) \mid x_{h'},a_{h'}] & & 0\le h'\le h-1 \\
&f_h(x_h,a_h)- (\Tcal^\zero_h g_{h+1})(x_{h},a_{h})  & & h'=h \\
&0 & &h+1\le h'\le H-1.
\end{aligned}
\right.
\end{align*}
From the definition of Q-type average Bellman error and the construction, we know that for any policy $\pi$ we can translate the Q-type reward-dependent average Bellman error for a function $g\in\Foff(R)$ to the zero reward Q-type average Bellman error of a function $\tilde f\in\Fon$ as the following
\begin{align}
\label{eq:off_to_on}
\EcalQ(g,\pi,h)&=\EE[g_{h}(x_h,a_h) - R_h(x_h,a_h) - V_g(x_{h+1})\mid a_{0:h-1}\sim\pi,a_h\sim\pi]\notag
\\
&=\EE[f_h(x_h,a_h)-(\Tcal^\zero_h g_{h+1})(x_{h},a_{h})\mid a_{0:h}\sim\pi]\notag
\\
&=\EE[\tilde f_h(x_h,a_h) -\zero-\tilde f_{h+1}(x_{h+1},a_{h+1})\mid a_{0:h}\sim\pi,a_{h+1}\sim \pi_{\tilde f}] \notag 
\\
&=\EcalQzero(\tilde f,\pi,h),
\end{align}
where in the third equality we notice that $\tilde f_{h+1}=\zero$.

We can verify that $\tilde f=(\tilde f_0,\ldots,\tilde f_{H-1})\in\Fon$. First let us consider level $h'=h$. From completeness (\pref{assum:completeness_F}), we know that $\Tcal^\zero_h g_{h+1}=\Tcal^\zero_h(f_{h+1}+R_{h+1})\in \Fcal_h$. Therefore, we have $\tilde f_h=f_h- \Tcal^\zero_h g_{h+1}\in \Fcal_h-\Fcal_h$. 
Then we consider level $0\le h'\le h-1$. By the definition, we use zero reward Bellman backup. 
From completeness and $\tilde f_h\in\Fcal_h-\Fcal_h$, we have $\tilde f_{h-1}=\Tcal^\zero_h \tilde f_h\in\Fcal_{h-1}-\Fcal_{h-1}$. By performing this inductive process backward, we have $\tilde f_{h'}\in \Fcal_{h'}-\Fcal_{h'}$ for any $0\le h'\le h-1$. For level $h+1\le h'\le H-1$, we immediately get $\tilde f_{h'}=\zero\in \Fcal_{h'}-\Fcal_{h'}$. 
Therefore, we can see $\tilde f=(\tilde f_0,\ldots,\tilde f_{H-1})\in\Fon$ from the definition of $\Fon$.\vspace{1em}

From the construction of $\tilde f$ (zero reward Bellman backup for level $0\le h'\le h-1$), we have 
\begin{align*}
V_{\tilde f}(x_0) &= \EE[\tilde f_h(x_h,a_h) \mid a_{0:h}\sim \pi_{\tilde f}] \\
&\ge\EE[\tilde f_h(x_h,a_h) \mid a_{0:h}\sim \pi_g]\\
&=\EE[g_{h}(x_h,a_h) - R_h(x_h,a_h) - V_g(x_{h+1})\mid a_{0:h}\sim\pi_g]\\
&\ge\varepsilon/H,
\end{align*} 
where $\pi_{\tilde f}$ is the greedy policy of $\tilde f$ and in fact it is the optimal policy when treating $\tilde f_h$ as the reward at level $h$ and there are no intermediate rewards. From the first property of the online phase (\cref{eq:rfolive_q_stop}), we know that all the survived value functions at the end of the online phase predict no more than $\varepsilon/H$. Therefore $\tilde f$ will be eliminated. We assume it is eliminated at iteration $t$ by policy $\pi^t$ in level $h^t$. \vspace{1em} 

From the Bellman backup construction of $\tilde f$, we know that $\tilde f$ can only be eliminated at level $h$. This can be seen from the following argument: By the construction of $\tilde f$, we have $\EcalQzero(\tilde f,\pi,h')=0$ for any $\pi$ and $h'\in[H], h'\neq h$. Applying the second property of the online phase (\cref{eq:rfolive_q_conc}), we have $\abr{ \hEcalQzero(\tilde f,\pi^t,h^t)-\EcalQzero(\tilde f,\pi^t,h^t)} \le 3\eelim/4$, which gives us $\abr{ \hEcalQzero(\tilde f,\pi^t,h^t)} \le 3\eelim/4$ if $h^t\neq h$.
Since the elimination threshold is set to $\eelim$, $\tilde f$ will not be eliminated at level $h^t\neq h$.\vspace{1em}

This implies that at some iteration $t$ in the online phase, we will collect some $\pi^t$ that eliminates $\tilde f$ at level $h$, i.e., it satisfies $\abr{\hEcalQzero(\tilde f,\pi^t,h^t)} > \eelim$ and $h^t=h$. Applying the second property of the online phase (\cref{eq:rfolive_q_conc}), we have $\abr{\hEcalQzero(\tilde f,\pi^t,h^t)-\EcalQzero(\tilde f,\pi^t,h^t)}\le \eelim/8$. This tells us $\abr{\EcalQzero(\tilde f,\pi^t,h^t)}> 7\eelim/8$. 
Then from \cref{eq:off_to_on} 
we have 
\begin{align*}
    \abr{\EcalQ(g,\pi^t,h^t)} = \abr{\EcalQzero(\tilde f,\pi^t,h^t)} > 7\eelim/8.
\end{align*}

Finally, the concentration argument of the offline phase (\cref{eq:conc_offline_q}) implies that $\abr{\hEcalQ(g,\pi^t,h^t)-\EcalQ(g,\pi^t,h^t)}$ $< \eelim/4$.
Hence, we get $\abr{\hEcalQ(g,\pi^t,h^t)} > \eelim/2$. This means that we will eliminate such $g$ by $\pi^t$ in the offline phase. 

\paragraph{Case (ii) of Claim 1} $\EcalQ(g,\pi_g,h)=\EE[g_{h}(x_h,a_h) - R_h(x_h,a_h) - V_g(x_{h+1}) \mid a_{0:h} \sim \pi_g ] \le -\varepsilon/H$.

Same as before, we have $\EcalQ(g,\pi_g,h)=\EE[f_h(x_h,a_h)- (\Tcal^\zero_h g_{h+1})(x_{h},a_{h})\mid a_{0:h}\sim\pi_g]\le-\varepsilon/H$. 
Now we let $\tilde{f}_h$ be the negated version of the one in case (i), and define $\tilde{f}$ as
\begin{align*}
\tilde f_{h'}(x_{h'},a_{h'})\hspace{-.15em}=\hspace{-.15em}\left\{
\begin{aligned}
&(\Tcal_{h'}^\zero\tilde g_{h'+1})(x_{h'},a_{h'})\hspace{-.15em}=\hspace{-.15em}\EE[\max_{a}\tilde g_{h'+1}(x_{h'+1},a)\mid x_{h'},a_{h'}] & & 0\le h'\le h-1 
\\
&(\Tcal^\zero_h g_{h+1})(x_h,a_h) - f_h(x_h,a_h)& & h'=h 
\\
&0 & &h+1\le h'\le H-1.
\end{aligned}
\right.
\end{align*}
Following the same steps as in case (i) we can verify that
$\tilde{f} \in \Fon$, and that
$V_{\tilde{f}}(x_0) \geq \varepsilon/H$. From here
the argument is identical to case (i).

\paragraph{Proof of Claim 2} (i) From the assumption, we know that realizability condition $Q^*_R=(Q^*_{R,0},\ldots,$ $Q^*_{R,H-1})\in\Foff(R)$ holds. (ii) For the second argument, we note that $\EcalQ(Q^*_R,\pi,h)= 0$ for any $\pi$ and $h\in[H]$ by the definition of the average Bellman error. From the concentration argument in the offline phase (\cref{eq:conc_offline_q}), we have $\abr{\hEcalQ(Q^*_R,\pi^t,h^t)} \le  \abr{ \EcalQ(Q^*_R, \pi^t, h^t)}  +\eelim/4 = \eelim/4.$ As a result, $Q^*_R$ will not be eliminated. 

\paragraph{Proof of Claim 3} From Claim 1, we know that in the offline phase for any $g\in \Foff(R)$, if $\exists h\in[H]$, s.t. $|\EcalQ(g,\pi_g,h)|\ge\varepsilon/H$, then it will be eliminated. Therefore from the policy loss decomposition in \citet{jiang2017contextual}, for all survived $g\in\Fsur(R)$ in the offline phase, we have 
\[
V_{g}(x_0)-V^{\pi_{g}}_{R,0}(x_0)=\sum_{h=0}^{H-1}\EcalQ(g,\pi_{g},h)< \varepsilon.
\]
Since $Q^*_R$ is not eliminated, similar as \citet{jiang2017contextual,jin2021bellman}, we have 
\begin{align*}
V^{\pi_{\hat g}}_{R,0}(x_0) &> V_{\hat g}(x_0)-\varepsilon \ge V^*_{R,0}(x_0)-\varepsilon.
\end{align*}
Notice that Claim 3 directly implies that \rfolive returns an $\varepsilon$-near optimal policy. This completes the proof.
\end{proof}

\subsection{Technical novelty over reward-aware \olive}
\label{app:novelty}

The key step of the analyses of reward-aware \olive \citep{jiang2017contextual,jin2021bellman} is to show that any bad function whose average Bellman error is large under the given reward function is eliminated (recall that they only have the online phase and the reward is always revealed). This is ensured by the online exploration process. However, the difficulty in our reward-free RL setting is that such a reward function is only revealed in the offline phase, where we no longer actively explore. To overcome this difficulty, we use completely new and novel proof techniques here: For each bad function $g\in \mathcal F_{\mathrm{off}}(R)$ with a large average Bellman error under the true reward $R$, we construct a surrogate function $\tilde f$ in the online phase. Our construction guarantees that $\tilde f$ has the same large average Bellman error as $g$, but the error is instead under the zero reward which we use during exploration (\cref{eq:off_to_on}). Then we show that all these constructed $\tilde f$ belong to the ``difference'' function class $\mathcal F_{\mathrm{on}}$ and $\tilde f$ will be eliminated in the online phase since we use $\mathcal F_{\mathrm{on}}$ and zero reward there. The collected data tuples (gathered constraints) that eliminate $\tilde f$ will be used in the offline phase and they guarantee eliminating its corresponding bad function $g$. Notice that in the design/definition of $\tilde f$, we need to guarantee that it has a large average Bellman error at the same timestep as $g$ does so that it can correctly witness the average Bellman error of $g$, which we ensure via a Bellman backup construction.

In summary, both the construction of the surrogate function $\tilde f$ and the translation of average Bellman error from bad function $g\in\mathcal F_{\mathrm{off}}(R)$ to $\tilde f\in\mathcal F_{\mathrm{on}}$ are novel to the best of our knowledge. They reflect crucial difference between reward-aware and reward-free RL. And at the same time, no reward-aware RL works have used such mechanisms before. 

We also provide a counterexample in \pref{app:other_variant} that shows other variant of \olive could fail even under realizability, completeness, and low Bellman Eluder dimension, where we know \rfolive has polynomial sample complexities.

\subsection{Q-type \rfolive for known representation linear completeness setting}
\label{app:linear_completeness}

We first discuss why stating a specific $B$ is equivalent to stating any $B>0$ in \pref{def:lin_completeness}. Assuming the statement hold for $B$, we will show that it holds for any $B'>0$. The reason is the following. Consider any $Q_{h+1}'=\langle \philc_{h+1},\theta'_{h+1} \rangle\in \Qcal_{h+1}(\{\philc\},B')$, where $\|\theta_{h+1}'\|_2\le B'\sqrt {\dlc}$. For $Q_{h+1}=\inner{\philc_{h+1}}{\theta_{h+1}' \frac{B\sqrt {\dlc}}{\|\theta_{h+1}'\|_2}} \in \Qcal_{h+1}(\{\philc\},B)$, there exists $\theta_h$ that satisfies $\langle \philc_h,\theta_h\rangle=\Tcal_h^\zero Q_{h+1}$ and $\|\theta_h\|_2\le B\sqrt {\dlc}$. Then we know that $\frac{\|\theta_{h+1}'\|_2}{B\sqrt{\dlc}}\langle \philc_h,\theta_h\rangle=\Tcal_h^\zero Q_{h+1}'$. Now we can choose $\theta_h'=\frac{\|\theta_{h+1}'\|_2}{B\sqrt {\dlc}}\theta_h$, and therefore we have $\Tcal_h^\zero Q'_{h+1}=\langle \philc_h,\theta'_h\rangle$ and $\|\theta_h'\|_2\le \|\theta_{h+1}'\|_2\le B'\sqrt {\dlc}
$ satisfies the norm constraint, i.e., $\langle \philc_h,\theta'_h\rangle\in\Qcal_h(\{\philc\},B')$.

Next, we show the formal corollary statement and the detailed proof of the theoretical result of Q-type \rfolive when instantiated to linear completeness setting.

\begin{corollary*}[Full version of \pref{corr:linear_complete}]
Fix $\delta \in (0,1)$. Consider an MDP $M$ that satisfies linear completeness (\pref{def:lin_completeness}) with the known feature $\philc$, and the linear reward class $\Rcal=\Rcal_1\times\ldots\times\Rcal_h$, where $\Rcal_h=\cbr{\langle\philc_h,\eta_h\rangle:\|\eta_h\|_2\le\sqrt{\dlc},\langle\philc_h(\cdot),\eta_h\rangle\in[0,1] }$. With probability at least $1-\delta$, for any $R\in\Rcal$, Q-type \rfolive (\pref{alg:rf_olive}) with $\Fcal=\Fcal(\{\philc\})$ outputs a policy $\hat\pi$ that satisfies $v^{\hat \pi}_R \ge v^*_R - \veps$ . The required number of episodes is
\begin{align*}
\tilde O\rbr{\frac{H^8\dlc^3\log(1/\delta)}{\varepsilon^2}}.
\end{align*} 
In \rfolive, we set
\[\eactv=\frac{\varepsilon}{2H^2},  \eelim=\frac{\varepsilon}{8H^2\sqrt{\dlc \iota}}, \nactv=\frac{H^6 \iota}{\varepsilon^2}, \nelim=\frac{H^7\dlc^2 \iota^3}{\varepsilon^2},
\]
where $\iota=\ce\log(H\dlc/\delta\varepsilon)$ and $\ce$ is a large enough constant. 
\end{corollary*}

We remark that although the $\tilde O\rbr{\frac{H^5\dlc^3\log(1/\delta)}{\varepsilon^2}}$ sample complexity rate of \citet{zanette2020provably} looks better than us, they need to assume $\veps\le \tilde O(\nu_{\min}/\sqrt {\dlc})$, where $\nu_{\min}$ is their explorability factor. Thus there is an implicit dependence on $1/\nu_{\min}$ in their sample complexity bound and their results are incomparable to us. More related discussions can be found in \pref{app:comparison}.

\begin{proof}
We first verify that $\Fcal(\{\philc\})$ satisfies the assumptions in \pref{thm:rfolive_q}. Here we have that 
$\Fcal(\{\philc\})=\Fcal_0(\{\philc\},H-1)\times\ldots\times \Fcal_{H-1}(\{\philc\},0)$, where 
\[
\Fcal_h(\{\philc\},B_h) = \big\{f_h(x_h,a_h) = \inner{\philc_h(x_h,a_h)}{\theta_h} : \|\theta_h\|_2 \le B_h\sqrt{\dlc},\langle\philc_h(\cdot),\theta_h\rangle\in[-B_h, B_h] \big\}.
\] 

We first verify the realizability assumption (\pref{assum:realizability_F}). For the last level, we have \[
Q^*_{R,H-1}=R_{H-1}+\zero\in\Fcal_{H-1}(\{\philc\},0)+R_{H-1}=\Fcal_{H-1}+R_{H-1}.
\] 
In addition, $Q^*_{R,H-1}=R_{H-1}=\langle\philc_{H-1},\eta_{H-1}\rangle$, where $\|\eta_{H-1}\|_2\le \sqrt {\dlc}$ and $\langle \philc_{H-1}(\cdot),\eta_{H-1} \rangle\in[0,1]$. Then for level $H-2$, we have 
\begin{align*}
    Q^*_{R,H-2}(x_{H-2},a_{H-2}) = {} & R_{H-2}(x_{H-2},a_{H-2})+\EE[\max_{a_{H-1}}Q^*_{R,H-1}(x_{H-1},a_{H-1})\mid x_{H-2},a_{H-2}] \\ 
    = {} & R_{H-2}(x_{H-2},a_{H-2})+\langle \philc_{H-2}(x_{H-2},a_{H-2}),\theta_{H-2}'\rangle , 
\end{align*}
where $\|\theta_{H-2}'\|_2\le \sqrt {\dlc}$ and $\langle \philc_{H-2}(\cdot),\theta_{H-2}'\rangle\in[0, 1]$. Here we apply the property of linear completeness (\pref{def:lin_completeness}). Therefore, we can set $\theta_{H-2}^*=\theta_{H-2}'$ and get $Q_{R,H-2}^*=R_{H-2}+\langle\philc_{H-2},\theta_{H-2}^*\rangle\in\Fcal_{H-2}(\{\philc\},1)+R_{H-2}$. Continuing this induction process backward, we get $Q_{R,h}^*\in\Fcal_h+R_h, \forall h\in[H]$, thus $Q^*_R\in\Fcal+R$.

For completeness assumption (\pref{assum:completeness_F}), again from the property of linear completeness, for any $h\in[H],f_{h+1}\in\Fcal_{h+1},R_{h+1}\in\Rcal_{h+1}$, we have that 
\begin{align*}
(\Tcal_{h}^\zero (f_{h+1}+R_{h+1}))(x_h,a_h) = {} & \EE\sbr{\max_{a_{h+1}} \inner{ \philc_{h+1}(x_{h+1},a_{h+1})}{\theta_{h+1}+\eta_{h+1}} \mid x_h,a_h} \\
= {} & \inner{\philc_h(x_h,a_h)}{\theta_{f+R,h}},
\end{align*}
where $\norm{\theta_{f+R,h}}_2\le (H-h-1)\sqrt {\dlc}$ and $\inner{ \philc_h(\cdot)}{\theta_{f+R,h}} \in[-(H-h-1), H-h-1]$. Thus $\inner{\philc_h}{\theta_{f+R,h}} \in \Fcal_h$, which implies that for any $f_{h+1}\in\Fcal_{h+1},R_{h+1}\in\Rcal_{h+1}$ we have $\Tcal_{h}^\zero (f_{h+1}+R_{h+1})\in \Fcal_h$. Similarly, we can show $\Tcal_{h}^\zero f_{h+1}\in \Fcal_h$. 

Moreover, for any $f_{h+1},f_{h+1}'\in\Fcal_{h+1}$, we can assume that $f_{h+1}=\langle \philc_{h+1},\theta_{h+1} \rangle$ and $f_{h+1}'=\langle \philc_{h+1},\theta_{h+1}' \rangle$, where $\|\theta_{h+1}\|_2,\|\theta_{h+1}'\|_2\le (H-h-2)\sqrt{\dlc}$ and $\langle \philc_{h+1}(\cdot),\theta_{h+1} \rangle,\langle \philc_{h+1}(\cdot),\theta_{h+1}' \rangle\in[-(H-h-2),H-h-2]$. Therefore, we have $f_{h+1}(\cdot) - f'_{h+1}(\cdot) \in [-2(H-h-2), 2(H-h-2)]$. From linear completeness (\pref{def:lin_completeness}), we know that there exists $\theta_h''$ that satisfy $\langle \philc_{h},\theta_{h}'' \rangle=\Tcal^\zero_h\rbr{\langle \philc_{h+1},\theta_{h+1} -\theta_{h+1}'\rangle}$ and $\|\theta_h''\|_2\le 2(H-h-2)\sqrt {\dlc}$ with $\inner{\philc_h(\cdot)}{\theta_h''} \in [-2\rbr{H-h-2}, 2\rbr{H-h-2}]$. Now, choosing $\theta_h=\theta_h''/2$ and $\theta_h'=-\theta_h''/2$, we know that $\langle \philc_{h},\theta_{h} \rangle-\langle \philc_{h},\theta_{h}' \rangle=\Tcal_h^\zero(f_{h+1}-f_{h+1}')$ and $\langle \philc_{h},\theta_{h} \rangle,\langle \philc_{h},\theta_{h}' \rangle\in \Fcal_{h}$. Hence we have $\Tcal^\zero_{h}(f_{h+1}-f_{h+1}')\in \Fcal_{h} - \Fcal_{h}$.

Therefore, from the above discussions, we get that completeness holds.

Invoking \pref{thm:rfolive_q}, the covering number argument (\pref{lem:covering}), and the bound on Bellman Eluder dimension (\pref{prop:lc_qbe}), we know that the output policy is $\varepsilon$-optimal and the sample complexity is
\begin{align*}
&~\tilde	O\rbr{\frac{\rbr{H^7\log\rbr{\Ncal_{\Fcal}\rbr{\varepsilon/512H^2\sqrt{\dlc \iota}}}+H^5\log\rbr{\Ncal_{\Rcal}\rbr{\varepsilon/512H^2\sqrt{\dlc \iota}}}}\dlc^2  \iota^3}{\varepsilon^2}}\\
=&~\tilde O\rbr{\frac{H^8\dlc^3\log(1/\delta)}{\varepsilon^2}} \qedhere.
\end{align*}

\end{proof}

As a final remark, \citet{zanette2020provably} assume $R$ is unknown but linear in $\philc$. In this case, 
we can instead construct $\Foff(R)=\Fcal_0(\{\philc\},H)\times\ldots\times \Fcal_{H-1}(\{\philc\},1)$, where the norm bound and the value range bound in $\Fcal_h(\{\philc\},H-h)$ are larger than that in $\Fcal_h(\{\philc\},H-h-1)$, thus capturing the reward-appended functions. One can easily follow the proof of \pref{thm:rfolive_q} and get the sample complexity result when using this new $\Foff(R)$ in the offline phase of \rfolive. This variant has the same sample complexity as \pref{corr:linear_complete}.

\subsection{Q-type \rfolive for known representation linear MDPs}
\label{app:linear_mdp}

In this part, we instantiate the general theoretical guarantee of Q-type \rfolive (\pref{thm:rfolive_q}) to the linear MDP setting, where the transition dynamics satisfy the low-rank decomposition (\pref{def:lowrank}) and $\philr$ is known. We construct the function class $\Fcal(\{\philr\})$ as $\Fcal(\{\philr\})=\Fcal_0(\{\philr\},H-1)\times\ldots\times \Fcal_{H-1}(\{\philr\},0)$, where 
\[
\Fcal_h(\{\philr\},B_h) = \big\{f_h(x_h,a_h) = \inner{\philr_h(x_h,a_h)}{\theta_h} : \|\theta_h\|_2 \le B_h\sqrt{\dlr},\langle \philr_h(\cdot),\theta_h\rangle \in [-B_h,B_h]  \big\}.
\]
In the following, we state the sample complexity result.

\begin{corollary}[Q-type \rfolive for linear MDPs]
\label{corr:linear_mdp}
Fix $\delta \in (0,1)$. Consider an MDP $M$ that admits a low-rank factorization in \pref{def:lowrank} and the feature $\philr$ is known, and we are given a reward function class $\Rcal$. With probability at least $1-\delta$, for any reward function $R \in \Rcal$, running Q-type version of \rfolive (\pref{alg:rf_olive}) with $\Fcal=\Fcal(\{\philr\})$ outputs a policy $\hat\pi$ that satisfies $v^{\hat \pi}_R \ge v^*_R - \veps$. The required number of episodes is
\begin{align*}
\tilde O\rbr{\frac{\rbr{H^8\dlr^3+H^5\dlr^2\log(\Ncal_{\Rcal}(\veps/512H^2\sqrt{\dlr \iota}))}\log(1/\delta)}{\varepsilon^2}}.
\end{align*}
In \rfolive, we set 
\[
\eactv=\frac{\varepsilon}{2H^2}, \eelim=\frac{\varepsilon}{8H^2\sqrt{\dlr\iota}}, \nactv=\frac{H^6 \iota}{\varepsilon^2}, 
\]
and
\[
\nelim=\frac{(H^7\dlr^2+H^4\dlr\log(\Ncal_{\Rcal}(\eelim/64))) \iota^3}{\varepsilon^2}=\frac{(H^7\dlr^2+H^4\dlr\log(\Ncal_{\Rcal}(\veps/512H^2\sqrt{\dlr \iota}))) \iota^3}{\varepsilon^2},
\]
where $\iota=\cf\log(H\dlr/\delta\varepsilon)$ and $\cf$ is a large enough constant.
\end{corollary}
\paragraph{Remark}
If we consider the entire linear reward class $\Rcal=\Rcal_0\times\ldots\times\Rcal_{H-1}$, where $\Rcal_h=\{\langle\philr,\eta_h\rangle:\|\eta_h\|_2\le\sqrt{\dlr},\langle\philr_h,\eta_h\rangle\in(\Xcal\times\Acal\rightarrow[0,1])\}$, 
then invoking \pref{corr:linear_mdp} and applying the similar covering argument of \pref{lem:covering} on the entire linear reward function class $\Rcal$ yields the sample complexity
\begin{align*}
\tilde O\rbr{\frac{H^8\dlr^3\log(1/\delta)}{\varepsilon^2}}.
\end{align*}

Recently, \citet{wagenmaker2022reward} improved the sharpest rate in the reward-free linear MDPs to $\tilde O\rbr{\frac{\dlr H^5(\dlr+\log(1/\delta))}{\veps^2}+\frac{\dlr^{9/2}H^6\log^{7/2}(1/\delta)}{\veps}}$. Although the focus of our work is not to obtain the optimal rate, the sample complexity bound of \rfolive is also independent of $K$ and not much worse than the current state of the art. In the reward-aware setting, \golf has a sharper rate than the subroutine \olive under the completeness type assumption \citep{jin2021bellman}. Since in \rfolive we only collect data when running a single (zero) reward \olive during the online phase and completeness (\pref{assum:completeness_F}) is satisfied in our paper, we believe that there also exists a reward-free version of \golf (by running \golf with zero reward function in the online phase and performing function elimination in the offline phase) that can potentially improve an $H\dqbee$ factor compared with \rfolive, thus matching the optimal $\dlr$ dependence in linear MDPs.

\begin{proof}
Similar as the proof of \pref{corr:low_rank}, we can verify that \pref{assum:realizability_F} and \pref{assum:completeness_F} hold. Invoking \pref{thm:rfolive_q} and noticing that the covering number argument (\pref{lem:covering}) and the bound on Q-type Bellman Eluder dimension (\pref{prop:linear_qbe}) completes the proof.
\end{proof}

\section{V-type \rfolive results}
\label{app:rfolive_v}

In this section, we present the results related to V-type \rfolive. In \pref{app:olive_v}, we provide the theoretical guarantee of V-type \olive \citep{jiang2017contextual,jin2021bellman} for completeness. In \pref{app:general_v}, we show the detailed proof of the sample complexity bound of  V-type \rfolive (\pref{thm:rfolive_v}). In \pref{app:low_rank}, we discuss the instantiation of V-type \rfolive to low-rank MDPs. 

\subsection{V-type \olive}
\label{app:olive_v}
First, we state the sample complexity of V-type \olive. Similar as Q-type \olive, since we consider the uniformly bounded reward setting ($0\le r_h\le 1$) instead of bounded total reward setting ($\forall h\in[H], r_h\ge 0$ and $\sum_{h=0}^{H-1} r_h\le 1$), we need to pay an additional $H^2$ dependency in $\nactv$ and $\nelim$.

\begin{proposition}[Sample complexity of V-type \olive, modification of Theorem 23 in \citet{jin2021bellman}]
\label{prop:olive_v}
Assume ($\frac{\varepsilon}{128H\sqrt{\dvbee}}=\eelim/8$) single-reward approximate realizability holds for $\Fcal$ in \pref{assum:realizability_olive} and $\Fcal$ is finite. If we set 
$$\eactv=\frac{\varepsilon}{4H},\ \eelim=\frac{\varepsilon}{16H\sqrt{\dvbee}}, \nactv=\frac{H^4 \iota}{\varepsilon^2},\text{ and } \nelim=\frac{H^4\dvbee K\log(|\Fcal|) \iota}{\varepsilon^2}
$$
where $\dvbee=\dvbe^R\big(\Fcal,\Dcal_{\Fcal},\varepsilon/8H\big)$ and  $\iota=\cc\log(H \dvbee K/\delta\varepsilon)$, then with probability at least $1-\delta$, V-type \olive (Algorithm 4 in \citet{jin2021bellman}) with $\Fcal$ will output an $\varepsilon$-optimal policy (under a single reward $R$) using at most  $ O(\dvbee H(\nactv +\nelim))$ episodes. Here $\cc$ is a large enough constant.
\end{proposition}

\subsection{Proof of V-type \rfolive under general function approximation}
\label{app:general_v}
In this part, we first provide the general statement of \pref{thm:rfolive_v} and then show the detailed proof.

\begin{theorem*}[Full version of \pref{thm:rfolive_v}]
Fix $\delta \in (0,1)$. Given a reward class $\Rcal$ and a function class $\Fcal$ that satisfies \pref{assum:realizability_F} and \pref{assum:completeness_F}, with probability at least $1-\delta$, for any $R\in\Rcal$, V-type \rfolive (\pref{alg:rf_olive}) with $\Fcal$ outputs a policy $\hat\pi$ that satisfies $v^{\hat \pi}_R \ge v^*_R - \veps$. The required number of episodes is
\begin{align*}
	O\rbr{\frac{\rbr{H^7\log\rbr{\Ncal_{\Fcal}\rbr{\varepsilon/2048H^2\sqrt{\dvbee}}}+H^5\log\rbr{\Ncal_{\Rcal}\rbr{\varepsilon/2048H^2\sqrt{\dvbee}}}}\dvbee^2K \iota}{\varepsilon^2}}.
\end{align*}
In \rfolive, we set
\[\eactv=\frac{\varepsilon}{8H^2},\ \eelim=\frac{\varepsilon}{32H^2\sqrt{\dvbee}},\ \nactv=\frac{H^6 \iota}{\varepsilon^2},
\] and 
\begin{align*}
\nelim=&~\frac{(H^6\log\rbr{\Ncal_{\Fcal}(\eelim/64)}+H^4\log\rbr{\Ncal_{\Rcal}(\eelim/64)})\dvbee K \iota}{\varepsilon^2}\\
=&~
\frac{\rbr{H^6\log\rbr{\Ncal_{\Fcal}\rbr{\varepsilon/2048H^2\sqrt{\dvbee}}}+H^4\log\rbr{\Ncal_{\Rcal}\rbr{\varepsilon/2048H^2\sqrt{\dvbee}}}}\dvbee K \iota}{\varepsilon^2},
\end{align*}
where $\dvbee=\dvbe^\zero(\Fcal-\Fcal,\Dcal_{\Fcal-\Fcal},\veps/(8H))$, $\iota=\cd\log(H\dvbee K/\delta\varepsilon)$ and $\cd$ is a large enough constant.
\end{theorem*}

\begin{proof}
This proof follows the similar structure as the proof of \pref{thm:rfolive_q}. The major difference is now we consider a discretized function class $\Zon$ in the online phase and consider a class of policy $\Pon$ in the offline elimination.

When we construct $\Zon$ (an $(\ecover)$-cover of $\Fon$), w.l.o.g, we can assume $\zero\in\Zon$, therefore the approximate realizability (\pref{assum:realizability_olive}) holds. From the online phase of V-type \rfolive (\pref{alg:rf_olive}), we can see that this phase is equivalent to running V-type \rfolive (Algorithm 4 in \citet{jin2021bellman}) with the input function class $\Zon$, the specified parameters $\eactv,\eelim,\nelim,\nactv$, and under the reward function $R=\zero$. Then the sample complexity is immediately from our specified values of $\eactv,\eelim,\nactv,\nelim$ and \pref{prop:olive_v} as we only collect samples in the online phase. Notice that we have the bound  $\log\rbr{|\Zon|}\le2\log\rbr{\Ncal_\Fcal(\ecover)}$ and such a constant 2 is absorbed by large enough $\cd$. Therefore, it remains to show that the algorithm can indeed output an $\varepsilon$-optimal policy with probability $1-\delta$ in the offline phase. We will show the following three claims hold with probability at least $1-\delta$.

    \paragraph{Claim 1} For any $g\in \Foff(R)$, if $\exists h\in[H]$, s.t. $|\EcalV(g,\pi_g,h)|\ge\varepsilon/H$, then it will be eliminated in the offline phase. 
    \paragraph{Claim 2} $Q^*_R\in\Foff(R)$ and $Q^*_R$ will not be eliminated in the offline phase.
    \paragraph{Claim 3} At the end of the offline phase, picking the optimistic function from the survived value functions gives us $\varepsilon$-optimal policy. \vspace{1em}

Before showing these three claims, we first state show properties from the online phase of V-type \rfolive and the concentration results in the offline phase. 

\paragraph{Properties from the online phase of V-type \rfolive} From the equivalence between the online phase of V-type \rfolive (\pref{alg:rf_olive}) and V-type \olive (Algorithm 4 in \citet{jin2021bellman}) with reward $\zero$, we know that with probability at least $1-\delta/4$, the online phase terminates within $\dvbee H+1$ iterations. In addition, with probability at least $1-\delta/4$, the following properties (\cref{eq:rfolive_v_stop} and \cref{eq:rfolive_v_conc}) hold for the first $\dvbee H+1$ iterations:

(i) When the online phase exists at iteration $T$ in \pref{line:set_T} (i.e., the elimination procedure is not activated in \rfolive), for any $f\in\Fcal^T$, it predicts no more than $\veps/(2H)$ value:
\begin{align}
\label{eq:rfolive_v_stop}
V_{f}(x_0)\le V_{f^T}(x_0)=V_{f^T}(x_0)-V^{\pi_{f^T}}_{\zero,0}(x_0)=\sum_{h=0}^{H-1}\EcalV(f^T,\pi^T,h)<2H\eactv<\varepsilon/(2H).
\end{align}
The first equality is due to any policy evaluation has value 0 under under the reward function $\zero$. The second equality is due to the policy loss decomposition in \citet{jiang2017contextual}. The second inequality is adapted from the ``concentration in the activation procedure'' part of the proof for Theorem 23 in \citet{jin2021bellman}.

(ii) For $T\le \dvbee H+1$, the concentration argument holds for any $f\in\Zon$ and $t\in [T]$:
\begin{align}
\label{eq:rfolive_v_conc}
\abr{\hEcalVzero(f, \pi^t, h^t)  - \EcalVzero(f, \pi^t, h^t)} < \eelim/8.
\end{align}
This is from the ``concentration in the elimination procedure'' step of the proof for Theorem 23 in \citet{jin2021bellman} and we adapt it with our parameters. 

\paragraph{Concentration results in the offline phase}
Let $\overline\Rcal$ be an $(\eelim/64)$-cover of $\Rcal$. For every $R\in\Rcal$, let $ R^{\mathrm{c}}=\argmin_{R'\in\overline\Rcal}\max_{h\in[H]}\|R_h-R'_h\|_\infty$. First consider any fixed $\pi'\in\Pon$ and $R\in\overline\Rcal$. Let $\Zcal(R)$ be an ($\eelim/64$)-cover of $\Foff(R)$ with cardinality $\Ncal_{\Foff(R)}(\eelim/64)=\Ncal_{\Fcal}(\eelim/64)$. 
For every $g\in\Foff(R)$, let $ g^{\mathrm c}=\argmin_{g'\in\Zcal(R)}\max_{h\in[H]}\|g_h-g'_h\|_\infty$. Then we consider any fixed $(t,g')\in[T]\times\Zcal(R)$ and calculate the upper bound of the second moment for $g$
\begin{align*}
\frac{\one\sbr{a_{h^t}^{(i)} = \pi'_{h^t}(x_{h^t}^{(i)})}}{1/K} \rbr{g'_{h^t}\rbr{x_{h^t}^{(i)},a_{h^t}^{(i)}} - r_{h^t}^{(i)} - V_{g'}\rbr{x_{h^t+1}^{(i)}}}.
\end{align*}
Let $y(x_{h^t},a_{h^t},r_{h^t},x_{h^t+1})=g'_{h^t}\rbr{x_{h^t}^{(i)},a_{h^t}^{(i)}} - r_{h^t}^{(i)} - V_{g'}\rbr{x_{h^t+1}^{(i)}} 
\subseteq[-2H,2H]$, then we have
\begin{equation*}
\begin{aligned}
	& \EE\sbr{\rbr{K\one\sbr{a_{h^t}^{(i)} = \pi'_{h^t}\rbr{x_{h^t}^{(i)}}}y(x_{h^t},a_{h^t},r_{h^t},x_{h^t+1})}^2\mid x_{h^t}\sim \pi^t, a_{h^t} \sim \unif(\Acal)} & \\
	\le {} & 4H^2K^2 \EE\sbr{\one\sbr{a_{h^t}^{(i)} = \pi'_{h^t}\rbr{x_{h^t}^{(i)}}}\mid x_{h^t}\sim \pi^t, a_{h^t}\sim\unif(\Acal)} = 4H^2K. &
\end{aligned}
\end{equation*}
Applying Bernstein's inequality and noticing the variance of the random variable is upper bounded by the second moment, with probability at least $1-\frac{\delta}{2T\Ncal_{\Fcal}(\eelim/64)\Ncal_{\Rcal}(\eelim/64)|\Pon|}$ we have
\begin{align*}
&~\abr{\hat \Ecal^R(g',\pi^t, \pi', h^t)-\Ecal^R(g',\pi^t, \pi', h^t)}\\
\le&\sqrt{\frac{4H^2K\log(4T\Ncal_{\Fcal}(\eelim/64)\Ncal_{\Rcal}(\eelim/64)|\Pon|/\delta)}{\nelim}}\\
&~\quad +\frac{4HK \log(4T\Ncal_{\Fcal}(\eelim/64)\Ncal_{\Rcal}(\eelim/64)|\Pon|/\delta)}{3\nelim}\\
<&~\frac{\eelim}{8}.
\end{align*}
The second inequality follows from $\eelim=\veps/\rbr{32H^2\sqrt{\dvbee}}$, $\iota=\cd\log(H\dvbee K/\delta\varepsilon),$ and
\begin{align*}
\nelim=&~\frac{(H^6\log\rbr{\Ncal_{\Fcal}(\eelim/64)}+H^4\log\rbr{\Ncal_{\Rcal}(\eelim/64)})\dvbee K\cdot \iota}{\varepsilon^2}
\end{align*}
with $\cd$ in $\iota$ being chosen large enough. Here we also notice that $|\Pon|=|\Zon|$ and $\log\rbr{|\Zon|}=\log\rbr{\Ncal_{\Fon}(\eelim/64)}\le2\log\rbr{\Ncal_{\Fcal}(\eelim/64)}$.

Union bounding over $(t,g')\in[T]\times\Zcal(R)$, with probability at least $1-\frac{\delta}{2\Ncal_{\Rcal}(\eelim/64)|\Pon|}$, we have that for any fixed $\pi'\in\Pon,R\in\overline\Rcal$ and all $g'\in\Zcal(R)$, $t\in[T]$ 
\begin{align*}
\abr{\hat \Ecal^R(g',\pi^t, \pi', h^t) - \Ecal^R(g', \pi^t, \pi',h^t)} <\eelim/8.
\end{align*}
Union bounding over $\pi'\in\Pon, R\in\overline\Rcal$, we have that with probability at least $1-\delta/2$, for all $R\in\overline\Rcal,\pi'\in\Pon,g\in\Foff(R),t\in[T]$,
\begin{align*}
\abr{\hat \Ecal^R \rbr{g, \pi^t, \pi,h^t}  - \Ecal^R \rbr{g, \pi^t, \pi',h^t}}  < \eelim/8. 
\end{align*}
Therefore, with probability at least $1-\delta/2$, for all $R\in\Rcal,\pi'\in\Pon,g\in\Foff(R),t\in[T]$, we have
\begin{align}
\label{eq:conc_offline_v}
	& \abr{\hat \Ecal^R \rbr{g, \pi^t, \pi',h^t}  - \Ecal^R \rbr{g, \pi^t, \pi',h^t}} & 
	\notag \\
	\le {} & \abr{\hat \Ecal^R \rbr{g, \pi^t, \pi',h^t} - \hat \Ecal^{R^{\mathrm{c}}} \rbr{g, \pi^t, \pi',h^t}} + \abr{\hat \Ecal^{R^{\mathrm{c}}} \rbr{g, \pi^t, \pi',h^t}  - \Ecal^{R^{\mathrm{c}}} \rbr{g, \pi^t, \pi',h^t}} & 
	\notag \\
	{} & + \abr{\Ecal^{R^{\mathrm{c}}}\rbr{g, \pi^t, \pi',h^t}  - \Ecal^R \rbr{g, \pi^t, \pi',h^t} } & 
	\notag \\
	\le {} & \eelim/64 
	+\eelim/8
	+ \eelim/64 & 
	\notag \\
	< {} & \eelim/4. &
\end{align}

All statements in our subsequent proof are under the event that all the different high-probability events (the online phase terminates within $\dvbee H+1$ iterations, and \cref{eq:rfolive_v_stop}, \cref{eq:rfolive_v_conc}, \cref{eq:conc_offline_v} hold for the first $\dvbee H+1$ iterations) discussed above hold with a total failure probability of $\delta$. 

\paragraph{Proof of Claim 1}  We consider any $g\in\Foff(R)$ that satisfies $\exists h\in[H]$, such that $|\EcalV(g,\pi_g,h)|\ge\varepsilon/H$. Recall the definition of $\Foff(R)$, we know that $g$ can be written as $g=(g_0,\ldots,g_{H-1})=(f_0+R_0,\ldots,f_{H-1} + R_{H-1})=(f_0 +R_0,\ldots,f_{H-1} + R_{H-1})$.
We will discuss the positive average Bellman error and the negative average Bellman error case separately.

\paragraph{Case (i) of Claim 1} $\EcalV(g,\pi_g,h)=\EE[g_{h}(x_h,a_h) - R_h(x_h,a_h) - V_g(x_{h+1}) \mid a_{0:h} \sim \pi_g] \ge \varepsilon/H$.

Since $\EcalV(g,\pi_g,h)=\EcalQ(g,\pi_g,h)$,
similar as the proof of \pref{thm:rfolive_q}, we know that
\begin{align*}
\varepsilon/H 
\le\EE[\tilde f_h(x_h,a_h)\mid a_{0:h}\sim\pi_g]. \tag{$\tilde f_h\defeq f_h- \Tcal^\zero_h g_{h+1}$}
\end{align*}
Same as in the proof of \pref{thm:rfolive_q}, here we construct a function $\tilde f$ as
\begin{align*}
\tilde f_{h'}(x_{h'},a_{h'})\hspace{-.15em}=\hspace{-.15em}\left\{
\begin{aligned}
&(\Tcal_{h'}^\zero \tilde f_{h'+1})(x_{h'},a_{h'})\hspace{-.15em}=\hspace{-.15em}\EE[\max_{a}\tilde f_{h'+1}(x_{h'+1},a) \mid x_{h'},a_{h'}] & & 0\le h'\le h-1 \\
&f_h(x_h,a_h)- (\Tcal^\zero_h g_{h+1})(x_{h},a_{h})  & & h'=h \\
&0 & &h+1\le h'\le H-1.
\end{aligned}
\right.
\end{align*}
From the definition of V-type average Bellman error and the construction, we know that for any policy $\pi$ we can translate the zero reward V-type average Bellman error of a function $\tilde f\in\Fon$ with roll-in policy $\pi$ to the average Bellman error under $R$ for a function $g\in\Foff(R)$ with roll-in policy $\pi_{0:h-1} \circ \pi_{\tilde f,h}$ (\pref{def:abe}) as the following
\begin{align}
\label{eq:on_to_off}
&~\EcalVzero(\tilde f,\pi,h)\notag
\\
=&~\EE\sbr{\tilde f_h(x_h,a_h) -\zero-V_{\tilde f}(x_{h+1})\mid a_{0:h-1}\sim\pi,a_{h}\sim \pi_{\tilde f}} \notag
\\
=&~\EE\sbr{\tilde f_h(x_h,a_h) \mid a_{0:h-1}\sim\pi,a_{h}\sim \pi_{\tilde f}} \notag
\\
=&~\EE\sbr{f_h(x_h,a_h)-(\Tcal^\zero_h g_{h+1})(x_{h},a_{h})\mid a_{0:h-1}\sim\pi,a_{h}\sim\pi_{\tilde f}}\notag
\\
=&~\EE\sbr{g_{h}(x_h,a_h) - R_h(x_h,a_h) - V_g(x_{h+1})\mid a_{0:h-1}\sim\pi, a_{h}\sim \pi_{\tilde f}}\notag
\\
=&~\Ecal^R(g,\pi, \pi_{\tilde f},h)
\end{align}
where the second equality is due to $\tilde f_{h+1}=\zero$. 

As the construction of $\tilde f$ and the assumptions of $\Fcal$ are the same as that in Q-type \rfolive and we use the same $\Fon=\Fcal-\Fcal$ in both places, following the same proof of \pref{thm:rfolive_q} directly gives us that $\tilde f=(\tilde f_0,\ldots,\tilde f_{H-1})\in\Fon$ and $V_{\tilde f}(x_0) \ge \veps/H$. \vspace{1em}

Since in the online phase we use $\Zon$, which is an ($\ecover$)-cover of $\Fon$, we know that there exists $\tilde f^{\mathrm{c}}\in\Zon$ such that $\max_{h'\in[H]}\|\tilde f_{h'}-\tilde f^{\mathrm{c}}_{h'}\|_\infty\le\ecover\le \veps/(2H)$. Notice that since $\forall h+1\le h'\le H-1$ we have $\zero\in\Zcal_{h'}$ and $f_{h'}=\zero$, thus w.l.o.g., we can assume that $\tilde f^{\mathrm{c}}_{h'}=\zero,\forall h+1\le h'\le H-1$. 

From the definition of $\tilde f^{\mathrm{c}}$ and $ \tilde f_0(x_0,\pi_{\tilde f}(x_0))=V_{\tilde f}(x_0)\ge\veps/H$, we have
\begin{align*}
V_{\tilde f^{\mathrm{c}}}(x_0)=\tilde f_0^{\mathrm{c}}(x_0,\pi_{\tilde f^{\mathrm{c}}}(x_0)) \ge \tilde f_0^{\mathrm{c}}(x_0,\pi_{\tilde f}(x_0)) \ge  \tilde f_0(x_0,\pi_{\tilde f}(x_0)) -\veps/(2H)\ge\varepsilon/(2H).
\end{align*} 

From the first property of the online phase (\cref{eq:rfolive_v_stop}), we know that all the survived value functions at the end of the online phase predict no more than $\varepsilon/(2H)$. Therefore $\tilde f^{\mathrm{c}}$ will be eliminated. We assume it is eliminated at iteration $t$ by policy $\pi^t$ in level $h^t$. \vspace{1em} 

For any policy $\pi$ and $h'\in[H], h'\neq h$, we have
\begin{align*}
&~\EcalVzero\rbr{\tilde f^{\mathrm{c}},\pi,h'} 
\notag \\
=&~\EE[\tilde f^{\mathrm{c}}_{h'}(x_{h'},a_{h'}) -\zero-\tilde f^{\mathrm{c}}_{h'+1}(x_{h'+1},a_{h'+1})\mid a_{0:h'-1}\sim\pi,a_{h':h'+1}\sim \pi_{\tilde f^{\mathrm{c}}}] \notag\\
\ge&~\EE[\tilde f_{h'}(x_{h'},a_{h'}) -\tilde f_{h'+1}(x_{h'+1},a_{h'+1})\mid a_{0:h'-1}\sim\pi,a_{h':h'+1}\sim \pi_{\tilde f^{\mathrm{c}}}] - 2  \ecover
\notag\\
\ge&~\EE[\tilde f_{h'}(x_{h'},a_{h'}) -\tilde f_{h'+1}(x_{h'+1},a_{h'+1})\mid a_{0:h'-1}\sim\pi,a_{h'}\sim \pi_{\tilde f^{\mathrm{c}}},a_{h'+1}\sim \pi_{\tilde f}]- 2  \ecover
\notag\\
=&~-\eelim /32.
\end{align*}
The first inequality is from the definition of $\tilde f^{\mathrm{c}}$. The second inequality is due to $\pi_{\tilde f}$ is the greedy policy of $\tilde f$. The last equality is due to the construction of $\tilde f$.

Similarly, on the other end, we also have 
\begin{align*}
&~\EcalVzero\rbr{\tilde f^{\mathrm{c}},\pi,h'} 
\notag \\
=&~\EE[\tilde f^{\mathrm{c}}_{h'}(x_{h'},a_{h'}) -\zero-\tilde f^{\mathrm{c}}_{h'+1}(x_{h'+1},a_{h'+1})\mid a_{0:h'-1}\sim\pi,a_{h':h'+1}\sim \pi_{\tilde f^{\mathrm{c}}}] \notag\\
\le&~
\EE[\tilde f^{\mathrm{c}}_{h'}(x_{h'},a_{h'}) -\tilde f^{\mathrm{c}}_{h'+1}(x_{h'+1},a_{h'+1})\mid a_{0:h'-1}\sim\pi,a_{h'}\sim \pi_{\tilde f^{\mathrm{c}}},a_{h'+1}\sim \pi_{\tilde f}]
\notag\\
\le&~
\EE[\tilde f_{h'}(x_{h'},a_{h'}) -\tilde f_{h'+1}(x_{h'+1},a_{h'+1})\mid a_{0:h'-1}\sim\pi,a_{h'}\sim \pi_{\tilde f^{\mathrm{c}}},a_{h'+1}\sim \pi_{\tilde f}] + 2  \ecover
\notag\\
=&~\eelim /32.
\end{align*}
Therefore, for any policy $\pi$ and $h'\in[H],h'\neq h$ we get
\begin{align}
\label{eq:appx_f_avc}
    \abr{\EcalVzero\rbr{\tilde f^{\mathrm{c}},\pi,h'}}\le\eelim/32.
\end{align}
From the Bellman backup construction of $\tilde f^{\mathrm{c}}$, we know that $\tilde f^{\mathrm{c}}$ can only be eliminated at level $h$. This can be seen from the following argument: Applying the concentration result of the online phase (\cref{eq:rfolive_v_conc}), we have $|\hEcalVzero(\tilde f^{\mathrm{c}},\pi^t,h^t)-\EcalVzero(\tilde f^{\mathrm{c}},\pi^t,h^t)|\le 3\eelim/4$. Further notice that \cref{eq:appx_f_avc}, we have
$|\hEcalVzero(\tilde f^{\mathrm{c}},\pi^t,h^t)|\le 3\eelim/4+\eelim/32$ 
if $h^t\neq h$. Since the elimination threshold is set to $\eelim$, $\tilde f^{\mathrm{c}}$ will not be eliminated at level $h^t\neq h$. \vspace{1em}

This implies that at some iteration $t$ in the online phase, we will collect some $\pi^t$ that eliminates $\tilde f^{\mathrm{c}}$ at level $h$, i.e., it satisfies $|\hEcalVzero(\tilde f^{\mathrm{c}},\pi^t,h^t)|> \eelim$ and $h^t=h$. Applying concentration argument for the online phase (\cref{eq:rfolive_v_conc}), we have $|\hEcalVzero(\tilde f^{\mathrm{c}},\pi^t,h^t)-\EcalVzero(\tilde f^{\mathrm{c}},\pi^t,h^t)|\le 3\eelim/16$. Therefore, \begin{align}
\label{eq:sep_case}
    |\EcalVzero(\tilde f^{\mathrm{c}},\pi^t,h^t)|> 13\eelim/16.
\end{align}
From the definition of the average Bellman error and $\tilde f,\tilde f^{\mathrm{c}}$, we have the following equations
\begin{align*}
\EcalVzero(\tilde f^{\mathrm{c}},\pi^t,h^t) = {} & \EE\sbr{\tilde f^{\mathrm{c}}_{h}(x_{h},a_{h})-\zero-\tilde f^{\mathrm{c}}_{h+1}(x_{h+1},a_{h+1})\mid a_{0:h-1}\sim\pi^t, a_{h:h+1}\sim  \pi_{\tilde f^{\mathrm{c}}}}\\
= {} & \EE\sbr{\tilde f^{\mathrm{c}}_{h}(x_{h},a_{h})\mid a_{0:h-1}\sim\pi^t, a_{h}\sim \pi_{\tilde f^\mathrm{c}}} \tag{$\tilde f^{\mathrm{c}}_{h+1}=\zero$},
\end{align*}
and
\begin{align*}
& \Ecal^R(g,\pi^t, \pi',h^t) & \\
= {} & \EE\sbr{g_{h}(x_h,a_h) - R_h(x_h,a_h) - V_g(x_{h+1})\mid a_{0:h-1}\sim\pi^t, a_{h}\sim \pi'} & \\
= {} & \EE\sbr{\tilde f_{h}(x_{h},a_{h}) \mid a_{0:h-1}\sim\pi^t, a_{h}\sim \pi'} \tag{\cref{eq:on_to_off}} & \\
\ge {} & \EE\sbr{\tilde f_{h}^{\mathrm{c}}(x_{h},a_{h}) \mid a_{0:h-1}\sim\pi^t, a_{h}\sim \pi'} - 2\ecover,
\end{align*}
where $\pi'\in \Pi_{\textrm{est}}^t=\Pon$.
Because $\pi_{\tilde f^\mathrm{c}}$ is the greedy policy of $\tilde f^\mathrm{c}$ and $\tilde f^\mathrm{c}\in\Zon$, we know that in the offline phase $\pi'=\pi_{\tilde f^\mathrm{c}}\in\Pon$ will be chosen for elimination. Then we get 
\begin{align*}
{} & \Ecal^R \rbr{g, \pi^t, \pi_{\tilde f^{\mathrm{c}}},h^t} 
\\
\ge {} &\EE\sbr{ \tilde f_{h}^{\mathrm{c}}(x_{h},a_{h}) \mid a_{0:h-1} \sim \pi^t, a_{h} \sim \pi_{\tilde f^\mathrm{c}}} - 2\ecover 
\\
= {} & \EE\sbr{ \tilde f_{h}^{\mathrm{c}}(x_{h},a_{h})-\zero -\tilde f_{h+1}^{\mathrm{c}}(x_{h+1},a_{h+1}) \mid a_{0:h-1} \sim \pi^t, a_{h:h+1} \sim \pi_{\tilde f^\mathrm{c}}} - 2\ecover
\\
= {} & \EcalVzero(\tilde f^\mathrm{c},\pi^t,h^t)-\eelim/32,
\end{align*}
where the first equality is due to $\tilde f^{\mathrm{c}}_{h+1}=\zero$.

Similarly we have
\begin{align*}
{} & \Ecal^R \rbr{g, \pi^t, \pi_{\tilde f^{\mathrm{c}}},h^t} 
\\
\le {} &\EE\sbr{ \tilde f_{h}^{\mathrm{c}}(x_{h},a_{h}) \mid a_{0:h-1} \sim \pi^t, a_{h} \sim \pi_{\tilde f^\mathrm{c}}} + 2\ecover
\\
= {} & \EE\sbr{ \tilde f_{h}^{\mathrm{c}}(x_{h},a_{h})-\zero -\tilde f_{h+1}^{\mathrm{c}}(x_{h+1},a_{h+1}) \mid a_{0:h-1} \sim \pi^t, a_{h:h+1} \sim \pi_{\tilde f^\mathrm{c}}} + 2\ecover
\\
= {} & \EcalVzero(\tilde f^\mathrm{c},\pi^t,h^t)+\eelim/32.
\end{align*}
Further, using the concentration argument for the offline phase (\cref{eq:conc_offline_v}), we get 
\begin{align*}
\abr{ \hat \Ecal^R \rbr{g,\pi^t, \pi_{\tilde f^{\mathrm{c}}},h^t} -\Ecal^R \rbr{g,\pi^t, \pi_{\tilde f^{\mathrm{c}}},h^t}}< \eelim/4.
\end{align*}
Hence, if $\EcalVzero(\tilde f^\mathrm{c},\pi^t,h^t) \ge 0$, we get $\EcalVzero(\tilde f^\mathrm{c},\pi^t,h^t) \ge 13\eelim/16$ from \cref{eq:sep_case}, which yields
\begin{align*}
\hat \Ecal^R\rbr{g,\pi^t, \pi_{\tilde f^{\mathrm{c}}},h^t} >&~  \Ecal^R\rbr{g,\pi^t, \pi_{\tilde f^{\mathrm{c}}},h^t} - \eelim/4
\\
\ge&~\EcalVzero(\tilde f^\mathrm{c},\pi^t,h^t)-\eelim/32 -\eelim/4
\\
\ge&~13\eelim/16-\eelim/32-\eelim/4>\eelim/2.
\end{align*}
Otherwise, we are in the case $\EcalVzero(\tilde f^\mathrm{c},\pi^t,h^t) < 0$ and we have $\EcalVzero(\tilde f^\mathrm{c},\pi^t,h^t) < -13\eelim/16$  from \cref{eq:sep_case}. This yields
\begin{align*}
\hat \Ecal^R\rbr{g,\pi^t, \pi_{\tilde f^{\mathrm{c}}},h^t} < &~  \Ecal^R\rbr{g,\pi^t, \pi_{\tilde f^{\mathrm{c}}},h^t} + \eelim/4
\\
\le&~\EcalVzero(\tilde f^\mathrm{c},\pi^t,h^t)+\eelim/32 +\eelim/4
\\
\le&~-13\eelim/16+\eelim/32+\eelim/4<-\eelim/2.
\end{align*}
Thus we always have $\abr{\hat \Ecal^R\rbr{g,\pi^t, \pi_{\tilde f^{\mathrm{c}}},h^t}}>\eelim /2$, which implies that we eliminate such $g$ by $\pi^t_{0:h^t-1}\circ\pi_{\tilde f^{\mathrm{c}},h^t}$ in the offline phase.

\paragraph{Case (ii) of Claim 1} $\EcalV(g,\pi_g,h)=\EE[g_{h}(x_h,a_h) - R_h(x_h,a_h) - V_g(x_{h+1})\mid a_{0:h} \sim \pi_g]\le-\varepsilon/H$.

Same as before, we have $\EcalV(g,\pi_g,h)=\EE[f_h(x_h,a_h)- (\Tcal^\zero_h g_{h+1})(x_{h},a_{h})\mid a_{0:h}\sim\pi_g]\le-\varepsilon/H$. Now we let $\tilde f_h$ be the negated version of the one in case (i), and define $\tilde  f$ as 
\begin{align*}
	\tilde f_{h'}(x_{h'},a_{h'})\hspace{-.15em}=\hspace{-.15em}\left\{
	\begin{aligned}
		&(\Tcal_{h'}^\zero\tilde g_{h'+1})(x_{h'},a_{h'})\hspace{-.15em}=\hspace{-.15em}\EE[\max_{a}\tilde g_{h'+1}(x_{h'+1},a)\mid x_{h'},a_{h'}] & & 0\le h'\le h-1 
		\\
		&(\Tcal^\zero_h g_{h+1})(x_h,a_h) - f_h(x_h,a_h)& & h'=h 
		\\
		&0 & &h+1\le h'\le H-1.
	\end{aligned}
	\right.
\end{align*}
Following the same steps as in case (i) we can verify $\tilde f\in\Fon$, construct $\tilde f^{\mathrm{c}}\in\Zon$ with $V_{\tilde f^{\mathrm{c}}}(x_0)\ge \veps/(2H)$, and show that $g$ is eliminated by $\pi^t_{0:h^t-1}\circ\pi_{\tilde f^{\mathrm{c}},h^t}$ in the offline phase for some $t,h^t$. 

\paragraph{Proof of Claim 2} (i) From the assumption, we know that realizability condition $Q^*_R=(Q^*_{R,0},\ldots,$ $Q^*_{R,H-1})\in\Foff(R)$ holds. (ii) For the second argument, we note that $\Ecal(Q^*_R,\pi,\pi',h)= 0$ for any $\pi$, $h\in[H]$, $\pi'\in\Pi_{\textrm{est}}^t$ by the definition of the average Bellman error. From the concentration argument in the offline phase (\cref{eq:conc_offline_v}), we have $| \hat \Ecal^R \rbr{Q^*_R,\pi^t, \pi', h^t}| \le  | \Ecal^R \rbr{Q^*_R, \pi^t, \pi', h^t}|  +\eelim/4 = \eelim/4.$ As a result, $Q^*_R$ will not be eliminated. 

\paragraph{Proof of Claim 3} From Claim 1, we know that in the offline phase, for any $g\in \Foff(R)$, if $\exists h\in[H]$, s.t. $|\EcalV(g,\pi_g,h)|\ge\varepsilon/H$, then it will be eliminated. Therefore from the policy loss decomposition in \citet{jiang2017contextual}, for all survived $g\in\Fsur(R)$ in the offline phase, we have 
\[
V_{g}(x_0)-V^{\pi_{g}}_{R,0}(x_0)=\sum_{h=0}^{H-1}\EcalV(g,\pi_{g},h)< \varepsilon.
\]
Since $Q^*_R$ is not eliminated, similar as \citet{jiang2017contextual}, we have 
\begin{align*}
V^{\pi_{\hat g}}_{R,0}(x_0) &> V_{\hat g}(x_0)-\varepsilon\ge V^*_{R,0}(x_0)-\varepsilon.
\end{align*}
In sum, we can see the three claims hold with probability at least $1-\delta$. Since Claim 3 directly implies that \rfolive returns an $\veps$-near optimal policy, we complete the proof.
\end{proof}

\subsection{V-type \rfolive for unknown representation low-rank MDPs}
\label{app:low_rank}
Here we provide the details of instantiating V-type \rfolive to low-rank MDPs \citep{agarwal2020flambe,modi2021model,uehara2021representation}. Firstly we remark that they assume the normalization in \pref{def:lowrank} holds for $f':\Xcal\rightarrow[0,1]$ instead of $f':\Xcal\rightarrow[-1,1]$. We use the different version for ease of presentation and our results also hold under their normalization. In addition, both versions are implied by the definition in \citet{jin2019provably}.

Now we show the complete corollary statement. 

\begin{corollary*}[Full version of \pref{corr:low_rank}]
Fix $\delta \in (0,1)$. Consider a low-rank MDP $M$ of embedding dimension $\dlr$ with a realizable feature class $\Philr$ (\pref{assum:realizability_low_rank}) and a reward function class $\Rcal$.
With probability at least $1-\delta$, for any $R \in \Rcal$, V-type \rfolive (\pref{alg:rf_olive}) with $\Fcal(\Philr)$ outputs a policy $\hat\pi$ that satisfies $v^{\hat \pi}_R \ge v^*_R - \veps$ . The required number of episodes is
\begin{align*}
\tilde O\rbr{\frac{\rbr{H^8\dlr^3\log(|\Philr|)+H^5\dlr^2\log(\Ncal_{\Rcal}(\veps/2048H^2 \sqrt{\dlr\iota}))}K\log(1/\delta)}{\varepsilon^2}}.
\end{align*}
In \rfolive, we 
\[\eactv=\frac{\varepsilon}{8H^2}, \eelim=\frac{\varepsilon}{32H^2\sqrt{\dlr \iota}},\ \nactv=\frac{H^6 \iota}{\varepsilon^2}
\] 
and 
\begin{align*}
\nelim=&~\frac{\rbr{H^7\dlr^2\log(|\Philr|)+H^4\dlr \log(\Ncal_{\Rcal}(\eelim/64))}K  \iota^3}{\varepsilon^2}\\
=&~\frac{\rbr{H^7\dlr^2\log(|\Philr|)+H^4\dlr \log(\Ncal_{\Rcal}(\veps/2048H^2\sqrt{\dlr \iota}))}K  \iota^3}{\varepsilon^2},
\end{align*}
where $\iota=\cg\log(HdK/\delta\varepsilon)$ and $\cg$ is large enough constant.
\end{corollary*}

Before the formal proof, we provide some discussions and comparisons. Firstly, when $\Rcal$ is finite, the bound becomes $
\tilde O\rbr{\frac{\rbr{H^8\dlr^3\log(|\Philr|)+H^5\dlr^2\log(|\Rcal|)}K\log(1/\delta)}{\varepsilon^2}}$. Compared with \citet{modi2021model}, our result significantly improves upon their $\tilde O\rbr{\frac{H^6\dlr^{11}K^{14}\log(|\Philr|/\delta)}{\etamin^5}+\frac{H^7\dlr^{3}K^{5}\log(|\Philr||\Rcal|/\delta)}{\veps^2\etamin}}$ rate and does not require the reachability assumption ($\etamin$ is their reachability factor). On the other hand, their algorithm is more computationally viable and achieves the optimal deployment complexity \citep{huang2021towards}. With the additional access to and the realizability assumption of the right feature candidate class $\Upsilon^{\mathrm{lr}}$ in low-rank MDPs, another related work \citet{agarwal2020flambe} provide a computationally efficient reward-free exploration guarantee but their rate $\frac{H^{22}\dlr^7 K^9\log(|\Philr||\Upsilon^{\mathrm{lr}}|/\delta)}{\varepsilon^{10}}$ is also much worse than ours. 

In the sequel, we present the detailed proof for the corollary.

\begin{proof}
We first verify that $\Fcal$ satisfies the assumptions in \pref{thm:rfolive_v}. Here we have that $\Fcal=\Fcal(\Philr)=\Fcal_0(\Philr,H-1)\times\ldots\times \Fcal_{H-1}(\Philr,0)$, where 
\begin{align*}
    \Fcal_h(\Philr,B_h) = \cbr{ \hspace{-.15em} f_h(x_h,a_h) = \inner{\philr_h(x_h,a_h)}{\theta_h}\hspace{-.15em} : \|\theta_h\|_2 \le B_h\sqrt{\dlr},\langle \philr_h(\cdot),\theta_h \rangle \in[-B_h, B_h] \hspace{-.15em} }.
\end{align*}
Applying the property of linear MDPs (\pref{lem:linear_mdp}) gives us that
\begin{align*}
Q^*_{R,h}(x_h,a_h) = {} & R_h(x_h,a_h)+\EE\sbr{\max_{a_{h+1}}Q^*_{R,h+1}(x_{h+1},a_{h+1}) \mid x_h,a_h} \\ 
= {} & R_h(x_h,a_h)+\langle \philr_h(x_h,a_h),\theta_h^*\rangle ,    
\end{align*}
where $\|\theta_h^*\|_2\le (H-(h+1))\sqrt {\dlr}$ and $\langle \philr_h(\cdot),\theta_h^*\rangle \in[0, H-(h+1)]$. Therefore, for any $h\in[H]$, we have $Q_{R,h}^*\in\Fcal_h(\Philr,H-h-1)+R_h=\Fcal_h+R_h$. This implies that for any $R\in\Rcal$, we get $Q^*_R\in\Fcal+R$, thus, realizability (\pref{assum:realizability_F}) holds. 

Again, applying \pref{lem:linear_mdp}, for any $h\in[H],f_{h+1}\in\Fcal_{h+1},R_{h+1}\in\Rcal_{h+1}$, we have that 
\begin{align*}
(\Tcal_{h}^\zero (f_{h+1}+R_{h+1}))(x_h,a_h) = {} & \EE\sbr{\max_{a_{h+1}}(f_{h+1}(x_{h+1},a_{h+1})+R_{h+1}(x_{h+1},a_{h+1}))\mid x_h,a_h} \\
= {} & \inner{\philr_h(x_h,a_h)}{\theta^*_{f+R,h}} ,    
\end{align*}
where $\|\theta^*_{f+R,h}\|_2\le (H-h-1)\sqrt {\dlr}$ and $\langle\philr_h(\cdot),\theta^*_{f+R,h}\rangle \in[-(H-h-1), H-h-1]$. Thus $\langle\philr_h,\theta^*_{f+R,h}\rangle\in\Fcal_h$. This implies that for any $f_{h+1}\in\Fcal_{h+1},R_{h+1}\in\Rcal_{h+1}$, we have $\Tcal_{h}^\zero (f_{h+1}+R_{h+1})\in \Fcal_h$. Similarly, we can show $\Tcal_{h}^\zero f_{h+1}\in \Fcal_h$. 

Moreover, for any $f_{h+1},f'_{h+1}\in\Fcal_{h+1}$, we have that $\|f_{h+1}-f'_{h+1}\|_\infty \le 2(H-h-2)$. Therefore, there exists $\theta^*_{f-f',h}$ such that $\Tcal_h^\zero(f_{h+1}-f'_{h+1})=\langle\philr_h,\theta^*_{f-f',h}\rangle\subseteq (\Xcal\times\Acal\rightarrow[-2(H-h-2),2(H-h-2)])$ and $\|\theta^*_{f-f',h}\|_2\le 2(H-h-2)\sqrt {\dlr}$. Then choosing $\theta_h=\theta^*_{f-f',h}/2$, $\theta_h'=-\theta^*_{f-f',h}/2$ and $f_h=\langle\philr_h,\theta_h\rangle,f_h'=\langle\philr_h,\theta_h'\rangle$ gives us both $f_h(\cdot), f_h'(\cdot) \in [-(H-h-2),(H-h-2)] \subseteq [-(H-h-1), H-h-1]$ and $\|\theta_h\|_2,\|\theta_h'\|_2\le (H-h-1)\sqrt{\dlr}$. Therefore, we have that $f_h - f_h' \in \Fcal_h - \Fcal_h$.

The above discussions imply that completeness (\pref{assum:completeness_F}) holds.

Invoking \pref{thm:rfolive_v} and further noticing the covering number argument (\pref{lem:covering}) and the bound on V-type Bellman Eluder dimension (\pref{prop:lr_vbe}), we know that the output policy is $\varepsilon$-optimal and the sample complexity is
\begin{align*}
&~\tilde O\rbr{\frac{\rbr{H^7\log(\Ncal_{\Fcal}(\veps/2048H^2\sqrt{\dlr \iota})) \iota+H^5\log(\Ncal_{\Rcal}(\veps/2048H^2\sqrt{\dlr \iota}))} \dlr^2 K \iota^3}{\varepsilon^2}}
\\
=&~\tilde O\rbr{\frac{\rbr{H^8\dlr^3\log(|\Philr|)+H^5\dlr^2\log(\Ncal_{\Rcal}(\veps/2048H^2\sqrt{\dlr \iota}))}K\log(1/\delta)}{\varepsilon^2}}.\qedhere
\end{align*}
\end{proof}

\section{Hardness result for unknown representation linear completeness setting}
\label{app:rf_lower_bound}
In this section, we provide the detailed construction and proof for the hardness result and more discussions.
\begin{theorem*}[Restatement of \pref{thm:rf_lower_bound}]
There exists a family of MDPs $\Mcal$, a reward class $\Rcal$ and a feature set $\Philc$, such that  $\forall M\in\Mcal$, the $(M, \Philc)$ pair satisfies \pref{assum:realizability_linear_completeness} (i.e., $\Philc$ is realizable linear complete feature class for any $M\in\Mcal$), yet it is information-theoretically impossible for an algorithm to obtain a $\poly\rbr{\dlc, H, \log (|\Philc|), \log (|\Rcal|), 1/\veps, \log (1/\delta)}$ sample complexity for reward-free exploration with the given reward class $\Rcal$.
\end{theorem*}
\begin{proof}
We present an exponential tree MDP as a hard instance, similar to the lower bound instances in \citet{modi2019sample}, and design ``one-hot'' realizable feature inspired by the construction in \citet{zanette2020learning} which they used to show that a low-IBE (Inherent Bellman Error) setting does not imply a low-rank/linear MDP. 

\paragraph{Family of hard instances} We consider a class of deterministic finite state MDPs $\Mcal$ with a singleton reward class. 
In our construction, for simplicity, the MDPs have a layered structure where the set of states an agent can encounter at any two timesteps $h$ and $h'$ ($h\neq h'$) are disjoint. Hence, we denote the respective state spaces for each timestep $h$ as $\Xcal_h$, and we always have $x_h\in\Xcal_h$. In this layered MDP, for each timestep $h\in[H]$, we only define the corresponding feature $\phi$, rewards at each $\Xcal_h$, and transition functions from $\Xcal_h$ to $\Xcal_{h+1}$. To convert it to the non-layered MDPs, at each timestep $h\in[H]$, we only need to let the features $\phi$ and reward functions be $\zero$ at the states outside $\Xcal_h$ and let transitions have 0 probability when transiting from states in $\Xcal_h$ to states outside $\Xcal_{h+1}$ and define the transition functions from some states outside $\Xcal_h$ arbitrarily.

Consider a complete binary tree of depth $H-2$ (we count the first layer $x_0$ as depth 0). The vertices at each level $h$ from the state space $\Xcal_h$ and the two outgoing edges represent the available actions at each state. The reward class is a singleton class $\{R\}$, where all states get zero reward other than $R_{H-1}(x^+,\mathrm{NULL})=1$. The starting state of the MDP is the root node $x_0$ and the dynamics are deterministic at all levels: each action $\cbr{\mathrm{left, right}}$ transits to the corresponding child node. Of all the $2^{H-2}$ nodes at level $H-2$, on one node $x_{H-2}^*$, one action $a_{H-2}^*$ transits to $x^+$ with probability $1$ whereas the other action and all actions for other nodes transit to $x^-$ (i.e., only $P_{H-1}(x^+\mid x_{H-2}^*,a_{H-2}^*)=1$). 
As we have $2^{H-1}$ many choices for $(x_{H-2}^*, a_{H-2}^*)$, we have $\abr{\Mcal} = 2^{H-1}$. We provide an illustration for $H=4$ in \pref{fig:hard_instance}.

\begin{figure}[htb]
	\center
	\begin{tikzpicture}[scale=1.75]
		\node[state, initial] (s0) at (0,0)     {$x_0$};
		\node[state] (s1) at (-1.5,-1)  {$x_1^0$};
		\node[state] (s2) at (1.5,-1) {$x_1^1$};
		\node[state] (s3) at (-2.5,-2)  {$x_2^0$};
		\node[state] (s4) at (-0.5,-2) {$x_2^1$};
		\node[state] (s5) at (0.5,-2)  {$x_2^2$};
		\node[state] (s6) at (2.5,-2) {$x_2^3$};
		\node[state, accepting] (s+) at (-1.5,-4) {\textcolor{OliveGreen}{$\boldsymbol{x^+}$}};
		\node[state, accepting] (s-) at (0.5,-4) {\textcolor{red}{$\boldsymbol{x^-}$}};
		
		\draw[->] (s0) edge (s1);
		\draw[->, dashed] (s0) edge (s2);
		\draw[->] (s1) edge (s3);
		\draw[->, dashed] (s1) edge (s4);
		\draw[->] (s2) edge (s5);
		\draw[->, dashed] (s2) edge (s6);
		\draw[->, dashed] (s3) edge (s-);
		\draw[->] (s3) edge[bend right] (s-);
		\draw[->] (s4) edge (s+);
		\draw[->, dashed] (s4) edge (s-);
		\draw[->] (s5) edge (s-);
		\draw[->, dashed] (s5) edge[bend left] (s-);
		\draw[->] (s6) edge (s-);
		\draw[->, dashed] (s6) edge[bend left] (s-);
	\end{tikzpicture}
\caption{A hard instance for $H=4$ with two actions: $\mathrm{left}$ (solid arrow) and $\mathrm{right}$ (dashed arrow). The complete binary tree portion ranges from timestep $0$ to $2$, and $x^+,x^-$ belong to timestep 3. On timestep $h=2$, only $(x^*_2, a_2^*) = (x_2^1, \mathrm{left})$ transits to the good state $x^+$, and all other state-action pairs transit to bad state $x^-$. Rewards for all state-action pairs are $0$ other than $R_3(x^+,\mathrm{NULL})=1$.}
\label{fig:hard_instance}
\end{figure}
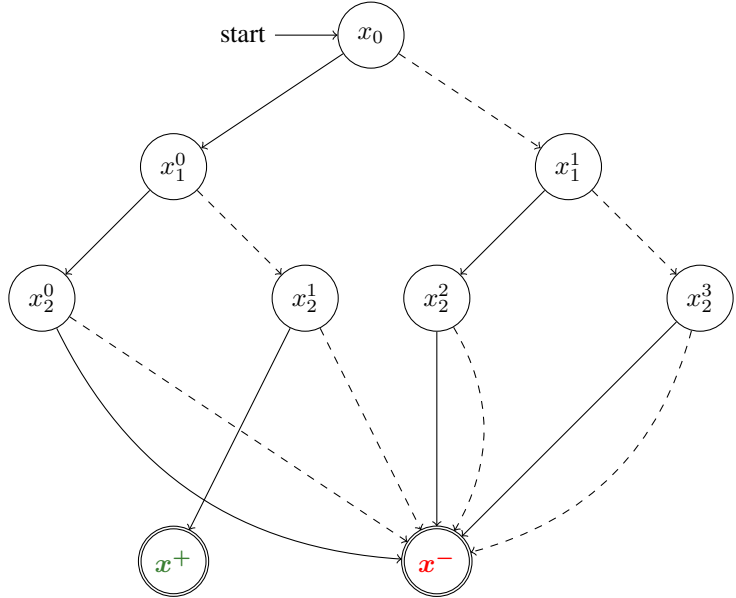

\paragraph{Constructing the feature class} We now construct a feature class $\Philc$ such that
for any MDP $M\in\Mcal$, $\Philc$ satisfies \pref{assum:realizability_linear_completeness} (i.e., the linearly complete feature under $M$ belongs to $\Philc$). We define the feature class in the following way: for each timestep $h \in [H-1]$, we define $\Philc_h = \big\{\phi_h^i: i \in [2^{h+1}], \phi_h^i[j,a] = \one[i=2\cdot j+a],\forall j \in [2^h], a \in \{0,1\}\big\} $, where $\phi_h^i[j,a]$ denotes the value of feature $\phi_h^i$ on the $j$-th state ($x_h^j$) and action $a$ at level $h$. Finally, the two nodes at timestep $H-1$ have a feature value of $\phi_{H-1}(x^+,\mathrm{NULL})=1$ for the rewarding node and $\phi_{H-1}(x^-,\mathrm{NULL}) = 0$ for the non-rewarding node. Since we define a feature class of size $\abr{\Philc_h} = 2^{h+1}$ for each level $h \in [H-1]$, the total size of the product class is $\abr{\Philc} = \Pi_{h=0}^{H-2} \abr{\Philc_h} = 2^{(H-1)H/2}$.

\paragraph{Verifying \pref{assum:realizability_linear_completeness}}
Notice that from our construction of $\Mcal$, there is a one-on-one correspondence between one of $2^{H-1}$ state-action pair $(x_{H-2}^*,a_{H-2}^*)$ and one of $2^{H-1}$ MDP $M\in\Mcal$. Therefore, for $i$-th such state-action pair ($i \in [2^{H-1}]$) at level $H-2$, we use $M^i$ to denote its corresponding MDP. Now consider any MDP $M^i \in \Mcal$. For any level $h \in [H-1]$, let $i_h$ denote the state-action pair (among $2^{h+1}$-many state-action pairs at depth $h$) which lies along the path from the root to the rewarding node $x^+$. To verify the realizability condition, we show that 
the feature $\philci = \rbr{\phi_0^{i_0}, \ldots, \phi^{i_h}_h, \ldots, \phi^{i_{H-2}}_{H-2}, \phi_{H-1}}\in\Philc$ satisfies the linear completeness structure (\pref{def:lin_completeness}).
Firstly, note that by definition $\norm{\phi_h^{i_h}(x_h,a_h)}_2 \le 1$ for all $h \in [H],x_h,a_h$. Now, we verify that the requirements in \pref{def:lin_completeness} are satisfied.

For any $h\in[H-1]$ and pair $(x_h,a_h)$ with $\philci_h(x_h,a_h) = 0$, all subsequent states $x_{h+1}$ reachable from $x_{h}$ and any action $a$ also have $\philci_{h+1}(x_{h+1},a) = 0$: zero-feature intermediate state-action pairs only transit to zero feature value states at the next timestep. 
Therefore, the backup condition is satisfied by default:
\begin{align*}
   \inner{\philci_{h}(x_h,a_h)}{\theta_{h}}-\rbr{\Tcal^\zero_h  \inner{\philci_{h+1}}{\theta_{h+1}}}(x_h,a_h) = 0-0=0.
\end{align*}

On the other hand, for any $h\in[H-1]$ and pair $(x_h,a_h)$ with $\philci_h(x_h,a_h)=1$, we have $\philci_{h+1}(x_{h+1},a)=1$ for one action along the path to $x^+$ and $\philci_{h+1}(x_{h+1},a')=0$ for the other. For any $\theta_{h+1} \in \RR$ (notice that $\dlc=1$ in our construction), we have 
\begin{align*}
    \rbr{\Tcal^\zero_h  \inner{\philci_{h+1}}{\theta_{h+1}}}(x_h,a_h) = \begin{cases} \theta_{h+1} & \theta_{h+1} \ge 0\\
    0 & \theta_{h+1} < 0.
    \end{cases}
\end{align*}
Thus, for both cases, we can set $\theta_h = \theta_{h+1}$ or $0$ to satisfy the linear completeness condition $ \langle\philci_{h}(x_h,a_h),\theta_{h}\rangle-\rbr{\Tcal^\zero_h  \langle\philci_{h+1},\theta_{h+1}\rangle}(x_h,a_h) = 0$.
Hence, the chosen feature mapping $\philci$ satisfies the linear completeness structure in \pref{def:lin_completeness}.

\paragraph{Lower bound for exploration} Learning in this family of MDPs $\Mcal$ is provably hard as the feature and reward classes do not reveal any information about the pair $(x^*_{H-2}, a^*_{H-2})$ and the agent has to try each of the $2^{H-1}$ paths \citep{krishnamurthy2016pac}. Hence, any learning agent has to sample $\Omega(2^{H})$ trajectories to find the optimal policy in any given instance $M \in \Mcal$. The stated lower bound statement follows from the fact that $\dlc=1, A=2$, $1/\veps$ is constant and the sample complexity is $\Omega(2^{H})$ which scales with $\abr{\Philc} = 2^{(H-1)H/2}$ and $\abr{\Xcal} = 2^{H-1}+1$.
\end{proof}

\paragraph{Discussions} 
The family $\Mcal$ of hard instances highlights a fundamental distinction between the low-rank and linear completeness settings when underlying true representations are unknown. Our result further highlights that assuming reachability \citep{modi2021model} and/or explorability \citep{zanette2020provably} does not alleviate the fundamental hardness. Reachability is satisfied as for each MDP in $\Mcal$, each node at every level can be reached with probability $1$ by taking the correct actions which lie along the path from the root node. Similarly, for explorability, we need to verify that for any $\theta \in \RR$, the constant $\max_\pi \min_{|\theta|=1}\abr{\EE_{\pi}[\langle \phi_h(x_h,a_h), \theta\rangle]}$ is large for all $h \in [H]$ (notice that $\theta$ is one dimension so $\|\theta\|_2=|\theta|$). Again, it is easy to see that for both values of $\theta \in \{-1,1\}$, the policy corresponding to the path from root node $x_0$ to $x^+$ maximizes this constant for all steps with a value of $1$.

Moreover, our constructed family of hard instances is quite general as it is applicable to the settings of online reward-specific exploration and learning with a generative model. In order to verify this for the former setting, note that our reward class is a singleton reward $\{R\}$ and exposing this reward (reward class) to the agent still does not disclose any information about the pair $(x_{H-2}^*, a_{H-2}^*)$ to the agent. Hence, the required number of trajectories to identify this pair is again $\Omega(2^{H})$. Similarly, for a generative model, the problem of identifying the pair $(x_{H-2}^*, a_{H-2}^*)$ is inherently a best-arm identification problem among the $2^{H-1}$ possibilities. Thus, the existing lower bounds for best-arm identification \citep{krishnamurthy2016pac} directly lead to a sample complexity bound of $\Omega(2^{H})$. In addition, we can see from the construction that our hardness result also shows that a polynomial in $|\Xcal|$ dependence is unavoidable in this case. We also remark that the stated hardness result can be easily tweaked to show a $1/\veps^2$ dependence for identifying an $\veps$-optimal policy by moving from deterministic transitions at timestep $H-1$ to stochastic transition probabilities: $P_{H-1}(x^+\mid x_{H-2}^*, a_{H-2}^*) = \frac{1}{2} + \veps$, $P_{H-1}(x^-\mid x_{H-2}^*, a_{H-2}^*) = \frac{1}{2} - \veps$, and $P_{H-1}(x^+\mid x_{H-2}, a_{H-2}) =P_{H-1}(x^-\mid x_{H-2}, a_{H-2}) =  \frac{1}{2}$ if $x_{H-2}\neq x_{H-2}^*$ or $a_{H-2}\neq a_{H-2}^*$. The realizable feature $\philci$ will be a two-dimensional representation after this modification, where we change previous one dimension values $0$ and $1$ to two dimension $(0,0)$ and $(1,-1)$ respectively. 

The hardness result highlights the insufficiency of realizability of a linearly complete feature (\pref{assum:realizability_linear_completeness}) in the representation learning setting and indicates that realizability of stronger completeness style features may be necessary for provably efficient reward-free RL.

\section{Algorithm-specific counterexample of \rfolive}
\label{app:alg_counterexp}
In this section, we show an algorithm-specific counterexample of \rfolive (\pref{alg:rf_olive}) that satisfies realizability (\pref{assum:realizability_F}) and has a low Bellman Eluder dimension, while only violates completeness (\pref{assum:completeness_F}). Together with the positive results (\pref{thm:rfolive_q} and \pref{thm:rfolive_v}), we conjecture that realizability-type assumptions are not sufficient for statistically efficient reward-free RL. As we know that \olive \citep{jiang2017contextual,jin2021bellman} only requires realizability and low Bellman Eluder dimension for reward-aware RL and \rfolive is its natural extension to the reward-free setting, we believe that the hardness between reward-aware and reward-free RL has a deep connection to the sharp separation between realizability and completeness \citep{chen2019information,wang2020statistical,wang2021exponential,xie2021batch,weisz2021query,weisz2021exponential,weisz2022tensorplan,foster2021offline}.

\begin{theorem}
\label{thm:alg_counter}
There exists an MDP $M$, a function class $\Fcal$, a reward class $\Rcal$, where \pref{assum:realizability_F} holds and the function class $\Fcal-\Fcal$ has a low Bellman Eluder dimension ($\dqbee$ defined in \pref{thm:rfolive_q}). However, with probability 0.25, (Q-type) \rfolive with infinite amount of data cannot output a 0.1-optimal policy for some $R\in\Rcal$.
\end{theorem}

\begin{proof}
We first discuss the counterexample shown in \pref{fig:alg_counter_rfolive} and \cref{table:alg_counter_rfolive} at a high level. In our construction, with probability 0.25, the agent will only explore and collect data at some specific place because it is sufficient to eliminate all candidate functions predicting large positive values at $x_0$. Then in the offline phase, the agent cannot eliminate some bad function because of lack of support in the collected data. Then performing function elimination in the offline phase fails. We provide more details in the sequel.

\begin{figure}[htb]
	\center
	\begin{tikzpicture}[scale=3]
		\node[state] (s0)      {$x_0$};
		\node[state] (s1) [below left=3em of s0]  {$x_A$};
		\node[state] (s2) [below right =3em of s0] {$x_B$};
		
		\draw[->] (s0) -> node[above = 3.3 em,left=.5em] {$\mathrm{left}$} (s1);
		\draw[->] (s0) -> node[above = 3.3 em,right=.5em] {$\mathrm{right}$} (s2);
	\end{tikzpicture}
\caption{Algorithm-specific counterexample of \rfolive without completeness assumption (\pref{assum:completeness_F}).}
\label{fig:alg_counter_rfolive}
\end{figure}

\renewcommand{\arraystretch}{1.25}
\begin{table}[htb]
\begin{center}
	\begin{tabular}{ |c|c|c|c|c| } 
		\hline
		&$(x_0,\mathrm{left})$ & $(x_0,\mathrm{right})$ & $(x_A,\mathrm{NULL})$ & $(x_B,\mathrm{NULL})$ \\ 
		\hline
		$R_1$ & 0 & 0 & 1 & 0 \\ 
		$R_2$ & 0 & 0 & 0.2 & 0.1 \\ 
		\hline 
		$Q_{R_1}^*$ & 1 & 0 & 1 & 0 \\ 
		$Q_{R_2}^*$ & 0.2 & 0.1 & 0.2 & 0.1 \\ 
		\hline		
		$f_{R_1}$ & 1 & 0 & 0 & 0 \\ 
		$f_{R_2}$ & 0.2 & 0.1 & 0 & 0 \\ 
		$f_{\mathrm{bad}}$ & 0.21 & 0.3 & 0.01 & 0.1 \\ 
		\hline		
		$f_{R_1}+R_1=Q_{R_1}^*$ & 1 & 0 & 1 & 0  \\ 
		$f_{R_2}+R_2=Q_{R_2}^*$ & 0.2 & 0.1 & 0.2 & 0.1  \\ 
		$f_{R_1}+R_2$ & 1 & 0 & 0.2 & 0.1 \\ 
		$f_{\mathrm{bad}}+R_2$ & 0.21 & 0.3 & 0.21 & 0.1 \\ 
		\hline
	\end{tabular}
\end{center}
\caption{Algorithm-specific counterexample for \rfolive without completeness assumption (\pref{assum:completeness_F}).}
\label{table:alg_counter_rfolive}
\end{table}

\paragraph{Construction} In \pref{fig:alg_counter_rfolive}, taking action $\mathrm{left}$ and action $\mathrm{right}$ in state $x_0$ transits to $x_A$ and $x_B$ respectively. We denote the null action at $x_A,x_B$ and the null state at level $H=2$ as $\mathrm{NULL}$ and $x_{\mathrm{NULL}}$ respectively. In this example, the length of horizon is $H=2$. We construct $\Fcal=\Fcal_0\times\Fcal_1$, where $\Fcal_0=\{\zero,f_{R_1,0},f_{R_2,0},f_{\mathrm{bad},0}\}$ and $\Fcal_1=\{\zero,f_{\mathrm{bad},1}\}$. In addition, we construct $\Rcal=\Rcal_0\times\Rcal_1$, where $\Rcal_0=\{\zero\}$ and $\Rcal_1=\{R_{1,1},R_{2,1}\}$. Recall that the second subscript of $f\in\Fcal_0,\Fcal_1$ and $R\in\Rcal_1$ is the index for the timestep. The details are shown in \cref{table:alg_counter_rfolive}. Notice that here we use the layered MDP for simplicity. To convert it to a non-layered MDP, we only need to set corresponding values in the transition function, reward function, and $f\in\Fcal$ to be 0.

\paragraph{Verifying realizability and low Bellman Eluder dimension} One can immediately see that realizability (\pref{assum:realizability_F}) is satisfied. For example, we have $Q^*_{R_1,1}=f_{R_1,1}+R_{1,1}=\zero+R_{1,1}$, which implies that $Q^*_{R_1,1}\in\Fcal_1+\Rcal_1$. Similarly, we can verify that $Q^*_{R_2,1}\in\Fcal_1+\Rcal_1,Q^*_{R_1,0}\in\Fcal_0+\Rcal_0,Q^*_{R_2,0}\in\Fcal_0+\Rcal_0$. 

In addition, $\Fcal-\Fcal$ has a low Bellman Eluder dimension. It is because the Bellman Eluder dimension can be upper bounded by the Bellman rank (\pref{prop:br_be}) and the Bellman rank can be upper bounded by the number of states \citep{jiang2017contextual}. Therefore, the Bellman Eluder dimension is just a small bounded finite number. Later we will show that with even infinite amount data \rfolive fails, which implies that we cannot get a polynomial sample complexity bound in this case.

\paragraph{Violation of completeness} 
We can easily see that for $f_{\mathrm{bad},1}\in\Fcal_1$ and $R_{2,1}\in\Rcal_1$, its Bellman backup $\Tcal^\zero_0(f_{\mathrm{bad},1}+R_{2,1})\notin \Fcal_0$. This means that completeness (\pref{assum:completeness_F}) does not hold.


\paragraph{\rfolive fails in the counterexample} We first consider running (Q-type) \rfolive on this counterexample during the online phase. In the following, we will assume the more favorable case where the agent can collect infinitely many number of samples in \pref{line:estimate_BE} and
\pref{line:collect_nelim} (i.e., no statistical/estimation error for the average Bellman error in the empirical version).

The agent will pick the most optimistic function for exploration. In the first iteration, such an optimistic function at level 0 will be equal to $f_{R_1,0}-\zero$. Therefore, starting from $x_0$, the agent will choose the action $\mathrm{left}$. With at least probability 0.5, it will pick level 0 to eliminate and collect data (i.e., collecting data at ($x_0,\mathrm{left}$) in \pref{line:collect_nelim}). The reason is that for \pref{line:deviation}, the large average Bellman error always exists at ($x_0,\mathrm{left}$) while the Bellman error could be large at ($x_{\mathrm{left}}, \mathrm{NULL}$). By adversarial tie-breaking, there is at least 0.5 probability that ($x_0,\mathrm{left}$) is chosen. 

Now consider the case that the agent pick ($x_0,\mathrm{left}$) to collect data in \pref{line:collect_nelim}. We can see that only function $\zero,f_{\mathrm{bad}}-f_{R_2},f_{R_2}-f_{\mathrm{bad}}\in\Fcal-\Fcal$ will survive while all other functions violate the collected constraint. Here we notice that for any $f\in\Fcal-\Fcal$ we have $V_f(x_A)=0$ or $V_f(x_A)=\pm 0.01$. So the survived function $f\in\Fcal-\Fcal$ belongs to one of the following cases: (i) $f(x_0,\mathrm{left})=0$ and $V_f(x_A)=0$, (ii) $f(x_0,\mathrm{left})=  0.01$ and $V_f(x_A)= 0.01$, or (iii) $f(x_0,\mathrm{left})= -0.01$ and $V_f(x_A)=- 0.01$. At the second iteration, \rfolive will choose $f_{\mathrm{bad}}-f_{R_2}$ (i.e., case (ii)) and action $\mathrm{right}$. Due to adversarial tie-breaking, we have that with probability 0.5, the agent collects data at ($x_B,\mathrm{NULL}$) and eliminates $f_{\mathrm{bad}}-f_{R_2}$. In this case, for the third iteration, the agent chooses function $\zero$ and then terminates in \pref{line:valid} since the average Bellman error is 0. 

For the offline phase, let us consider the reward function $R_2$ and the elimination on $\Foff(R_2)= \Fcal+R_2$. Recall that in the online phase, we only collect data on ($x_0,\mathrm{left},x_A$) and ($x_B,\mathrm{NULL},x_{\mathrm{NULL}}$). It is easy to see that we will eliminate $f_{R_1}+R_2$ from this constraint. However, we cannot eliminate either $f_{R_2}+R_2$ or $f_{\mathrm{bad}}+R_2$ because they all have zero average Bellman error under these two constraints. Then by optimistic selection criteria, the agent will pick $f_{\mathrm{bad}}+R_2$. This induces a sub-optimal policy ($\mathrm{right}$) with accuracy $\veps = 0.1$.

Therefore, with probability at least 0.25, (Q-type) \rfolive fails to output a 0.1 optimal policy for some reward $R\in\Rcal$ in this counterexample.
\end{proof}

\section{Discussions on other variants of \olive}
\label{app:other_variant}

In this section, we briefly discuss that some other variant of \olive in the reward-free setting could easily fail under \pref{assum:realizability_F}, \pref{assum:completeness_F}, and low Bellman Eluder dimension (where we know that \rfolive works). 

One adaptation of \olive to the reward free case is to perform exploration on the joint function class space and we call it \jolive. More specifically, we maintain a version space $\Fcal^t+\Rcal^t\subseteq \Fcal + \Rcal$ during the online phase. In each online iteration, we pick the most optimistic function $f^t_{\mathrm{on}}=f^t+R^t=\argmax_{f+R\in \Fcal^t+\Rcal^t} V_{f+R}(x_0)$ and explore according to $\pi^t=\pi_{f_{\mathrm{on}}^t}$. Here for $f^t_{\mathrm{on}}$, we decompose it as the sum of $f^t\in\Fcal^t$ and $R^t\in\Rcal^t$. Then we roll out policy $\pi^t$ and estimate the average Bellman error. For the termination condition, like \olive, we check whether $f_{\mathrm{on}}^t$ has a small average Bellman error under reward $R^t$. If the algorithm is not terminated, we pick a level for elimination and collect the constraint. At the end of each online iteration, we shrink the version space of $\Fcal^t$ and $\Rcal^t$ according to the average Bellman error. For the offline phase, we use collected constraints to perform elimination like \rfolive and then output the greedy policy of the optimistic survived function. We will show that this variant could get stuck in the following counterexample even in the case that we are allowed to collect infinite amount of samples to build estimates (i.e., no statistical/estimation error).

\paragraph{Construction} 
We consider the MDP in \pref{fig:alg_counter_rfolive} and reward function class $\Rcal=\Rcal_0\times\Rcal_1$, where $\Rcal_0=\{\zero\}$ and $\Rcal_1=\{R_{1,1},R_{2,1}\}$. The function class $\Fcal=\Fcal_0\times\Fcal_1$ is constructed as $\Fcal_0=\{\zero,f_{R_1,0},f_{R_2,0},f_{\mathrm{bad},0}\}$ and $\Fcal_1=\{\zero\}$. Compared with the counterexample in \pref{thm:alg_counter}, the difference is that $f_{\mathrm{bad}}$ is changed and now we only have a single function $\zero$ in $\Fcal_1$. The details as shown in \cref{table:alg_counter_joint_olive}. As discussed in the proof of \pref{thm:alg_counter}, we can convert the layered MDP here to a non-layered one.

\renewcommand{\arraystretch}{1.25}
\begin{table}[htb]
\begin{center}
	\begin{tabular}{ |c|c|c|c|c| } 
		\hline
		&$(x_0,\mathrm{left})$ & $(x_0,\mathrm{right})$ & $(x_A,\mathrm{NULL})$ & $(x_B,\mathrm{NULL})$ \\ 
		\hline
		$R_1$ & 0 & 0 & 1 & 0 \\ 
		$R_2$ & 0 & 0 & 0.2 & 0.1 \\ 
		\hline 
		$Q_{R_1}^*$ & 1 & 0 & 1 & 0 \\ 
		$Q_{R_2}^*$ & 0.2 & 0.1 & 0.2 & 0.1 \\
		\hline		
		$f_{R_1}$ & 1 & 0 & 0 & 0 \\ 
		$f_{R_2}$ & 0.2 & 0.1 & 0 & 0 \\ 
		$f_{\mathrm{bad}}$ & 0.2 & 0.3 & 0 & 0 \\ 
		\hline		
		$f_{R_1}+R_1=Q_{R_1}^*$ & 1 & 0 & 1 & 0  \\ 
		$f_{R_2}+R_2=Q_{R_2}^*$ & 0.2 & 0.1 & 0.2 & 0.1 \\ 
		$f_{R_1}+R_2$ & 1 & 0 & 0.2 & 0.1 \\ 
		$f_{\mathrm{bad}}+R_2$ & 0.2 & 0.3 & 0.2 & 0.1 \\ 
		\hline
	\end{tabular}
\end{center}
\caption{Algorithm-specific counterexample for \jolive under all assumptions.}
\label{table:alg_counter_joint_olive}
\end{table}

\paragraph{Verifying realizability, completeness, and low Bellman Eluder dimension} Realizability and low Bellman Eluder dimension can be verified in the same way as the counterexample for \rfolive in \pref{thm:alg_counter}. For completeness (\pref{assum:completeness_F}), one can easily verify that by noticing we have $\Fcal_1=\{\zero\}$ now.

\paragraph{\jolive fails in the counterexample} In \jolive, the agent will pick the optimistic function in the joint function space to explore during the online phase. At the first iteration, there are two candidates $f_{R_1}+R_1$ and $f_{R_1}+R_2$. By adversarial tie-breaking, the agent will choose $f_{R_1}+R_1$ with probability 0.5 and choose action $\mathrm{left}$. Then the agent will terminate immediately because the average Bellman error is 0 for $f_{R_1}+R_1$ everywhere under reward $R_1$.

For the offline phase, we similarly consider reward $R_2$. It is easy to see that $f_{R_1}+R_2$ will be eliminated while both $f_{R_2}+R_2$ and $f_{\mathrm{bad}}+R_2$ will survive. Then by optimistic selection, the agent will choose the greedy policy of $f_{\mathrm{bad}}+R_2$. This induces a sub-optimal policy ($\mathrm{right}$) with accuracy $\veps=0.1$. 

Therefore, with probability 0.5, \jolive fails in this counterexample.

\section{Auxiliary results}
\label{app:auxuliary}
In this section, we provide auxiliary results for the paper. We show covering number arguments in \pref{app:covering} and some bounds on Bellman Eluder dimensions in \pref{app:dimBE}.
\subsection{Covering number}
\label{app:covering}
In this part, we present the covering number argument for the linear function class used in the paper.
\begin{lemma}[Size of $\varepsilon$-cover for linear function class]
\label{lem:covering}
We have three claims here
\begin{enumerate}
    \item Consider $\Fcal(\{\philc\})=\Fcal_0(\{\philc\},H-1)\times\ldots\times \Fcal_{H-1}(\{\philc\},0)$, where $\Fcal_h(\{\philc\},B_h) = \big\{f_h(x_h,a_h) = \inner{\philc_h(x_h,a_h)}{\theta_h} : \|\theta_h\|_2 \le B_h\sqrt{\dlc},\langle \philc_h(\cdot),\theta_h\rangle \in[-B_h, B_h] \big\}$. Then we have $\Ncal_{\Fcal(\{\philc\})}(\veps)\le\rbr{\frac{2H^2\sqrt{\dlc}}{\veps}}^{\dlc}$.
    \item Consider $\Fcal(\{\philr\})=\Fcal_0(\{\philr\},H-1)\times\ldots\times \Fcal_{H-1}(\{\philr\},0)$, where $\Fcal_h(\{\philr\},B_h) = \big\{f_h(x_h,a_h) = \inner{\philr_h(x_h,a_h)}{\theta_h} : \|\theta_h\|_2 \le B_h\sqrt{\dlr},\langle \philr_h(\cdot),\theta_h\rangle \in[-B_h, B_h] \big\}$. Then we have $\Ncal_{\Fcal(\{\philr\})}(\veps)\le\rbr{\frac{2H^2\sqrt{\dlr}}{\veps}}^{\dlr}$.
    \item Consider $\Fcal(\Philr)=\Fcal_0(\Philr,H-1)\times\ldots\times \Fcal_{H-1}(\Philr,0)$, where $\Fcal_h(\Philr,B_h) = \big\{f_h(x_h,a_h) = \inner{\phi_h(x_h,a_h)}{\theta_h} : \phi_h\in\Philr_h,\|\theta_h\|_2 \le B_h\sqrt{\dlr},\langle \phi_h(\cdot),\theta_h\rangle \in[-B_h, B_h] \big\}$. Then we have $\Ncal_{\Fcal(\Philr)}(\veps)\le|\Philr|\rbr{\frac{2H^2\sqrt{\dlr}}{\veps}}^{\dlr}$.
\end{enumerate}
\end{lemma}
\begin{proof}
This is a standard result. For the first one, we can construct a cover over $\{\theta_h:\|\theta_h\|_2\le (H-h-1)\sqrt{\dlc}\}$ in the 2-norm at scale $\varepsilon/H$ for each level $h\in[H]$. Then this cover immediately implies a cover over the function $\Fcal(\{\philc\})$. The covering number directly follows the covering number of the 2-norm ball. The second follows the same steps. For the third result, we additionally union over $\phi\in\Philr$.
\end{proof}

\subsection{Bounds on the Bellman Eluder dimension}
\label{app:dimBE}

In this part, we show that Q-type and V-type Bellman Eluder dimensions for the instantiated linear MDP, low-rank MDP, and linear completeness with known feature settings are indeed small. We will use the following relation between Bellman rank and Bellman Eluder dimension from \citet{jin2021bellman}:
\begin{proposition}[Bellman rank $\subseteq$ Bellman Eluder dimension, Proposition 11 and 21 in \citet{jin2021bellman}]
\label{prop:br_be}
If an MDP with function class $\Fcal$ has Q-type (or V-type) Bellman
rank $\dbr$ with normalization parameter $\zeta$, then the respective Bellman Eluder dimension $\dqbe^R(\Fcal, \Dcal_\Fcal, \varepsilon)$ (or $\dvbe^R(\Fcal, \Dcal_\Fcal, \varepsilon)$) is bounded by $\tilde O \rbr{1+\dbr^R\log\rbr{1+\frac{\zeta}{\veps}}}$.
\end{proposition}

\paragraph{Linear/low-rank MDPs} Before stating the result for low-rank MDPs, we recall the following well-known property for the class:
\begin{lemma}[\citet{jin2019provably,modi2021model}]
\label{lem:linear_mdp}
Consider a low-rank MDP $M$ (\pref{def:lowrank}) with embedding dimension $\dlr$. For any function $f: \Xcal \rightarrow [-c,c]$, we have:
\begin{align*}
    \EE \sbr{f(x_{h+1}) \mid x_h,a_h } = \inner{\philr_h(x_h,a_h)}{\theta^*_f}
\end{align*}
where $\theta^*_f \in \RR^{\dlr}$ and we have $\|\theta^*_f\|_2 \le c\sqrt{\dlr}$. A similar linear representation is true for $\EE_{a \sim \pi_{h+1}}[f(x_{h+1}, a)\mid x_h,a_h]$ where $f: \Xcal \times \Acal \rightarrow [-c,c]$ and a policy $\pi_{h+1}: \Xcal \rightarrow \Acal$.
\end{lemma}
\begin{proof}
For state-value function $f$, we have 
\begin{align*}
    \EE \sbr{f(x_{h+1})\mid x_h,a_h } &= \int f(x_{h+1})P_{h}(x_{h+1}\mid x_h,a_h) d(x_{h+1})\\
    &=\int f(x_{h+1})\inner{\philr_h(x_h,a_h)}{\mulr_h(x_{h+1})} d(x_{h+1})\\
    &=\inner{\philr_h(x_h,a_h)}{\int f(x_{h+1})\mulr_h(x_{h+1})d(x_{h+1})}\\
    &=\inner{\philr_h(x_h,a_h)}{\theta_f^*}, 
\end{align*}
where $\theta_f^*\coloneqq \int f(x_{h+1})\mulr_h(x_{h+1})d(x_{h+1})$ is a function of $f$. Additionally, we obtain $\|\theta_f^*\|_2\le c\sqrt {\dlr}$ from \pref{def:lowrank}.

For Q-value function $f$, we similarly have
\begin{align*}
    \EE_{a\sim\pi_{h+1}} \sbr{f(x_{h+1},a)\mid x_h,a_h } =\inner{\philr_h(x_h,a_h)}{\theta_f^*}, 
\end{align*}
where $\theta_f^*\coloneqq \iint f(x_{h+1},a_{h+1})\pi(a_{h+1}\mid x_{h+1}) \mulr_h(x_{h+1})d(x_{h+1})d(a_{h+1})$ and $\|\theta_f^*\|_2 \le c\sqrt {\dlr}$.
\end{proof}

Now, we can state the following bound on the V-type Bellman Eluder dimension for low-rank MDPs:
\begin{proposition}[Low-rank MDP] 
\label{prop:lr_vbe}
Consider a low-rank MDP $M$ of embedding dimension $\dlr$ with a realizable feature class $\Philr$ (\pref{assum:realizability_low_rank}). Define the corresponding linear function class $\Fcal(\Philr)=\Fcal_0(\Philr,H-1)\times\ldots\times \Fcal_{H-1}(\Philr,0)$ using 
\[
\Fcal_h(\Philr,\hspace{-.1em}B_h)\hspace{-.2em}=\hspace{-.2em}\cbr{\hspace{-.2em} f_h(x_h,a_h) \hspace{-.2em}=\hspace{-.2em} \inner{\phi_h(x_h,a_h)}{\hspace{-.1em}\theta_h}\hspace{-.2em} :\hspace{-.2em} \phi_h \in \Philr_h, \hspace{-.2em} \|\theta_h\|_2 \hspace{-.2em}\le\hspace{-.2em} B_h\sqrt{\dlr},\langle\phi_h(\cdot),\hspace{-.1em}\theta_h\rangle\hspace{-.2em} \in\hspace{-.2em} [-B_h,\hspace{-.1em}B_h] \hspace{-.2em}}\hspace{-.2em}.
\]
Then, for the difference class $\Fon = \Fcal(\Philr) - \Fcal(\Philr)$ we have
\begin{align*}
    \dvbe^\zero(\Fon, \Dcal_{\Fon}, \veps) \le O\rbr{1+\dlr \log \rbr{1+\frac{H \sqrt{\dlr}}{\veps}}}.
\end{align*}
\end{proposition}
\begin{proof}
We start by showing that the V-type Bellman rank for function class $\Fon$ in the low-rank case is small. To that end, consider the Bellman error defined for any roll-in policy $\pi$ and function $f \in \Fon$:
\begin{align*}
    \EcalVzero(f,\pi,h)= {} & \EE\sbr{f_h(x_h,a_h)-f_{h+1}(x_{h+1},a_{h+1})\mid a_{0:h-1}\sim\pi, a_{h:h+1}\sim \pi_f} \\
    = {} & \EE\sbr{f_h(x_h,a_h) - \inner{\philr_h(x_h,a_h)}{\theta^*_{f,h}} \mid a_{0:h-1}\sim\pi, a_{h}\sim \pi_f} 
\end{align*}
where we used \pref{lem:linear_mdp} for low-rank MDPs to write $\EE\sbr{f_{h+1}(x_{h+1},a_{h+1})\mid x_h,a_h, a_{h+1} \sim \pi_f}$ as $\langle\philr_h(x_h,a_h),$ $\theta^*_{f,h}\rangle$. Here, $f_{h+1} \in [-2(H-h-2), 2(H-h-2)]$ and $f_h \in [-2(H-h-1),2(H-h-1)]$ implying that $f_h - \inner{\philr_h(x_h,a_h)}{\theta^*_{f,h}} \in [-4(H-h-1), 4(H-h-1)]$. Therefore, using \pref{lem:linear_mdp} again, we have:
\begin{align*}
    \EcalVzero(f,\pi,h)= {} & \EE\sbr{f_h(x_h,a_h) - \inner{\philr_h(x_h,a_h)}{\theta^*_{f,h}} \mid a_{0:h-1}\sim\pi, a_{h}\sim \pi_f} \\
    = {} & \EE\sbr{\inner{\philr_{h-1}(x_{h-1},a_{h-1})}{\tilde \theta(f)} \mid x_{h-1},a_{h-1},a_{0:h-1} \sim \pi} \\
    = {} & \inner{\nu(\pi)}{\tilde \theta(f)}
\end{align*}
where $\|\tilde \theta(f)\|_2 \le 4(H-h-1) \sqrt{\dlr}$ and $(\nu(\pi))(x_{h-1},a_{h-1}) = \EE[\philr_{h-1}(x_{h-1},a_{h-1}) \mid a_{0:h-1} \sim \pi]$. Hence, the V-type Bellman rank for this function class is bounded by $\dlr$ with normalization parameter $4(H-h-1) \sqrt{\dlr}$. Finally using \pref{prop:br_be}, we get the desired bound on the Bellman Eluder dimension.
\end{proof}

For linear MDPs, where the feature $\philr$ is the known feature case, we show that its Q-type Bellman Eluder dimension is also small:
\begin{proposition}[Linear MDP] 
\label{prop:linear_qbe}
Consider a low-rank MDP $M$ (\pref{def:lowrank}) with embedding dimension $\dlr$ and $\philr$ is known. Define the corresponding linear class $\Fcal(\{\philr\})=\Fcal_0(\{\philr\},H-1)\times\ldots\times \Fcal_{H-1}(\{\philr\},0)$ using 
\begin{align*}
\Fcal_h(\{\philr\},B_h)\hspace{-.15em}=\hspace{-.15em}\cbr{f_h(x_h,a_h) \hspace{-.15em}=\hspace{-.15em} \inner{\philr_h(x_h,a_h)}{\theta_h} : \|\theta_h\|_2 \le B_h\sqrt{\dlr},\langle\philr_h(\cdot),\theta_h\rangle \in [-B_h,B_h] }.
\end{align*}
Then, for the difference class $\Fon = \Fcal(\{\philr\}) - \Fcal(\{\philr\})$ we have
\begin{align*}
    \dqbe^\zero(\Fon, \Dcal_{\Fon}, \veps)  \le O\rbr{1+\dlr \log \rbr{1+\frac{H \sqrt{\dlr}}{\veps}}}.
\end{align*}
and
\begin{align*}
     \dvbe^\zero(\Fon, \Dcal_{\Fon}, \veps) \le O\rbr{1+\dlr \log \rbr{1+\frac{H \sqrt{\dlr}}{\veps}}}.
\end{align*}
\end{proposition}
\begin{proof}
The V-type Bellman Eluder dimension bound (\pref{prop:lr_vbe}) implies the same upper bound for $\dvbe^\zero(\Fon, \Dcal_{\Fon},\veps)$, where $\Fon$ is defined using the singleton feature class $\cbr{\philr}$. For Q-type Bellman Eluder dimension, we again start with the Q-type Bellman rank. For any $f\in\Fon$, we have:
\begin{align*}
    \EcalQzero(f,\pi,h)= {} & \EE\sbr{f_h(x_h,a_h)-f_{h+1}(x_{h+1},a_{h+1})\mid a_{0:h}\sim\pi, a_{h+1} \sim \pi_f} \\
    = {} & \EE\sbr{\inner{\philr_h(x_h,a_h)}{\theta_{h}-\theta_h'} - \inner{\philr_h(x_h,a_h)}{\theta^*_{f,h}} \mid a_{0:h}\sim\pi} \\
    = {} & \inner{\EE\sbr{\philr_h(x_h,a_h) \mid a_{0:h}\sim\pi}}{\theta_{h} -\theta_h'- \theta^*_{f,h}}.
\end{align*}
Using the same magnitude calculations for $\theta^*_{f,h}$, we have $\|\theta_{h} -\theta_h'- \theta^*_{f,h}\|_2 \le 4(H-h-1) \sqrt{\dlr}$. Therefore, we again have the Q-type Bellman rank bounded by $\dlr$ with normalization parameter $4(H-h-1) \sqrt{\dlr}$. Using \pref{prop:br_be}, we get the same bound on the Q-type Bellman Eluder dimension.
\end{proof}

\paragraph{Linear completeness setting} For the linear completeness setting in the known feature case, we show that its Q-type Bellman Eluder dimension is small. 
\begin{proposition}[Linear completeness setting]
\label{prop:lc_qbe}
Consider an MDP $M$ that satisfies linear completeness (\pref{def:lin_completeness}) with feature $\philc$. Define the corresponding linear class $\Fcal(\{\philc\})=\Fcal_0(\{\philc\},H-1)\times\ldots\times \Fcal_{H-1}(\{\philc\},0)$ using 
\begin{align*}\Fcal_h(\{\philc\},B_h)\hspace{-.15em}=\hspace{-.15em}\cbr{f_h(x_h,a_h)\hspace{-.15em} =\hspace{-.15em} \inner{\philc_h(x_h,a_h)}{\theta_h} \hspace{-.15em}:\hspace{-.15em} \|\theta_h\|_2 \le B_h\sqrt{\dlc},\langle\philc_h(\cdot),\theta_h\rangle \in [-B_h,B_h] }.
\end{align*}
Then, for the difference class $\Fon = \Fcal(\{\philc\}) - \Fcal(\{\philc\})$ we have:
\begin{align*}
    \dqbe^\zero(\Fon, \Dcal_{\Fon}, \veps) \le O\rbr{1+\dlc \log \rbr{1+\frac{H \sqrt{\dlc}}{\veps}}}.
\end{align*}
\end{proposition}
\begin{proof}
Consider the Q-type Bellman rank, for any $f\in\Fon$, we have:
\begin{align*}
    {}&\EcalQzero(f,\pi,h)
    \\
    = {} & \EE\sbr{f_h(x_h,a_h)-f_{h+1}(x_{h+1},a_{h+1})\mid a_{0:h}\sim\pi, a_{h+1} \sim \pi_f} \\
    = {} & \EE\sbr{\inner{\philc_h(x_h,a_h)}{\theta_{h}-\theta_h'} - \inner{\philc_{h+1}(x_{h+1},a_{h+1})}{\theta_{h+1} - \theta'_{h+1}} \mid a_{0:h}\sim\pi, a_{h+1}\sim \pi_f} \\
    = {} & \EE\sbr{\inner{\philc_h(x_h,a_h)}{\theta_{h}-\theta_h'} - \inner{\philc_{h}(x_{h},a_{h})}{\theta^*_{f,h}} \mid a_{0:h}\sim \pi} \\
    = {} & \inner{\EE\sbr{\philc_h(x_h,a_h) \mid a_{0:h}\sim\pi}}{\theta_{h} - \theta_h'-\theta^*_{f,h}}
\end{align*}
where the penultimate step follows from \pref{def:lin_completeness} with the $\|\theta_h-\theta_h'\|_2,\|\theta^*_{f,h}\|_2, \|\theta_{h+1}-\theta'_{h+1}\|_2 \le 2(H-h-1)\sqrt{\dlc}$. Thus, we have $\|\theta_{h} - \theta^*_{f,h}\|_2 \le 4(H-h-1) \sqrt{\dlc}$ implying a Q-type Bellman rank bound of $\dlc$ with normalization parameter $4(H-h-1) \sqrt{\dlc}$. Using \pref{prop:br_be}, we get the stated bound on the Q-type Bellman Eluder dimension for linear completeness setting.
\end{proof}

\end{document}